\documentclass{article}
\usepackage[symbol]{footmisc}
\usepackage{geometry}

\usepackage[utf8]{inputenc}
\usepackage{amsmath}
\usepackage{amsfonts}
\usepackage{amsbsy}
\usepackage{amssymb}
\usepackage{bbm}
\usepackage{bm}
\usepackage{float}
\usepackage{subfig}
\usepackage{multirow}
\usepackage[boxed]{algorithm2e}
\usepackage{algorithmic} 
\usepackage{mathtools}

\usepackage{colortbl}

\usepackage{graphicx,epstopdf}
\usepackage{fancyhdr}
\usepackage{xcolor}
\definecolor{darkgreen}{RGB}{30 150 30}
\usepackage[colorlinks=true,linkcolor=darkgreen,citecolor=darkgreen,urlcolor=blue]{hyperref}

\usepackage{array,multirow,booktabs}
\definecolor{Gray}{gray}{0.9}

\newcommand{\N}{\mathbb{N}}
\newcommand{\R}{\mathbb{R}}
\DeclareMathOperator*{\argmin}{arg\,min}

\newcommand{\bs}[1]{\boldsymbol{#1}}

\title{\LARGE{\textbf{Weight Matrix Dimensionality Reduction in Deep Learning via Kronecker Multi-layer Architectures}}\footnote{This research was sponsored by ARL under Cooperative Agreement Number W911NF-12-2-0023. The views and conclusions contained in this document are those of the authors and should not be interpreted as representing the official policies, either expressed or implied, of ARL or the U.S. Government. The U.S. Government is authorized to reproduce and distribute reprints for Government purposes notwithstanding any copyright notation herein.  The first and third authors are partially supported by NSF DMS-1848508 and AFOSR FA9550-20-1-0338.}}
\author{Jarom D. Hogue\footnote{Scientific Computing and Imaging Institute, University of Utah, Salt Lake City, UT(\href{mailto:jdhogue@sci.utah.edu}{jdhogue@sci.utah.edu}).} \and Robert M. Kirby\footnote{Scientific Computing and Imaging Institute and School of Computing, University of Utah, Salt Lake City, UT(\href{mailto:kirby@cs.utah.edu}{kirby@cs.utah.edu}).} \and Akil  Narayan\footnote{Scientific Computing and Imaging Institute and Department of Mathematics, University of Utah, Salt Lake City, UT(\href{mailto:akil@sci.utah.edu}{akil@sci.utah.edu}).}}

\date{\small\today}

\usepackage{titlesec}
\titleformat{\section}[runin]{\normalfont\bfseries}{\thesection.}{0.5em}{}
\titleformat{\subsection}[runin]{\normalfont\bfseries}{\thesubsection.}{0.5em}{}

\usepackage{amsthm}
\newtheorem{theorem}{Theorem}[section]
\newtheorem{corollary}{Corollary}[theorem]
\newtheorem{lemma}[theorem]{Lemma}

\begin{document}

\maketitle

\noindent\rule{\textwidth}{0.8pt}

\begin{abstract}
  Deep learning using neural networks is an effective technique for generating models of complex data. However, training such models can be expensive when networks have large model capacity resulting from a large number of layers and nodes. For training in such computationally prohibitive regimes, a reduction of trainable parameters eases the computational burden, and allows implementations of more robust networks. We propose one such novel type of parameter reduction via a new deep learning architecture based on fast matrix multiplication of a matrix Kronecker product decomposition; our network construction can be viewed as a Kronecker product-induced sparsification of an ``extended'' fully connected network. Analysis and practical examples show that this architecture allows a neural network to be trained and implemented with a significant reduction in computational time and resources, while achieving a similar or better error level compared to a traditional feedforward neural network.  
\end{abstract}

\section{Introduction}
Statistical learning using deep neural networks has achieved impressive results in building models for prediction, summarization, and classification of large data sets \cite{goodfellow_deep_2016}. The capacity of such models is dictated by the depth and width (number of layers and nodes, respectively) of the neural network, but such high-capacity networks impose a nontrivial computational burden during training. In such regimes, some reduction of trainable parameters may be implemented in some form to ease the computational burden.  While this may take several forms, of particular interest is accelerating training of neural networks without degrading model performance. One of the computational burdens that arise for high-capacity networks during training is the cost of forward- and back-propagation, amounting to the cost of evaluation of the network and the cost of implementing the computational graph corresponding to the chain rule for differentiation, respectively. In this paper, we analogize this problem to that of matrix multiplication: Matrix-vector multiplication for large matrices can be expensive, but is much more efficient if certain structural properties of matrices exist that can be computationally exploited. In particular, we exploit the fact that the Kronecker product \cite{van1993} provides a low dimensional representation of a large matrix, and use a corresponding implementation of this dimensionality reduction for deep learning. 

Our approach aims to gain computational efficiency by imparting a Kronecker product structure on the \textit{architecture} of a neural network. Ultimately we aim to accelerate training of neural networks without degrading predictive accuracy. We implement a ``dual layer'' approach that is inspired by the structure of a Kronecker product and show that this procedure can significantly reduce computational time for both forward computation and back-propagation, when the network size/capacity is relatively large, compared to a fully connected feedforward neural network. Our approach is a type of connection sparsification approach, but of an ``extended'' network and not the original one. See Figure~\ref{fig:nodes} for a visual depiction of the new Kronecker Dual Layer architecture that we propose. We also show in such cases that we can maintain or even improve accuracy when training on several practical examples.

In short, the contributions of this paper are as follows: (i) We propose a new Kronecker product-inspired deep learning architecture, the Kronecker Dual Layer (KDL), that exploits simplification of arithmetic operations in Kronecker products for matrix multiplication to effect acceleration in both forward- and back-propagation phases of learning; (ii) we provide proof-of-concept theoretical analysis suggesting when a KDL network can be expected to perform well compared to fully connected networks; (iii) we demonstrate the practical effectiveness of KDL architectures on real-world datasets through several test examples, see Table~\ref{Tab:sizes}; and lastly, (iv) we provide open source code for a Tensorflow implementation at \href{https://github.com/JaromHogue/KDLayers}{github.com/JaromHogue/KDLayers}.

\subsection{Related work}\label{ssec:related}
The Kronecker product  has already been incorporated in several areas within the deep learning framework: (i) In \cite{Movshovitz, Zhou} the authors apply a Kronecker product decomposition (KPD) to decompose weight matrices of a trained network, although this typically requires a large number of terms for acceptable accuracy and is thus of limited applicability;  (ii) a generalized KPD is extended to multi-dimensional tensors in \cite{Hameed} to reduce the number of weight parameters and computational complexity in convolutional neural networks; (iii) the Kronecker product has been shown as a viable method to reduce the computational time for back-propagation via an approximate inverse of the Fisher information matrix, \cite{Martens}, providing a means to increase decay rate in the loss; and (iv) a ``Kronecker neural network'' in \cite{Jagtap}, has been established to implement adaptive activation functions in order to avoid local minima while training. We emphasize that our approach is distinct from these methods, as we fundamentally alter the network architecture in an attempt to accelerate training.

In addition to the alternative uses of the Kronecker product mentioned above, there are other methods that seek to reduce the nodes and/or connections of a trained network in order to reduce the computational burden of training or prediction. Dropout, see e.g. \cite{hinton_improving_2012,labach_survey_2019,srivastava_dropout_2014}, randomly ignores nodes or connections with a set probability when training, and has the added benefit of reducing co-adaptation of features.  Pruning, see e.g. \cite{aghasi_fast_2020,zhu_prune_2018}, on the other hand, seeks to force weights with a minimal impact to zero, thereby increasing sparsity within the trained network. Although these methods share the same broad goal as this work, our KDL approach splits layers based on a Kronecker product to form a new architecture, and is a novel means of reducing the required computational resources while maintaining accuracy that is not directly comparable to Dropout, pruning, or other related approaches and extensions.

\subsection{Notation}
A plain lowercase letter $v$ will denote either a scalar or a function, a bold lowercase letter $\bs{v}$ will represent a vector, and a plain uppercase letter $V$ will represent a matrix.  Subscripts will denote indices within a vector or matrix, with a colon denoting \textsc{Matlab}-style slicing of the full range of indices.  Parentheses in the superscript will denote variations based on layer, terms in a summand, and so forth, with multiple such designations separated by commas.  In addition, matrices may be numbered with a single digit in the subscript, in which case indices will be noted in parenthesis following the single digit in the subscript.

\section{Fully Connected Network}

We give a brief introduction to deep neural networks \cite{goodfellow_deep_2016}. An artificial neural network is a function defined by a series of compositions and can be identified by the choice of an activation function $\phi: \R \rightarrow \R$ and a collection of weight matrices and bias vectors. With $n_1$ and $n_L$ the input and output dimensions of the map, respectively, a \textit{fully connected neural network} (FNN) mapping input $\bs{x} \in \R^{n_1}$ to $\bs{y} \in \R^{n_L}$ can be defined as,
\begin{align}\label{eq:fnn}
  \bs{y} &= \bs{a}^{(L)}, & \bs{a}^{(\ell+1)} &= \left(\phi \circ \tilde{h}_{W^{(\ell+1)}, \bs{b}^{(\ell+1)}}\right)\left(\bs{a}^{(\ell)}\right), \hskip 5pt \ell = 1, \ldots, L-1,
\end{align}
where $\bs{a}^{(1)} = \bs{x}$, $L \in \N$ is the number of layers of the network, and $\tilde{h}_{W, \bs{b}}$ is an affine map defined through its weight matrix $W \in \R^{n_\ell \times n_{\ell+1}}$ and bias vector $\bs{b} \in \R^{n_{\ell+1}}$,
\begin{align*}
  \tilde{h}_{W, \bs{b}} &: \R^{n_\ell} \rightarrow \R^{n_{\ell+1}}, & \tilde{h}_{\{W, \bs{b}\}}(\bs{a}) &= W \bs{a} + \bs{b}.
\end{align*}
In \eqref{eq:fnn}, $\phi$ operating on vectors is defined componentwise. With our notation, the $L-2$ intermediate stages $\left\{\bs{a}^{(\ell)}\right\}_{\ell=2}^{L-1}$ are \textit{hidden layers}, and each component of $\bs{a}^{(\ell)}$ is a \textit{node}. We let $n_\ell$, $\ell \in [L]$, denote the number of units in layer $\ell$, so that $\bs{a}^{(\ell)} \in \R^{n_\ell}$.  Using the notation,
\begin{align*}
  \tilde{\theta} &\coloneqq \left\{ W^{(2)}, \bs{b}^{(2)}, \ldots, W^{(L)}, \bs{b}^{(L)}  \right\} \text{ with}\\
  \tilde{\theta}^{(\ell)} &\coloneqq \left\{ W^{(\ell)}, \bs{b}^{(\ell)} \right\},
\end{align*}
the FNN input-to-output map is then,
\begin{align}\label{eq:htilde-def}
  \bs{y} &= \bs{y}(\bs{x}; \tilde{\theta}) = \left( \phi \circ \tilde{h}_{\tilde{\theta}^{(L)}} \circ \phi \circ \tilde{h}_{\tilde{\theta}^{(L-1)}} \cdots \phi \circ \tilde{h}_{\tilde{\theta}^{(2)}} \right) (\bs{x}), & \tilde{h}_{\tilde{\theta}^{(\ell)}} &= \tilde{h}_{\{W^{(\ell)}, \bs{b}^{(\ell)}\}}.
\end{align}
We will focus on the fixed-model capacity neural network setup where the architectural parameters $L$ and $\{n_\ell\}_{\ell \in [L]}$, along with the activation function $\phi$ are fixed before training. Some popular choices of activation function include the hyperbolic tangent, the sigmoid function, a rectified linear unit, and a linear map. The weight matrices $W^{(\ell)}$ and bias vectors $\bs{b}^{(\ell)}$ are updated through optimization-based training; we seek to choose $\tilde{\theta}$ to minimize an $\ell^2$-type loss function that balances model complexity $R(\tilde{\theta})$ against fidelity to available training data $(\bs{x}_m, \bs{y}_m)_{m \in [M]}$,
\begin{align*}
  {\mathcal{L}}\left(\tilde{\theta}\right) = \sum_{m=1}^M \frac{1}{2}{\mathcal{L}}_m\left(\tilde{\theta}\right) + \frac{\lambda}{2} R\left(\tilde{\theta}\right) = \sum_{m=1}^M \frac{1}{2}\left\| \bs{y}(\bs{x}_m) - \bs{y}_m \right\|_2^2 + \frac{\lambda}{2} R\left(\tilde{\theta}\right)
\end{align*}
where $\lambda > 0$ is a tunable hyperparameter. 
In this paper, we choose $R$ as a Tikhonov-type regularization,
\begin{align}\label{eq:R-def}
  R\left(\tilde{\theta}\right) = \sum_{\ell=2}^L \left(\left\|W^{(\ell)}\right\|_F^2 + \left\|\bs{b}^{(\ell)}\right\|_2^2 \right).
\end{align}
Minimization of ${\mathcal{L}}$ over the optimization variables $\tilde{\theta}$ proceeds typically with first-order or quasi-Newton methods, so that computation of $\frac{\partial {\mathcal{L}}}{\partial \tilde{\theta}}$ is required. Practical algorithms achieve this through back-propagation, summarized by the iteration,
\begin{align*}
  W^{(L+1)T} \bm{\delta}_m^{(L+1)} &\coloneqq \bs{y}(\bs{x}_m) - \bs{y}_m, & 
  \bm{\delta}^{(\ell)} &= \phi'(W^{(\ell)}\bm{a}^{(\ell-1)}+\bm{b}^{(\ell)})\circ W^{(\ell+1)T} \bm{\delta}^{(\ell+1)},
\end{align*}
for $\ell = L, \ldots, 2$, where $\circ$ between vectors denotes a componentwise (Hadamard) product, $\phi': \R \rightarrow \R$ is the derivative of $\phi$, and application to vectors is again defined componentwise.  This results in the relations,
\begin{align*}
  \frac{\partial {\mathcal{L}}_m}{\partial W^{(\ell)}} &= \bs{\delta}^{(\ell)} \bs{a}^{(\ell)T}, & 
  \frac{\partial {\mathcal{L}}_m}{\partial \bs{b}^{(\ell)}} &= \bs{\delta}^{(\ell)}.
\end{align*}
To train the network, e.g., with simple gradient descent and fixed learning rate $\eta$, we implement the update,
\begin{align*}
  W^{(\ell)} &\gets W^{(\ell)} - \eta \sum_{m=1}^M \frac{\partial {\mathcal{L}}_m}{\partial W^{(\ell)}} - \eta \lambda W^{(\ell)}, & 
  \bs{b}^{(\ell)} &\gets \bs{b}^{(\ell)} - \eta \sum_{m=1}^M \frac{\partial {\mathcal{L}}_m}{\partial \bs{b}^{(\ell)}} - \eta \lambda b^{(\ell)}.
\end{align*}
In practice, more sophisticated optimization algorithms are used, e.g., \cite[Chapter 8]{goodfellow_deep_2016}. In all the expressions above, application of matrix-vector multiplications involving $W^{(\ell)}$ can form a substantial portion of the computational burden, especially if the hidden layers have large dimension $n_\ell$. In this manuscript, we seek to alleviate this burden while retaining model capacity.

\section{The Kronecker Product}
As matrix and vector sizes increase, matrix-vector operations require more computational resources and time; the Kronecker product \cite{van1993} is one strategy to ameliorate this complexity when the matrices involved have a certain type of exploitable structure. Given $L \in \R^{m_1 \times n_1}$ and $R \in \R^{m_2 \times n_2}$, the Kronecker product (KP) $L \otimes R$ is defined as,
\begin{align}\label{eqn:Kron2}
K \coloneqq L\otimes R = \left[\begin{array}{ccc} l_{11}R & \cdots & l_{1n_1}R\\
\vdots & \ddots & \vdots\\
l_{m_11}R & \cdots & l_{m_1n_1}R\end{array}\right] \in \R^{m_1 m_2 \times n_1 n_2}.
\end{align} 
Given $\bs{x} \in \R^{n_1 n_2}$, computing $K \bs{x}$ can be accomplished via the relation,
\begin{align}\label{eq:kron-rearrangement}
  K\bs{x} &= R X L^T, & X &= \mathrm{mat}(\bs{x}),
\end{align}
where $\mathrm{mat}: \R^{n_1 n_2} \rightarrow \R^{n_2 \times n_1}$ is a matricization operation, the inverse of vectorization $\mathrm{vec}: \R^{n_2 \times n_1} \rightarrow \R^{n_1 n_2}$. The major appeal of the above representation is that $\bs{x} \mapsto K \bs{x}$ requires $\mathcal{O}(m_1 m_2 n_1 n_2)$ operations, whereas $X \mapsto R X L^T$ requires only $\mathcal{O}(n_2 n_1 m_2 + n_1 n_2 m_1)$ operations, which can result in substantial computational savings.

While many matrices cannot be represented exactly as a Kronecker product, all matrices whose row and column dimensions are not prime integers can be approximated by a sum of Kronecker product matrices. To explain this further, let $W \in \R^{m \times n}$ be given, with $m = m_1 m_2$ and $n = n_1 n_2$ arbitrary integer factorizations of $m$ and $n$. A rank-$k$ KPD approximation of $W$ is given by
\begin{align*}
  W &\approx \sum_{j=1}^k L^{(j)} \otimes R^{(j)}, & L^{(j)} \in \R^{m_1 \times n_1}, \; R^{(j)} \in \R^{m_2 \times n_2},
\end{align*}
for some choice of matrices $L^{(j)}, R^{(j)}$. Note that the phrase \textit{rank-$k$ KPD approximation} is a slight abuse of terminology, as $L^{(j)} \otimes R^{(j)}$ is not necessarily a rank-$1$ matrix. A Frobenius-norm optimal rank-$k$ KPD approximation can be determined by introducing a \textit{rearrangement} $\mathcal{R}$ of $W$, identified through a block partition of $W$,
\begin{align*}
W = \left[\begin{array}{ccc} W_1 & \cdots & W_{m_1(n_1-1)+1}\\
\vdots & \ddots & \vdots\\
W_{m_1} & \cdots & W_{m_1n_1}\end{array}\right], \hskip 5pt W_{j} \in \R^{m_2 \times n_2} \hskip 5pt
  \stackrel{\bs{w}_{j} = \mathrm{vec}(W_{j}) }{\Longrightarrow} 
  \mathcal{R}(W) = \left[\begin{array}{c} \bm{w}_1^T\\ \bm{w}_2^T\\  \vdots\\ \bm{w}_{m_1n_1}^T\end{array}\right].
\end{align*}
With $\bs{l} = \mathrm{vec}(L)$ and $\bs{r} = \mathrm{vec}(R)$, then note that $\mathcal{R}(L \otimes R) = \bs{l} \bs{r}^T$, and therefore \linebreak $\left\| W - L \otimes R\right\|_F = \left\| \mathcal{R}(W) - \bs{l}\bs{r}^T \right\|_F$. This last relation allows one to leverage the Hilbert-Schmidt-Eckart-Young theorem on Frobenius-norm optimal low-rank approximations using the singular value decomposition \cite[Theorem~2.4.8]{Golub}. While this immediately yields an optimal rank-1 KPD approximation \cite[Corrolary~2.2]{van1993}, the result is generalizable to rank-$k$ approximations for $k > 1$. 
\begin{lemma}[\cite{van2000}]\label{lemma:kron-approx}
  Given $W\in\mathbb{R}^{m_1m_2\times n_1n_2}$, let its rearrangement have singular value decomposition (SVD) $\mathbb{R}^{m_1n_1\times m_2n_2}\ni\mathcal{R}(W) = U\Sigma V^T = \sum_{i=1}^r \sigma_i \bm{u}_i \bm{v}_i^T$, with the singular values $\{\sigma_i\}_{i \in [r]}$ arranged in non-increasing order and $r = \mathrm{rank}(\mathcal{R}(W))$. Let $W_k$ denote a rank-$k$ KPD defined by,
  \begin{align*}
    W_k \coloneqq \sum_{j=1}^k \mathrm{mat}(\sigma_i \bs{u}_i) \otimes \mathrm{mat}(\bs{v}_i).
  \end{align*}
  Then $W_k$ is an optimal rank-$k$ KPD approximation to $W$:
  \begin{align*}
    W_k &\in \argmin_{L^{(j)} \in \R^{m_1 \times n_1}, R^{(j)} \in \R^{m_2 \times n_2}, j \in [k]} \left\| W - \sum_{j=1}^k L^{(j)} \otimes R^{(j)} \right\|_F^2, \\
    \left\| W - W_k \right\|_F^2 &= \sum_{i=k+1}^r \sigma_i^2.
  \end{align*}
\end{lemma}

\section{KP Dual Layer Networks}

This section introduces our new neural network architecture. Given an FNN with layer $\ell$ nodal states $\bs{a}^{(\ell)}$ defined by \eqref{eq:fnn}, a KPD $W^{(\ell)} = \sum_{i=1}^r L^{(\ell,i)T}\otimes R^{(\ell,i)}$, with $r = \mathrm{rank}(\mathcal{R}(W^{(\ell)}))$ could be used to produce a new intermediate layer utilizing KP-based multiplication. I.e., we have the exact representation,
\begin{align}\label{eq:fnn-kp-representation}
\bs{a}^{(\ell)} = \tilde{h}_{\tilde{\theta}^{(\ell)}}(\bs{a}^{(\ell-1)}) = \text{vec}\left(\sum_{i=1}^r R^{(\ell,i)}A^{(\ell-1)}L^{(\ell,i)}\right) + \bm{b}^{(\ell)}, 
\end{align}
where $A^{(\ell-1)} = \text{mat}\left(\bm{a}^{(\ell-1)}\right)$ is reshaped to form a matrix with appropriate dimensions. A straightforward approach is then to truncate the sum to $k < r$ terms. However, the error incurred by such an approach is bounded by the truncated singular values of the weight rearrangement, and such singular values are not guaranteed to decay quickly, in turn requiring large $k$ to accurately represent $W^{(\ell)}$.  Figure~\ref{fig:comp} demonstrates this for examples that will be introduced in Table~\ref{Tab:times}.  For each of these examples, a rank close to the full rank of the rearrangement of the weight matrices is needed to achieve a test error similar to the test error when using the trained weight matrices, making direct KPD-based truncation unattractive.

\begin{figure}[!htbp]\begin{center}
\subfloat[BSD (a) 
\label{bikes8_comp}]{\includegraphics[width=0.25\textwidth]{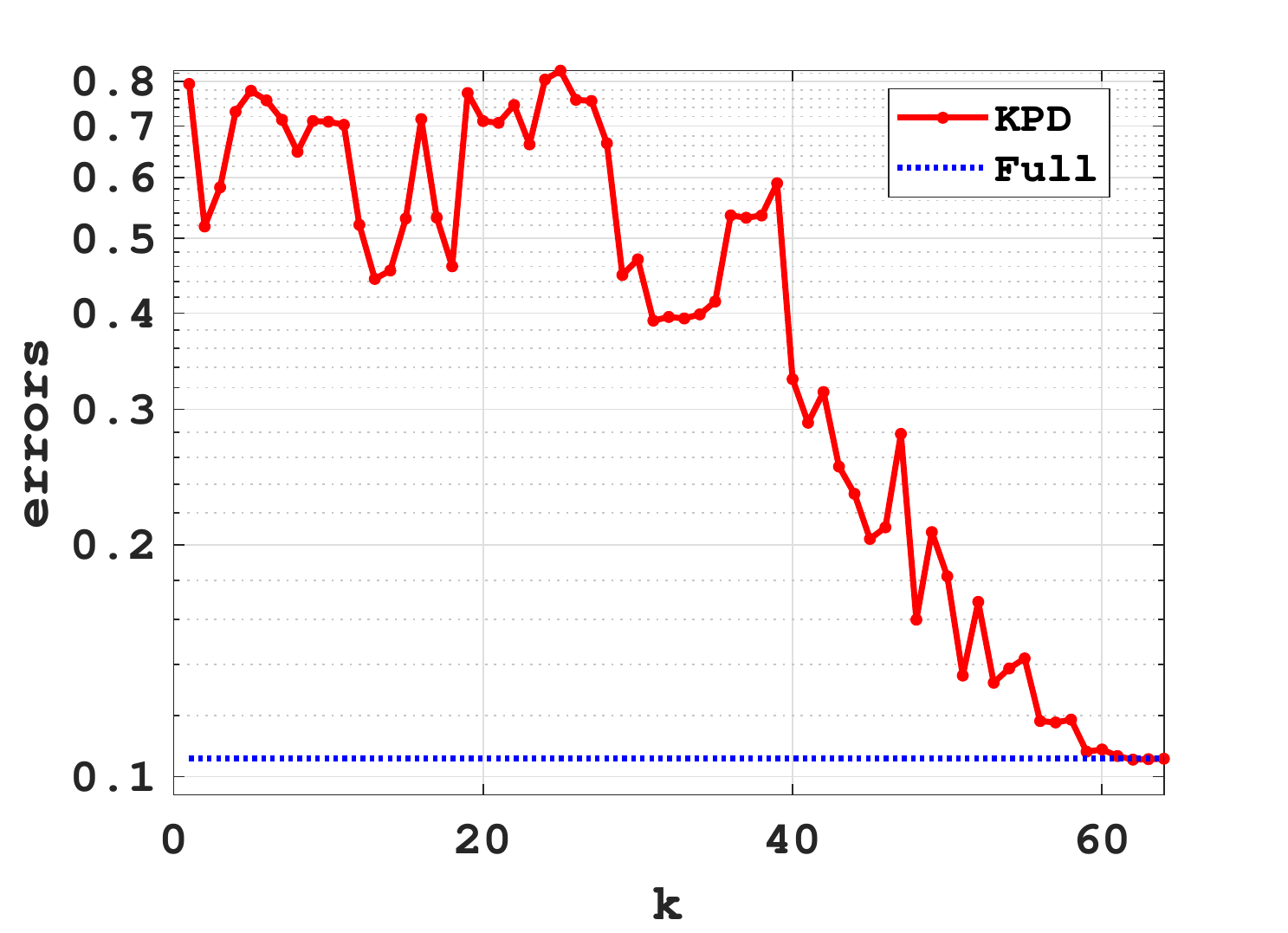}}
\subfloat[BSD (b)  
\label{bikes_comp}]{\includegraphics[width=0.25\textwidth]{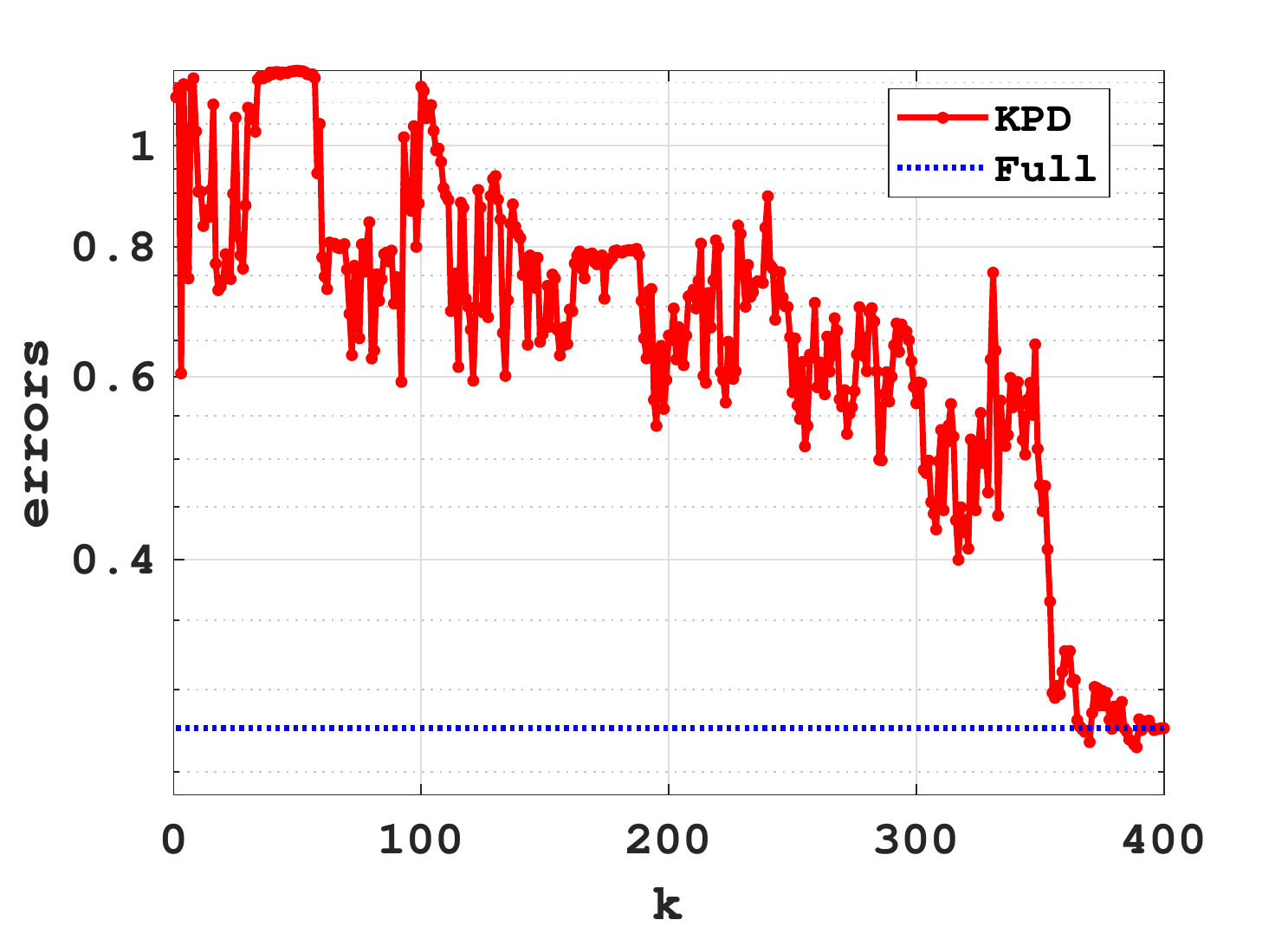}}
\subfloat[BF  
\label{blog_comp}]{\includegraphics[width=0.25\textwidth]{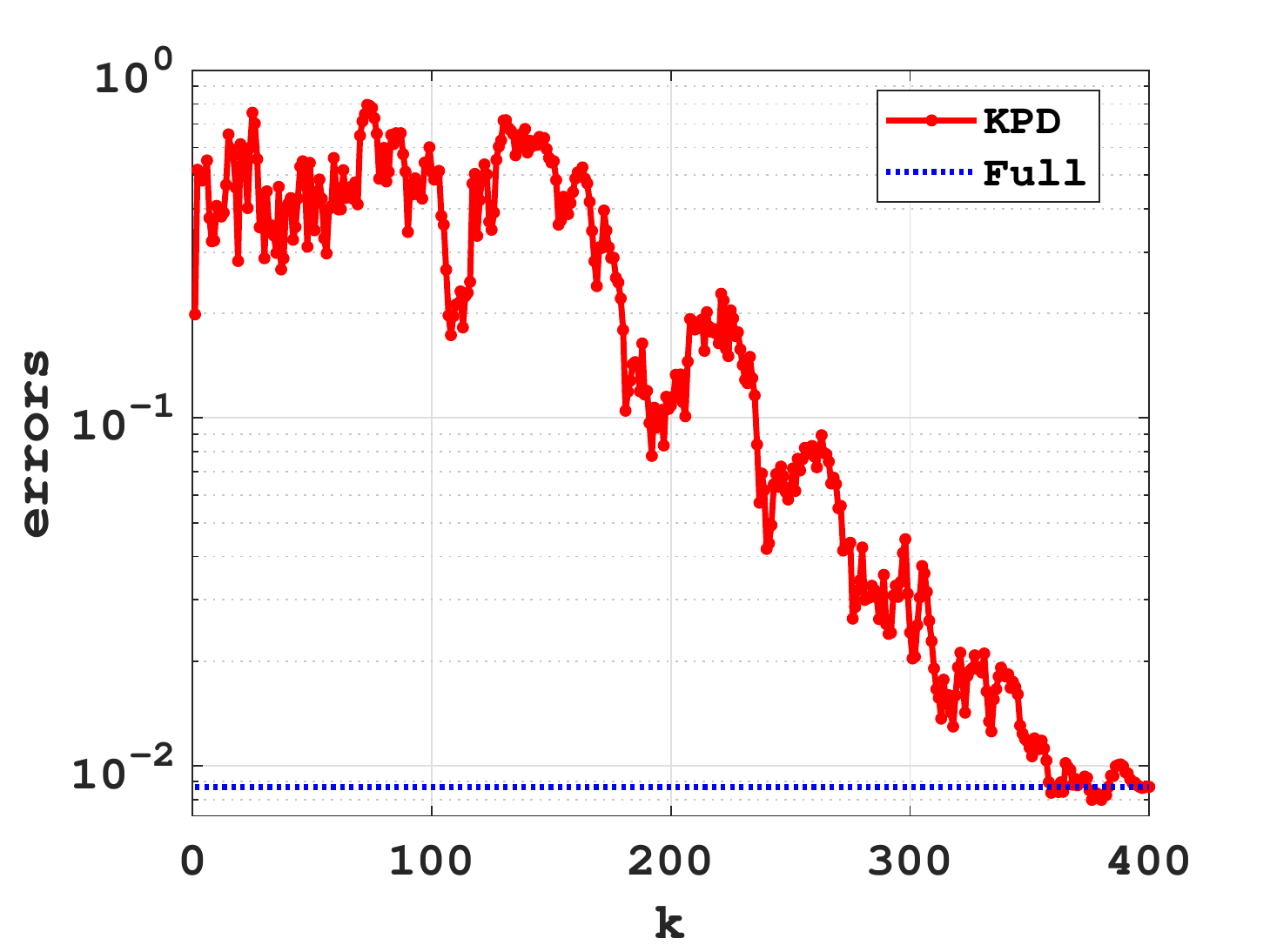}}
\subfloat[MNIST  
\label{MNIST_comp}]{\includegraphics[width=0.25\textwidth]{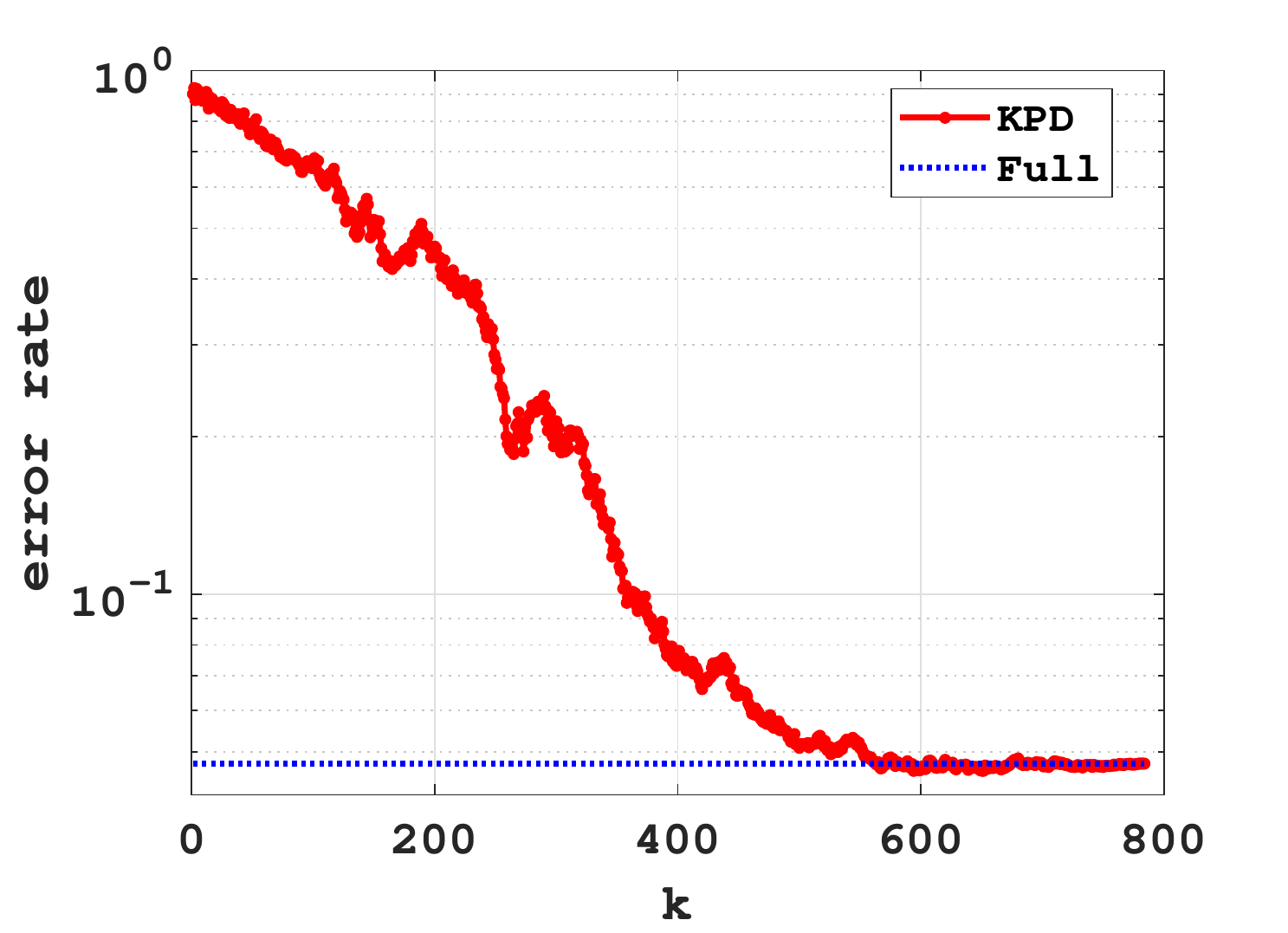}}
\end{center}
\caption{Figures \ref{bikes8_comp} to \ref{MNIST_comp} show the test errors across choice of KPD rank, $k$, compared to the full rank of the rearrangement of each weight matrix for BSD (a and b), BF, and MNIST respectively as defined in Table~\ref{Tab:times}. 
\label{fig:comp}}
\end{figure}

To balance this potential loss of accuracy caused by reducing the Kronecker rank, we augment model capacity via inclusion of an extra activation function and bias vector at the new intermediate layer. The ``intermediate layer'' is the result of performing only one matrix multiplication in the KPD matrix-vector multiplication, e.g., it is the state $A^{(\ell-1)} L^{(\ell,i)}$ in \eqref{eq:fnn-kp-representation}.
With $A_R^{(\ell)}$ the new matrix-valued intermediate state, and with $\phi_1$ and $\phi_2$ activation functions, then we define a rank-$k$ KP dual layer (KDL) as the operation, 
\begin{align*}
  A_R^{(\ell)} = \left( \phi_1 \circ h_{\theta_R^{(\ell)}} \circ \phi_2 \circ \widetilde{h}_{\theta_L^{(\ell)}} \right) \left(A_R^{(\ell-1)}\right),
\end{align*}
where we have introduced new layer-$\ell$ parameters $\theta_L^{(\ell)}$ and $\theta_R^{(\ell)}$,
\begin{align}\nonumber
\theta_L^{(\ell)} &\coloneqq \left\{ W_L^{(\ell,1)},B_L^{(\ell,1)},\cdots,W_L^{(\ell,k)},B_L^{(\ell,k)} \right\},\\\nonumber
  \theta_R^{(\ell)} &\coloneqq \left\{ W_R^{(\ell,1)},B_R^{(\ell,1)},\cdots,W_R^{(\ell,k)},B_R^{(\ell,k)} \right\},\\\label{eq:theta-def}
\theta &\coloneqq \left\{ \theta_L^{(2)}, \theta_R^{(2)}, \cdots, \theta_L^{(L)},\theta_R^{(L)} \right\}.
\end{align}
The function $\tilde{h}_{\theta_L^{(\ell)}}$ is an extension of the affine function introduced in \eqref{eq:htilde-def}\footnote{$\widetilde{h}$ in \eqref{eq:htilde-def} takes a single weight matrix and bias vector as parameters, but $\widetilde{h}$ here takes $k$ weight matrices and bias vectors as parameters.}, and the newly introduced functions $h_\theta$, with a slight abuse in notation, are $k$-fold sums / collections of FNN affine maps $\widetilde{h}$,
\begin{align}
  \widetilde{h}_{\theta_L^{(\ell,i)}}\left( A_R^{(\ell-1)} \right) &= A_R^{(\ell-1)}W_L^{(\ell,i)} + B_L^{(\ell,i)},\nonumber\\
  h_{\theta_L^{(\ell)}}\left( A_R^{(\ell-1)} \right) &= \left(\widetilde{h}_{\theta_L^{(\ell,1)}}\left( A_R^{(\ell-1)} \right),  \ldots, \widetilde{h}_{\theta_L^{(\ell,k)}}\left( A_R^{(\ell-1)} \right)\right)\nonumber\\
  A_L^{(\ell,i)} &= \left(\phi_2 \circ \widetilde{h}_{\theta_L^{(\ell,i)}}\right)\left( A_R^{(\ell-1)} \right),\label{Eqn:KPlayer1}\\ 
  h_{\theta_R^{(\ell)}}\left( \left\{A_L^{(\ell,i)}\right\}_{i=1}^k \right) &= \sum_{i=1}^k \widetilde{h}_{\theta_R^{(\ell,i)}}\left(A_L^{(\ell,i)}\right) = \sum_{i=1}^k \left(W_R^{(\ell,i)} A_L^{(\ell,i)} + B_R^{(\ell,i)}\right), \nonumber\\
A_R^{(\ell)} &= \left( \phi_1 \circ h_{\theta_R^{(\ell)}} \circ \phi_2 \circ h_{\theta_L^{(\ell)}} \right) \left(A_R^{(\ell-1)}\right), \label{Eqn:KPlayer2}
\end{align} 
effectively splitting the layer into two intermediate states $A_L^{(\ell,i)}$, $i \in [k]$, and $A_R^{(\ell)}$.  

In what follows, we will use the notation,
\begin{subequations}\label{eq:f-def}
\begin{align}\label{eq:fl-def}
  f_\ell (A) = \left( \phi_1 \circ h_{\theta_R^{(\ell)}} \circ \phi_2 \circ \tilde{h}_{\theta_L^{(\ell)}} \right) (A),
\end{align}
to denote the function that transitions the nodes at layer $\ell-1$ to those at layer $\ell$ for $\ell \geq 2$. 
Note that reshaping between layers is only necessary if the output dimensions from one layer do not match the input dimensions from the next layer. For notational convenience in our analysis, we define,
\begin{align}\label{eq:f1-def}
  f_1(\cdot) = X = \mathrm{mat}(\bs{x}),
\end{align}
\end{subequations}
and in particular we have $f_1(0) = X$, where $\bs{x}$ is the input to the neural network. A higher order splitting into Kronecker multi-layers (KML) is discussed in Section 2 in the supplemental materials, and is hierarchical in nature.  While KDLs may be applied only to a subset of layers in an FNN, of interest here is a neural network based solely on KDLs, i.e., a KDL-NN, where intermediate layers are left in matrix form (instead of vectorizing). 
The KDL-NN is now given by 
\begin{align}\label{eq:kdl-nn}
  Y_k^{\kappa}(X) = Y_k^\kappa(X;\theta) = \left( f_L \circ f_{L-1} \cdots \circ f_2 \right) (X).
\end{align}
The associated regularized loss for training a rank-$k$ KDL-NN over all data points $(\bs{x}_m, \bs{y}_m)$ for $m=1, \ldots, M$, is given by,
\begin{align*}
  \mathcal{L}(\theta) = \sum_{m=1}^M \frac{1}{2} \mathcal{L}_m^\kappa(\theta) + \frac{\lambda}{2} R^\kappa(\theta) = \sum_{m=1}^M \left\| \mathrm{mat}(\bs{y}_m) - Y_k^\kappa(\mathrm{mat}(\bs{x}_m)) \right\|^2_F + \frac{\lambda}{2} R^\kappa(\theta),
\end{align*}
where $R^\kappa$ is the same Tikhonov-type regularizer as in \eqref{eq:R-def}, but sums the squared Frobenius norms for all weights and bias matrices contained in $\theta$.

\subsection{KDL-NN architecture}
In what follows, we use notation that describes node configuration as a sequence of numbers that count nodes in each layer, e.g., the FNN in Figure~\ref{full} corresponds to ${4|9|9|4}$, describing a fully connected network with 4 input nodes, 2 hidden layers with 9 nodes each, and an output layer with 4 nodes.  A rank-$1$ KDL-NN may be described by using pairs of values for each layer corresponding to the matricized shape of the nodes, such as ${(2,2) | (3,3) | (3,3) | (2,2)}$, where the input and output layers with 4 nodes are shaped as $2\times 2$ matrices and the two hidden layers with 9 nodes each are shaped as $3\times 3$ matrices; the layers added by the KDL-NN compared to the FNN architecture correspond to matrices of size $2\times 3$, $3\times 3$ and $3\times 2$, respectively. We denote a rank-$k$ KDL-NN with similar notation, ${(2,2) |^k (3,3) |^k (3,3) |^k (2,2)}$, indicating that $k$ factors of each of the $2 \times 3$, $3 \times 3$, and $3 \times 2$ added intermediate layers are present. (A bar $|$ without a superscript indicates $k=1$.) In our numerical examples, we provide a comparison against a third type of network, an ``extended'' FNN, labeled E-FNN, corresponding to the same number of nodes as the KDL-NN, but with full connections between nodes in sequential layers. An E-FNN corresponding to the same number of nodes as the KPD-NN architecture ${(2,2) | (3,3) | (3,3) | (2,2)}$ then corresponds to ${4|6|9|9|9|6|4}$.

We visually compare the architecture of the above FNN, KDL-NN, and E-FNN examples in Figure~\ref{fig:nodes} for ranks $k=1$ and $k=2$. The KDL-NN sparsifies some connections of the FNN while adding more nodes. Alternatively, the KDL-NN is a type of structured connection sparsification of the E-FNN. Note that we provide comparisons against E-FNNs as a baseline for model capacity; our goal is not to computationally sparsify E-FNNs, but rather to build new architectures based on simpler FNNs. Table~\ref{Tab:connects} summarizes architectures described above and reports the total number of trainable parmeters for the networks shown in Figure~\ref{fig:nodes}.  Note that KDL-NNs are fully connected across each set of dual layers.

\begin{figure}[!htbp]\begin{center}
  \resizebox{\textwidth}{!}{
  \renewcommand{\tabcolsep}{0pt}
  \begin{tabular}{ccc}
    \multirow{2}{*}[4em]{\subfloat[FNN (175 parameters) \label{full}]{\includegraphics[width=0.33\textwidth]{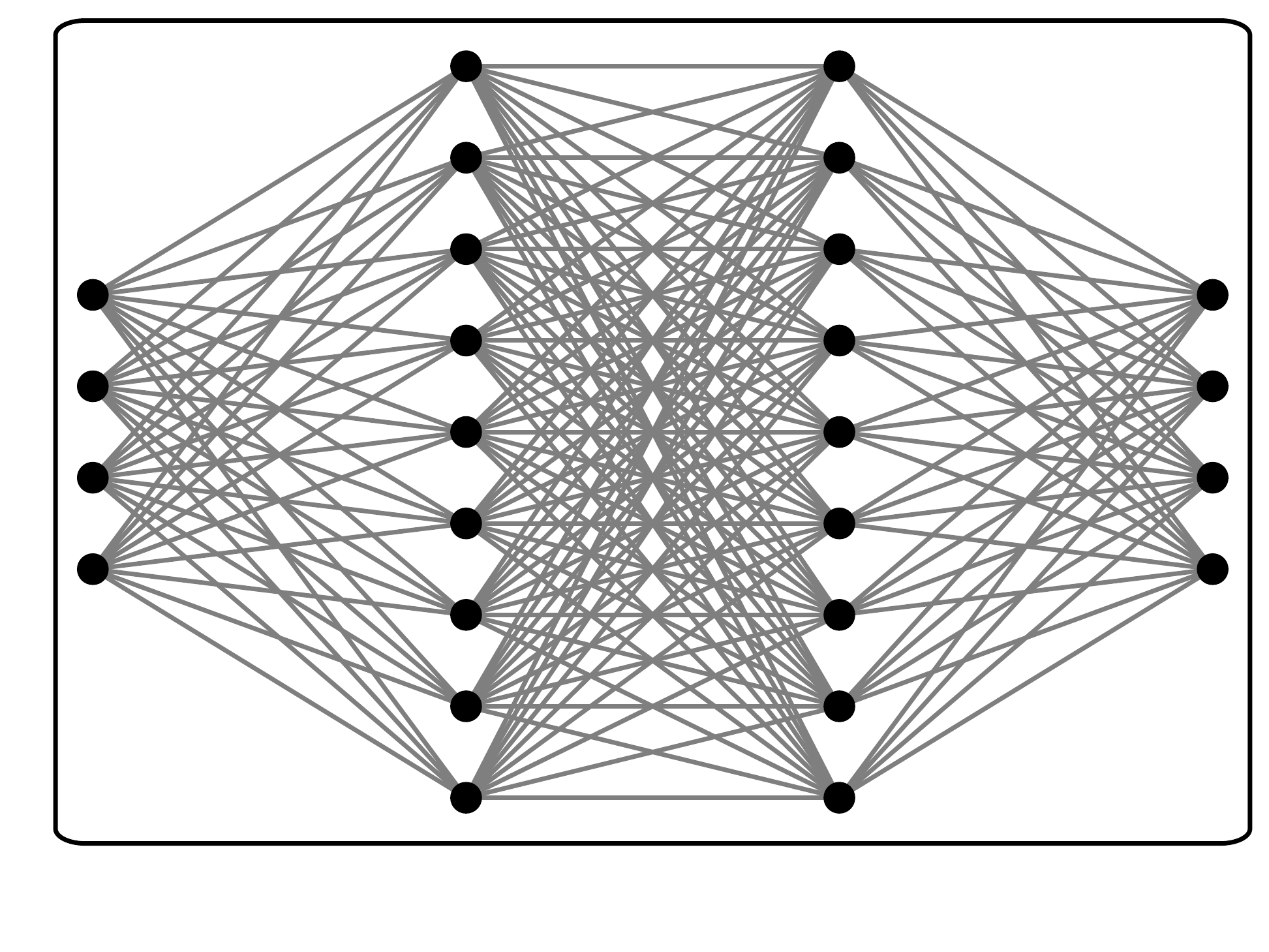}}} & 
    \subfloat[KDL-NN (157 parameters) \label{KP}]{\includegraphics[width=0.33\textwidth]{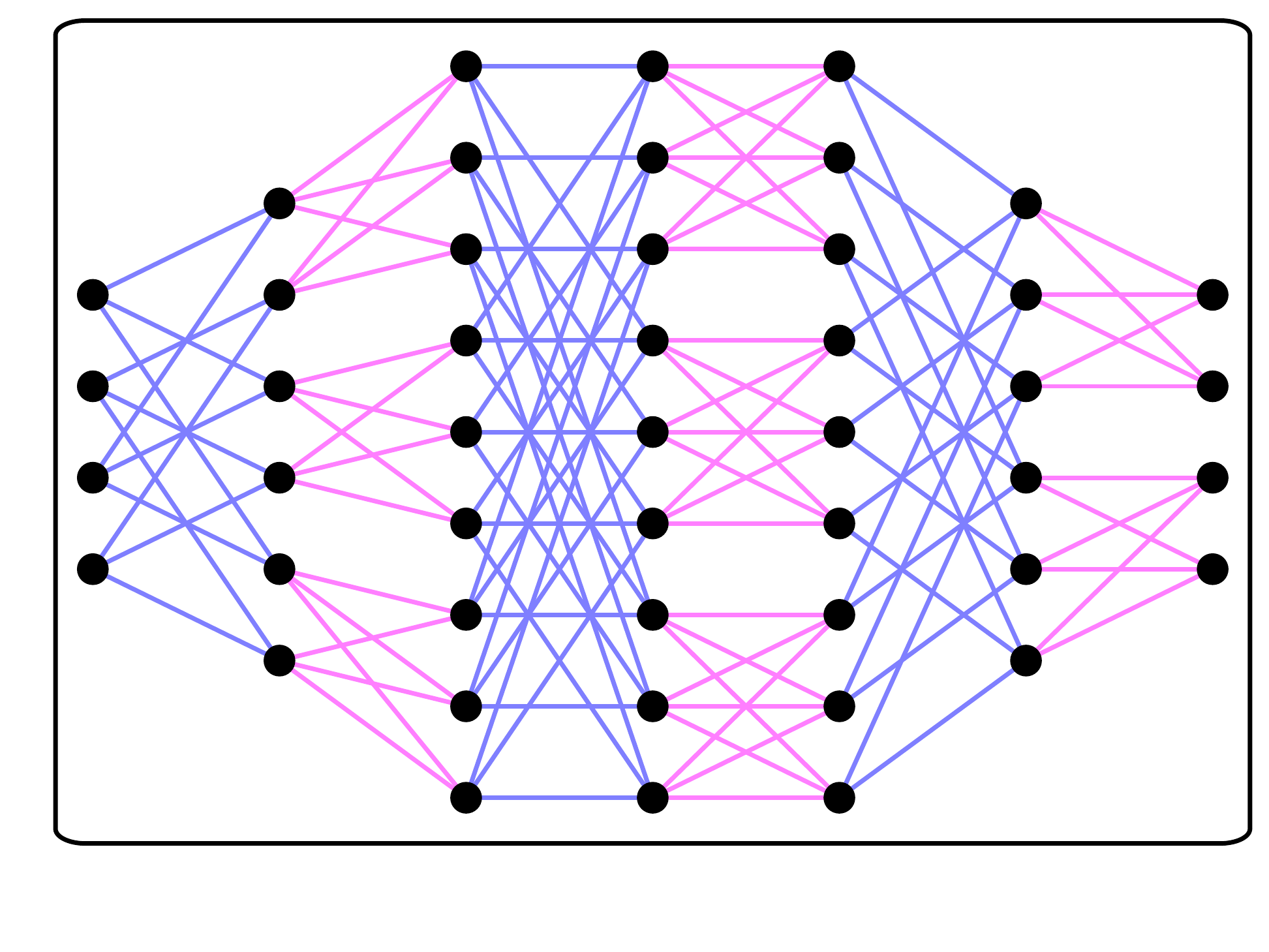}} & 
  \subfloat[E-FNN (361 parameters) \label{Extended}]{\includegraphics[width=0.33\textwidth]{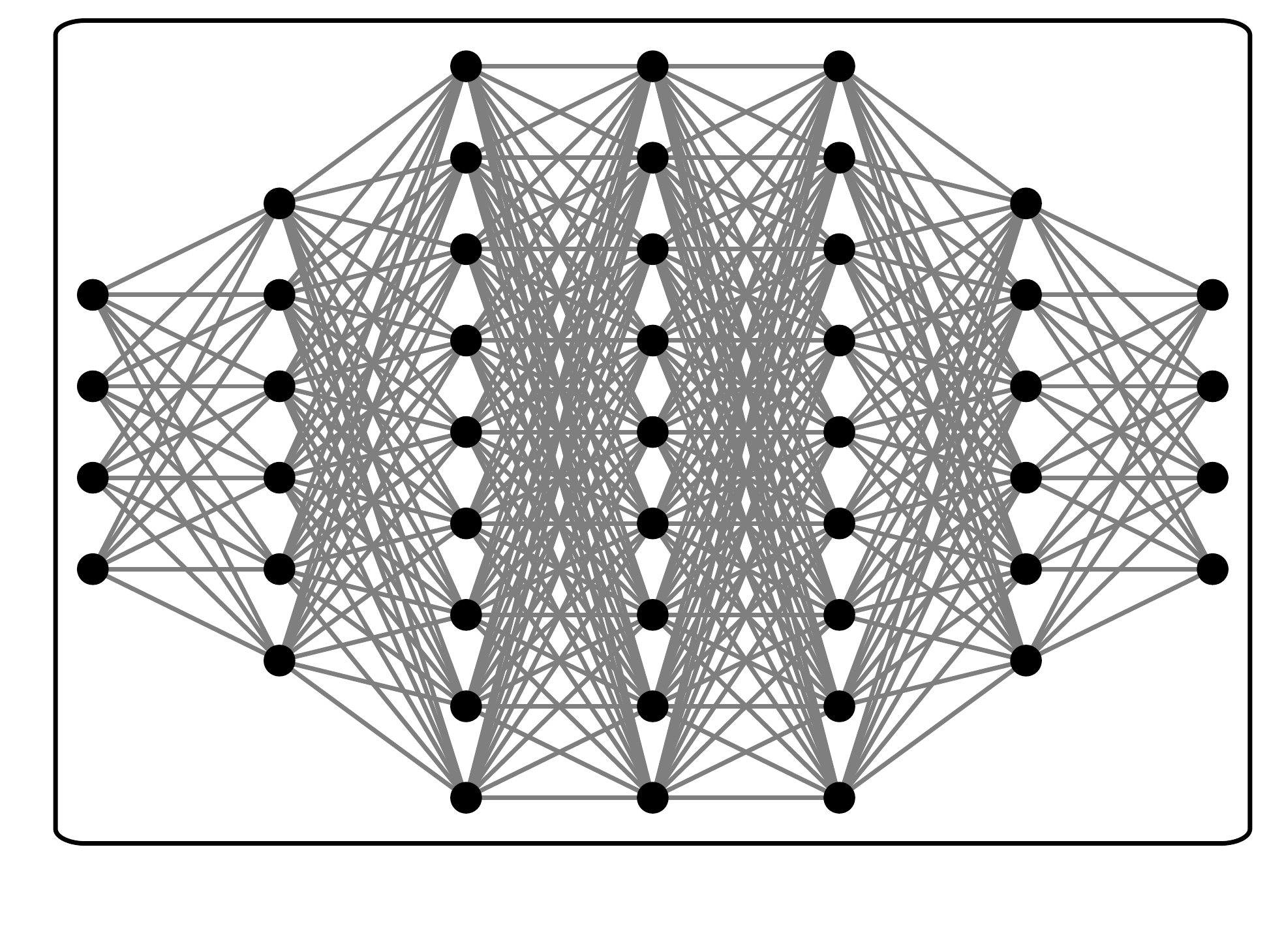}} \\
    & 
    \subfloat[{\medmuskip=0mu \thinmuskip=0mu \thickmuskip=0mu KDL-NN, $k=2$ (292 parameters)} \label{KP2}]{\includegraphics[width=0.33\textwidth]{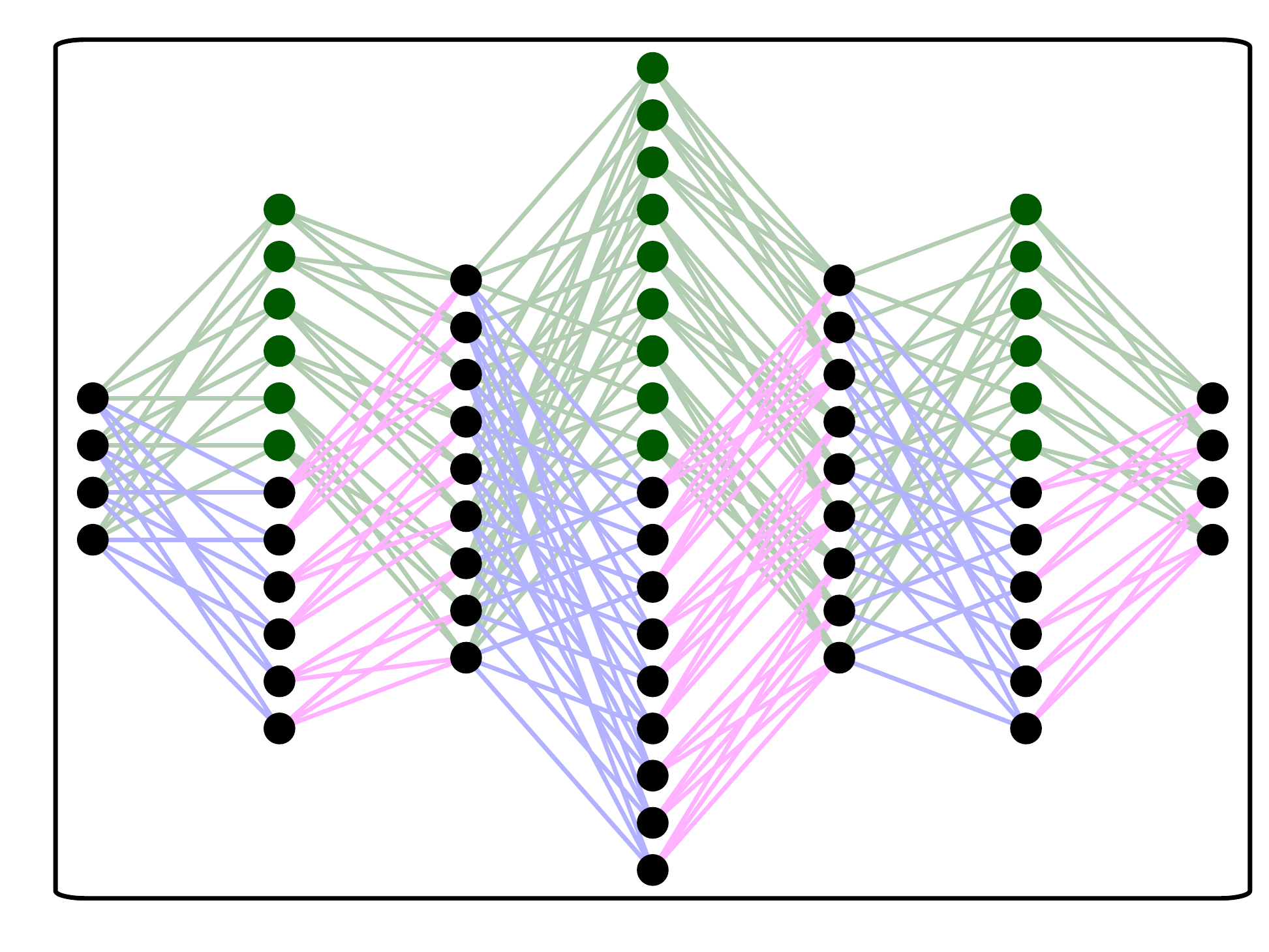}} 
    &
    \subfloat[E-FNN, {\medmuskip=0mu \thinmuskip=0mu \thickmuskip=0mu $k=2$} (700 parameters) \label{Extended2}]{\includegraphics[width=0.33\textwidth]{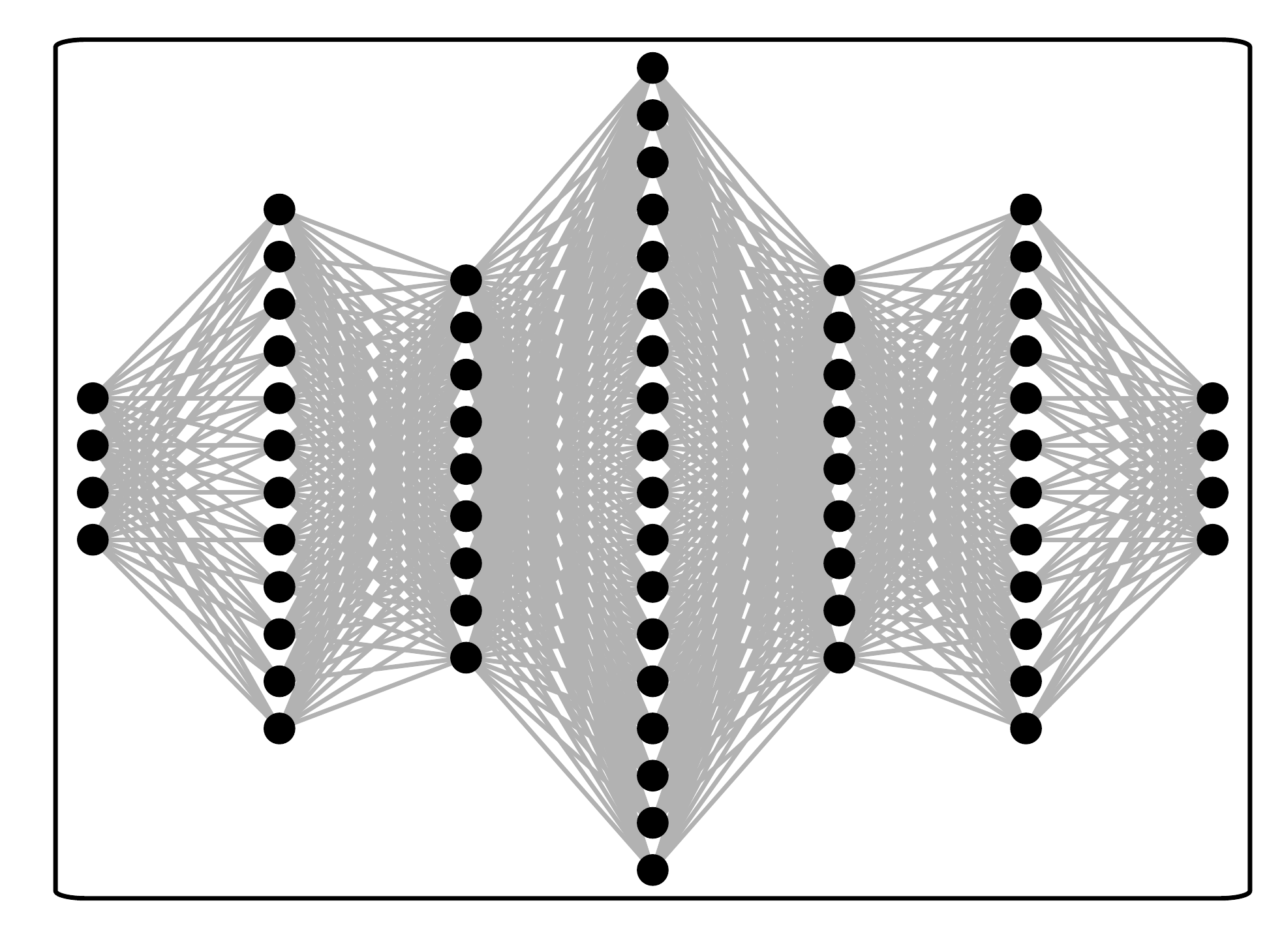}}
  \end{tabular}
}
\end{center}
\caption{Left: A fully connected FNN is shown in Figure~\ref{full}. Top row: Figure~\ref{KP} shows the resulting connections for reshaping the input and output into $2\times 2$ blocks and the hidden layers into $3\times 3$ blocks with intermediate layers added corresponding to a KDL-NN rank~1 formulation, where the alternating pattern represents multiplying from the right shown with blue connections and from the left with maroon connections. The same nodes from the KDL-NN are then used to show the resulting E-FNN in Figure~\ref{Extended}.  Bottom row: corresponding KPD-NN rank~2 formulation (Figure~\ref{KP2}), with the additional nodes and connections plotted in green. The same nodes are then used to show the resulting E-FNN in Figure~\ref{Extended2}.  The architecture and total number of parameters for all networks pictured here are given in Table~\ref{Tab:connects}. \label{fig:nodes}}
\end{figure}

\begin{table}[!htbp]
  \centering 
    \renewcommand{\tabcolsep}{0.4cm}
    \renewcommand{\arraystretch}{1.3}
    {\scriptsize
    \begin{tabular}{@{}rcc@{}}
      \toprule
      Network type & Architecture & Number of parameters \\\midrule
      FNN & $4| 9| 9| 4$ & 175\\[5pt]
      KDL-NN & $(2,2) | (3, 3)| (3, 3)| (2, 2)$ & 157\\
      corresponding E-FNN & $4| 6| 9| 9| 9| 6| 4$ & 361\\[5pt]
      KDL-NN ($k = 2$) & $(2, 2) |^2 (3, 3)|^2 (3, 3)|^2 (2, 2)$ & 292\\ 
      corresponding E-FNN & $4| 12| 9| 18| 9| 12| 4$ & 700\\ 
    \bottomrule
    \end{tabular}
  }
    \renewcommand{\arraystretch}{1}
    \renewcommand{\tabcolsep}{12pt}
  \caption{Architectures and parameter counts for networks shown in Figure~\ref{fig:nodes}.}\label{Tab:connects}
\end{table}

\subsection{Truncation error analysis for KDL-NN architectures}

In this section, we analyze the theoretical performance of \textit{simplified} KDL-NN architectures based on the Kronecker product approximation error estimate in Lemma~\ref{lemma:kron-approx}. Because we use these estimates that are based on a truncated SVD, we expect the performance of a trained KDL-NN to perform much better. In addition, the simplified networks have smaller model capacity than the practical networks we use (as described next). Nevertheless, these results show that KDL-NNs can inherit accuracy from the approximability of fully connected weight matrices by Kronecker sums. 

We recall that our KDL-NN architectures involve two activation functions $\phi_1$ and $\phi_2$ in \eqref{Eqn:KPlayer2}. The \textit{simplified} networks we analyze below take $\phi_2$ as the identity (linear) function, which allows us to directly tie the approximation error to Kronecker sum truncation errors of matrices. In practice (i.e., in our numerical results) we use non-identity activation functions. 

To state our results, we require a norm on the KDL-NN parameters \eqref{eq:theta-def}. Consider $\theta^{(\ell)}$, the KDL-NN parameters associated to layer $\ell$,
\begin{align*}
  \theta^{(\ell)} = \left\{ \theta_L^{(\ell)}, \theta_R^{(\ell)} \right\},
\end{align*}
and introduce a particular function $|\cdot|_\kappa$ operating on such layer-$\ell$ parameters, 
\begin{align*}
  \left|\theta^{(\ell)}\right|_\kappa^2 \coloneqq \sum_{q=1}^k \left\| W_L^{(\ell,q)} \right\|_F \left\| W_R^{(\ell,q)} \right\|_F.
\end{align*}
It is straightforward to show that this function is non-negative but dominated by the function $\theta^{(\ell)} \mapsto \left\| \mathrm{vec}(\theta^{(\ell)}) \right\|_2$, so that $|\cdot|_\kappa$ is weaker than a corresponding standard $\ell^2$ norm on the vectorized parameters. 

We assume that an FNN architecture is given, and that an associated KDL-NN architecture is prescribed (e.g., so that we have a well-defined matricization of input/output vectors). Our results are relative to a \textit{trained} FNN $\bs{y}$ in \eqref{eq:htilde-def}, which we express in matricized form:
\begin{align*}
  Y &= Y(X), & Y &= \mathrm{mat}(\bs{y}), & X &= \mathrm{mat}(\bs{x}).
\end{align*}
This trained FNN has weight matrices $W^{(\ell)}$. By constructing a KDL-NN whose parameters correspond to approximating $W^{(\ell)}$ via Kronecker product sums, one expects that the classical Kronecker sum bounds in Lemma~\ref{lemma:kron-approx} can be leveraged for error estimates in the KDL-NN case. Our results appear in terms of $\ell^2$ norms of truncated singular values from a Kronecker product rearrangement: 
\begin{align}\label{eq:kdl-truncation}
  \epsilon^{(\ell,k)} \coloneqq \sqrt{\sum_{i=k+1}^{r^{(\ell)}} (\sigma_i^{(\ell)})^2 } \stackrel{\textrm{Lemma } \ref{lemma:kron-approx}}{=} \min_{L^{(j)} \in \R^{m_1 \times n_1}, R^{(j)} \in \R^{m_2 \times n_2}, j \in [k]} \left\| W^{(\ell)} - \sum_{j=1}^k L^{(\ell,j)} \otimes R^{(\ell,j)} \right\|_F^2
\end{align}
where $\left\{\sigma_i^{(\ell)}\right\}_{i \in r^{(\ell)}}$ are the ordered singular values of the rank-$r^{(\ell)}$ rearrangement of the FNN weight matrix $W^{(\ell)}$. 
Our main technical result characterizes errors for a simplified KDL-NN relative to an FNN using the ``norms'' $\left|\theta^{(\ell)}\right|_\kappa$ and the Kronecker rank truncation parameters $\epsilon^{(\ell)}$.
\begin{theorem}\label{thm:main-result}
  Suppose a trained FNN $Y$ is given, and consider a rank-$k$ KDL-NN network $Y^\kappa_k$ in \eqref{eq:kdl-nn} where we choose the inner activation function $\phi_2(x) \equiv x$, and assume that $\phi_1$ is $c_1$-Lipschitz for some $c_1 > 0$. Then given a data pair $(\bs{x}_m, \bs{y}_m)$ with corresponding matricization $(X_m, Y_m)$, training $Y^\kappa_k$ over the KDL-NN parameters $\theta$ yields,
\begin{align}\label{eq:main-result}
  \argmin_{\theta} \left\| Y(X) - Y_k^\kappa(X;\theta) \right\|_F \leq \sum_{i=2}^L \epsilon^{(i,k)} \left(\prod_{j=i+1}^L \left|\theta^{(j)}\right|_{\kappa}^2 \right) \sum_{k=1}^{i-1} c_1^{L-k} \left\|f_k(0) \right\|_F.
\end{align}
  where the parameters $\theta^{(\ell)}$ are defined by setting $W_L^{(\ell,i)} = L^{(\ell,i)}$ and $W_R^{(\ell,i)} = R^{(\ell,i)}$, with $L^{(\ell,i)}$ and $R^{(\ell,i)}$ defined as in \eqref{eq:kdl-truncation} and $f_1$ is as defined in \eqref{eq:f1-def}.
\end{theorem}
Under some additional assumptions, the above can be simplified to more clearly reveal the components of the bound.
\begin{corollary}
  Suppose $f_{\ell}(0) = 0$ for $\ell= 2, \ldots, L$ and the activation function $\phi_1(\bm{x}) = \phi(x)$ is 1-Lipschitz (ReLu, Tanh, and Sigmoid are such examples \cite{Scaman}), then Theorem~\ref{thm:main-result} holds with
\begin{equation*}
\argmin_{\theta} \left\| Y(X) - Y_k^\kappa(X;\theta) \right\|_F \leq \left( \sum_{i=2}^L \epsilon^{(i,k)} \prod_{j=i+1}^L \left|\theta^{(j)}\right|_{\kappa}^2 \right) \left\| X \right\|_F.
\end{equation*}
\end{corollary}

We provide the proof of Theorem~\ref{thm:main-result} in Appendix~\ref{app:proof}. This theorem provides a theoretical connection between the size of the KPD truncation errors $\epsilon^{(i,k)}$ of a trained FNN and the predictive performance of a KDL-NN relative to this FNN. This result does \textit{not} immediately translate into a practical error estimate since (a) we have made the simplifying assumption that $\phi_2$ is the identity, and (b) training a KDL-NN does \textit{not} involve KPD truncations from weights of an FNN. We also do not expect this bound to be sharp since its proof (see Appendix~\ref{app:proof}) invokes the triangle inequality several times.

Nevertheless, the components of the estimate in \eqref{eq:main-result} give insight into when we expect KDL-NN approaches to work well: First, if all the $\epsilon^{(i,k)}$, $i = 2, \ldots, L$ are small, then we expect that a KDL-NN can perform at least as well as a corresponding FNN. I.e., when trained weight matrices of an FNN have ``small'' Kronecker rank, we expect KDL-NNs to perform well. The remaining terms can be interpreted as quantities that measure how well-behaved a KDL-NN is. For example, appearance of the $|\theta^{(j)}|_\kappa$ functions indicates that the size of the weight matrices affects performance, and $f_k(0)$ is the output of a layer-$k$ KDL function with zero input. Note in particular that $f_1(0) = X$, so that the norm of the input $X$ to the KDL-NN affects the bound, as expected.

\subsection{Numerical Cost of Forward Operations and Back-Propagation}\label{ssec:cost}

We now discuss the computational cost of a KDL-NN and give a broad technical explanation of why we expect the KDL-NN to be more efficient in practice. 
Given a KDL-NN defined by \eqref{eq:kdl-nn}, gradient descent updates are performed on layer $\ell$ from $L$ to $2$ via the relations,
\begin{align*}
W_R^{(\ell,i)} &\leftarrow (1-\lambda\eta)W_R^{(\ell,i)} - \eta \Delta_1^{(\ell,i)}A_L^{(\ell,i)T}\\
W_L^{(\ell,i)} &\leftarrow (1-\lambda\eta)W_L^{(\ell,i)} - \eta A_R^{(\ell-1)T}\Delta_2^{(\ell,i)}\\
B_R^{(\ell,i)} &\leftarrow (1-\lambda\eta)B_R^{(\ell,i)} - \eta \Delta_1^{(\ell,i)}\\
B_L^{(\ell,i)} &\leftarrow (1-\lambda\eta)B_L^{(\ell,i)} - \eta \Delta_2^{(\ell,i)},
\end{align*}
where the intermediate matrices $\Delta_1^{(\ell,i)}$, $\Delta_2^{(\ell,i)}$ and $\Gamma^{(\ell)}$ are defined in Appendix~\ref{app:back}.
A simple implementation of back-propagation with learning rate $\eta$ for KDL pair $\ell$ from layer pairs $L$ to $2$ is given in Algorithm~\ref{alg:bp}.

\begin{algorithm}
\caption{KDL Back-Propagation}
\label{alg:bp}
\begin{algorithmic}[1]
\STATE{Given: $\ell$, $\lambda$, $\phi_1$, $\phi_2$, $Z_R^{(\ell,i)}$, $\Gamma^{(\ell)}$, $Z_L^{(\ell,i)}$, $W_R^{(\ell,i)}$, $W_L^{(\ell,i)}$, $B_R^{(\ell,i)}$, $B_L^{(\ell,i)}$}
\STATE{Initialize $\Gamma^{(\ell-1)} = 0$}
\FOR{ $i = 1$ to $k$ }
\STATE{$\Delta_1^{(\ell,i)} = \phi_1'(Z_R^{(\ell,i)})\circ \Gamma^{(\ell)}$}
\STATE{$\Delta_2 = ((W_R^{(\ell,i)})^T\Delta_1^{(\ell,i)})\circ \phi_2'(Z_L^{(\ell,i)})$}
\STATE{$\Gamma^{(\ell-1)} \leftarrow \Gamma^{(\ell-1)} + \Delta_2^{(\ell,i)} (W_L^{(\ell,i)})^T$}
\STATE{$W_R^{(\ell,i)} \leftarrow (1-\lambda)W_R^{(\ell,i)} - \eta\Delta_1^{(\ell,i)} (A_L^{(\ell,i)})^T$}
\STATE{$W_L^{(\ell,i)} \leftarrow (1-\lambda)W_L^{(\ell,i)} - \eta(A_R^{(\ell-1)T}\Delta_2^{(\ell,i)})$}
\STATE{$B_R^{(\ell,i)} \leftarrow (1-\lambda)B_R^{(\ell,i)} - \eta\Delta_1^{(\ell,i)}$}
\STATE{$B_L^{(\ell,i)} \leftarrow (1-\lambda)B_L^{(\ell,i)} - \eta \Delta_2^{(\ell,i)}$}
\ENDFOR
\end{algorithmic}
\end{algorithm}

The cost of forward operations for an FNN layer \eqref{eq:fnn} with $W\in\mathbb{R}^{m_1m_2\times n_1n_2}$ is dominated by $\mathcal{O}(m_1m_2n_1n_2)$ flops and the cost of the activation function operating on $m_1m_2$ values. In addition, back-propagation requires action by $\phi'$ on $m_1m_2$ values, and is then dominated by $\mathcal{O}(m_1m_2n_1n_2)$ flops for the update.

In comparison, the KDL-NN formulation in \eqref{eq:kdl-nn} is dominated by $\mathcal{O}(m_2n_1(n_2 + m_1))$ in total, with activation function operating on $m_1n_2$ and $m_1m_2$ elements respectively.  Back-propagation then requires action by $\phi'$ on $m_1n_2$ and $m_1m_2$ values respectively.  Updates on the KDL weight matrices are dominated by $\mathcal{O}(m_1n_2(m_2 + n_1))$ flops.  The dominant cost for the KDL-NNs is minimized when $m_2 \approx n_1 \approx m_1 \approx n_2$, reflecting the same savings that one achieves in matrix multiplication involving reshaping of Kronecker product matrices. We show in our results that these savings are considerable in practice.

\section{Numerical Results}

We compare our deep learning performance (training cost and test data accuracy) for our novel KDL-NN architecture against FNNs and E-FNNs under two main objectives.  The first objective is to isolate the effects of implementing a KDL-NN on the training time and accuracy, for which a ``bare bones'' implementation in \textsc{Matlab} is utilized.  The second objective, covered in Subsection~\ref{sub:Tensorflow}, is to determine the ease of implementing KDLs in an existing framework, for which KDL-NNs and convolutional neural networks (CNNs) with KDLs are implemented in Tensorflow.   

Here, the networks are trained first on the function
$$ f(\bm{x}) = \left(\frac{\prod_{k=1}^{\lceil n_1/2 \rceil}\left( 1 + 4^k x_k^2 \right)}{\prod_{k=\lceil n_1/2 \rceil+1}^{n_1}\left( 100 + 5x_k \right)}\right)^\frac{1}{d} $$
evaluted for normal random $x_k\in (-1,1)$ for $k=1,\cdots,n_1$, similar to \cite{adcock,chkifa}, and then on the Bike Sharing Dataset (BSD) \cite{Fanaee-T}, the BlogFeedback data set (BF) \cite{Buza}, and the MNIST data set \cite{Lecun}.  We will also compare results combined with a convolutional neural network (CNN) on the CIFAR-10 data set \cite{Krizhevsky}.  We summarize the sizes of these data sets and the number of inputs/outputs in Table~\ref{Tab:sizes}.

\begin{table}[!htbp]
  \begin{center}
  \resizebox{0.8\textwidth}{!}{
    \renewcommand{\tabcolsep}{0.4cm}
    \renewcommand{\arraystretch}{1.3}
    {\scriptsize
    \begin{tabular}{@{}rrrrc@{}}
      \toprule
      Data set & Inputs $n_1$ & Outputs $n_L$ & $(M, \widetilde{M})$ \\\midrule
      $f(\bm{x})$ \cite{adcock,chkifa} & 8 & 1 & (10000, 1000) \\
      BSD \cite{Fanaee-T} & 14 &  1 & (13903, 3476) \\
      BF \cite{Buza} & 280 & 1 & (41918, 10479) \\
      MNIST \cite{Lecun} & 784 & 10 & (60000, 10000) \\
      CIFAR-10 \cite{Krizhevsky} & 3072 & 10 & (50000, 10000)\\
    \bottomrule
    \end{tabular}
  }
    \renewcommand{\arraystretch}{1}
    \renewcommand{\tabcolsep}{12pt}
  }
  \end{center}
  \caption{Available data and input/output sizes $n_1/n_L$ for the examples in this paper. Also shown are the number of training/test points $M/\widetilde{M}$.}\label{Tab:sizes}
\end{table}

While $M$ data points are used to train the networks, $\widetilde{M}$ points $\left(\widetilde{\bs{x}}_m, \widetilde{\bs{y}}_m\right)_{m \in [\widetilde{M}]}$ used as test data to determine performance on unseen inputs. We report standard $\ell^2$ test losses,
\begin{align*}
  \mathcal{L} = \frac{1}{\widetilde{M}} \sum_{m=1}^{\widetilde{M}} \left\| \widetilde{y}_m - \widetilde{y}_\ast(\widetilde{x}_m) \right\|_2^2,
\end{align*}
where $\widetilde{y}_\ast$ is the (vectorized) output of an FNN, E-FNN, or KDL-NN architecture. All tests are run on a system with a 2.10GHz $\times 64$ processor with 125.5 GiB memory using \textsc{Matlab} R2021a.  

In all examples, we prescribe an FNN architecture, make choices for integer factorizations of input and hidden layer sizes (e.g., $9 = 3 \times 3$), and derive a corresponding KDL-NN and subsequently E-FNN architecture. The discussion in Subsection~\ref{ssec:cost} motivates that the choice of integer factorization should maximize the geometric mean of the factor sizes in order to minimize the training cost. The ReLu activation function is used for training $f(\bm{x})$ in Table~\ref{Tab:sizes}, and tanh is used for all other examples.

\subsection{Fixed-rank KDL-NN performance}

We demonstrate the efficacy of KDL-NN architectures with fixed Kronecker rank $k$. Errors and timing are shown for $f(\bm{x})$ with rank~$1$ KDL-NN in Figure~\ref{fig:fun}, and for BSD, BF, and MNIST in Figure~\ref{fig:val} with ranks~1 and 2 KDL-NNs.  This data is summarized in Table~\ref{Tab:times}.   Note that stochastic gradient descent is used here, but the choice of optimizer and minibatch size do not seem to significantly effect the results.  Similar results are available in Figure~\ref{fig:val_a100} in the supplemental materials using the Adam optimizer with a minibatch size of 100.

For matrix-vector multiplication, the Kronecker product operations \eqref{eq:kron-rearrangement} significantly improve practical efficiency when the matrices $L$ and $R$ are large, but not when they are small. This property extends to our KDL-NN architecture. We demonstrate this by prescribing two different FNN architectures for the BSD dataset: BSD(a) corresponds to an FNN where the $L$ and $R$ matrices have ``small'' sizes, and BSD(b) to one where they have ``large'' sizes. The results in Figure~\ref{fig:val} use stochastic gradient descent (SGD) and show that for the BSD(a) architecture, the FNN is more efficient to train than the KDL-NN network due to the small sizes of the factorized matrices. However, for BSD(b), we increase the network size and observe that KDL-NN training is much faster. These observations are consistent with what one would expect for Kronecker product-based matrix multiplication.  In terms of accuracy, we see that the KDL-NN architecture tends to maintain or improve model capacity compared to FNN architectures, even for rank~1 KDL-NNs.

\begin{figure}[!htbp]
  \begin{center}
\subfloat[ $f(\bm{x})$
\label{fun_test}]{\includegraphics[width=0.35\textwidth]{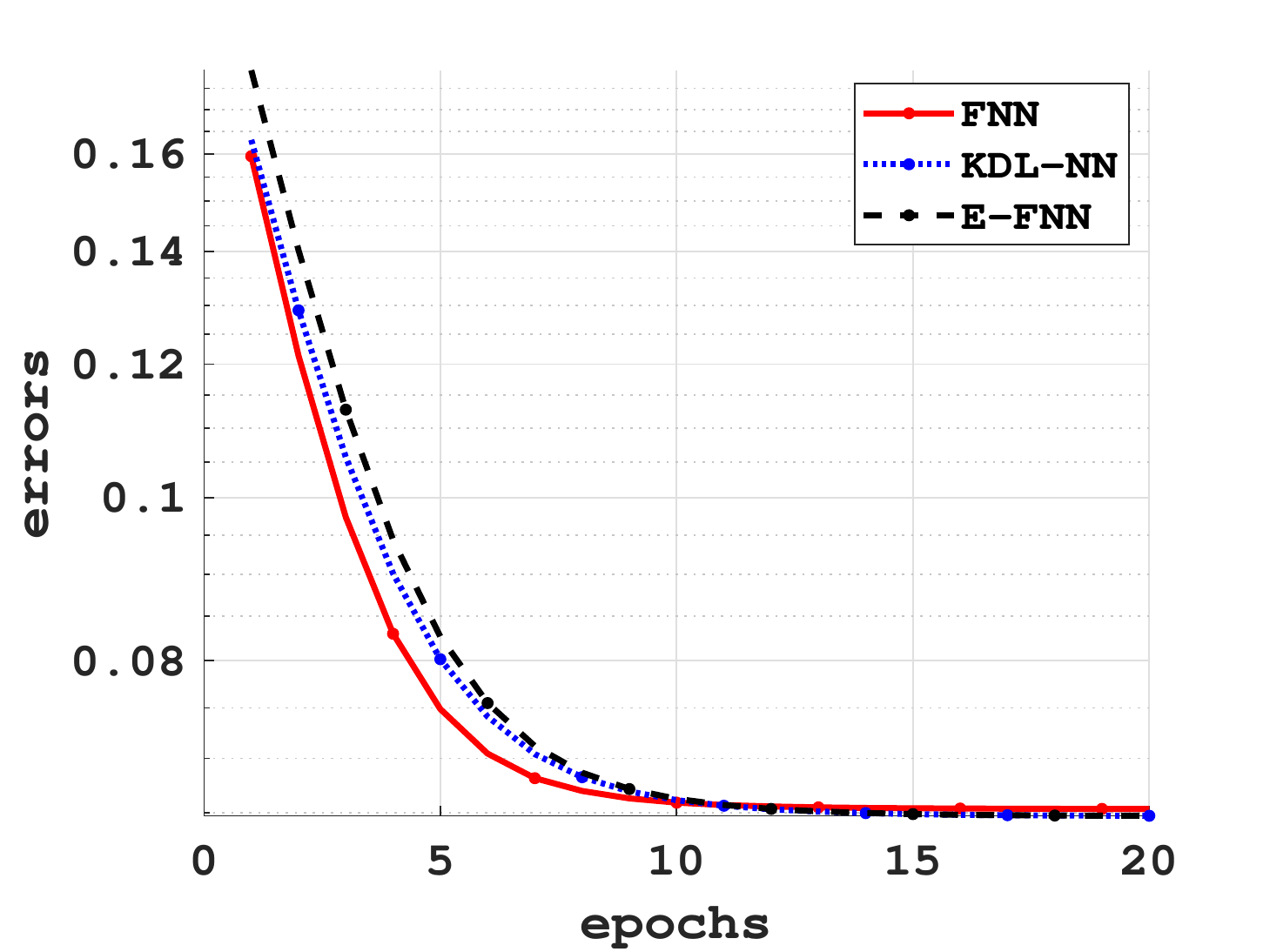}}
\subfloat[ $f(\bm{x})$
\label{fun_hist}]{\includegraphics[width=0.35\textwidth]{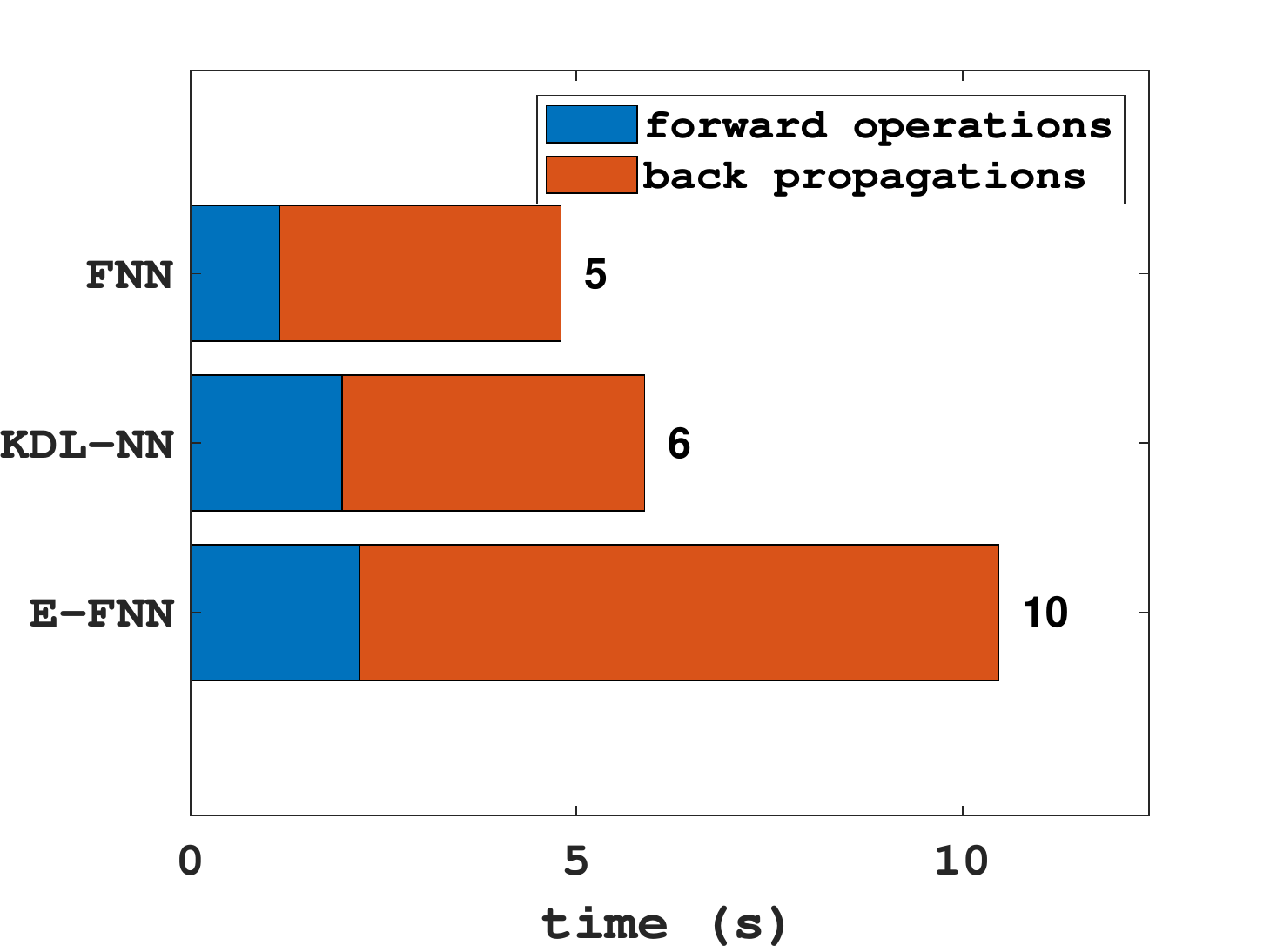}}
\end{center}
\caption{Figures~\ref{fun_test} and \ref{fun_hist} show the test errors and timing, broken down by forward operations and back-propagations, for $f(\bm{x})$ using networks defined in Table~\ref{Tab:times} with Kronecker rank~1. \label{fig:fun}}
\end{figure}

\begin{figure}[!htbp]
  \begin{center}
\includegraphics[width=0.24\textwidth]{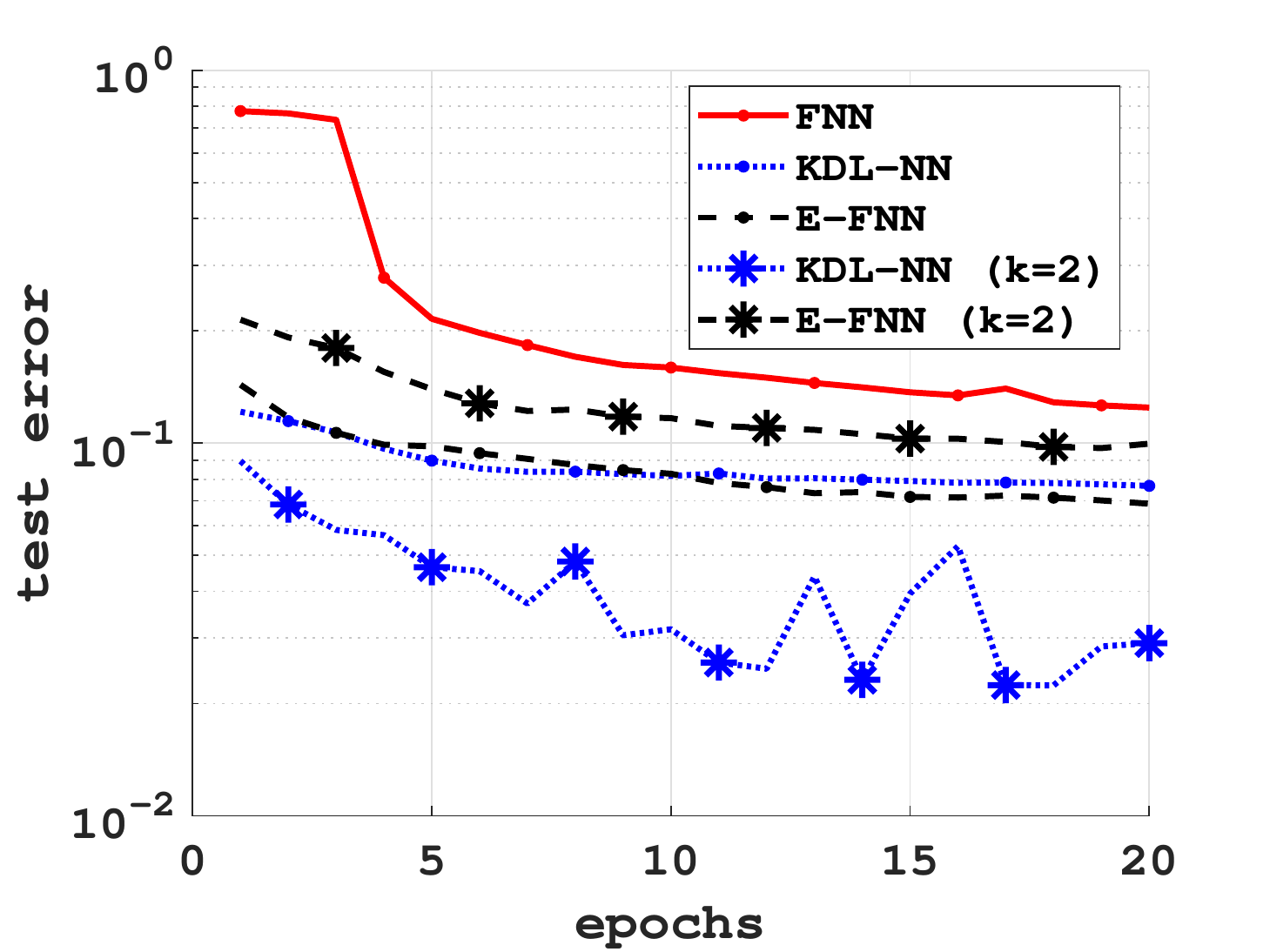}
\includegraphics[width=0.24\textwidth]{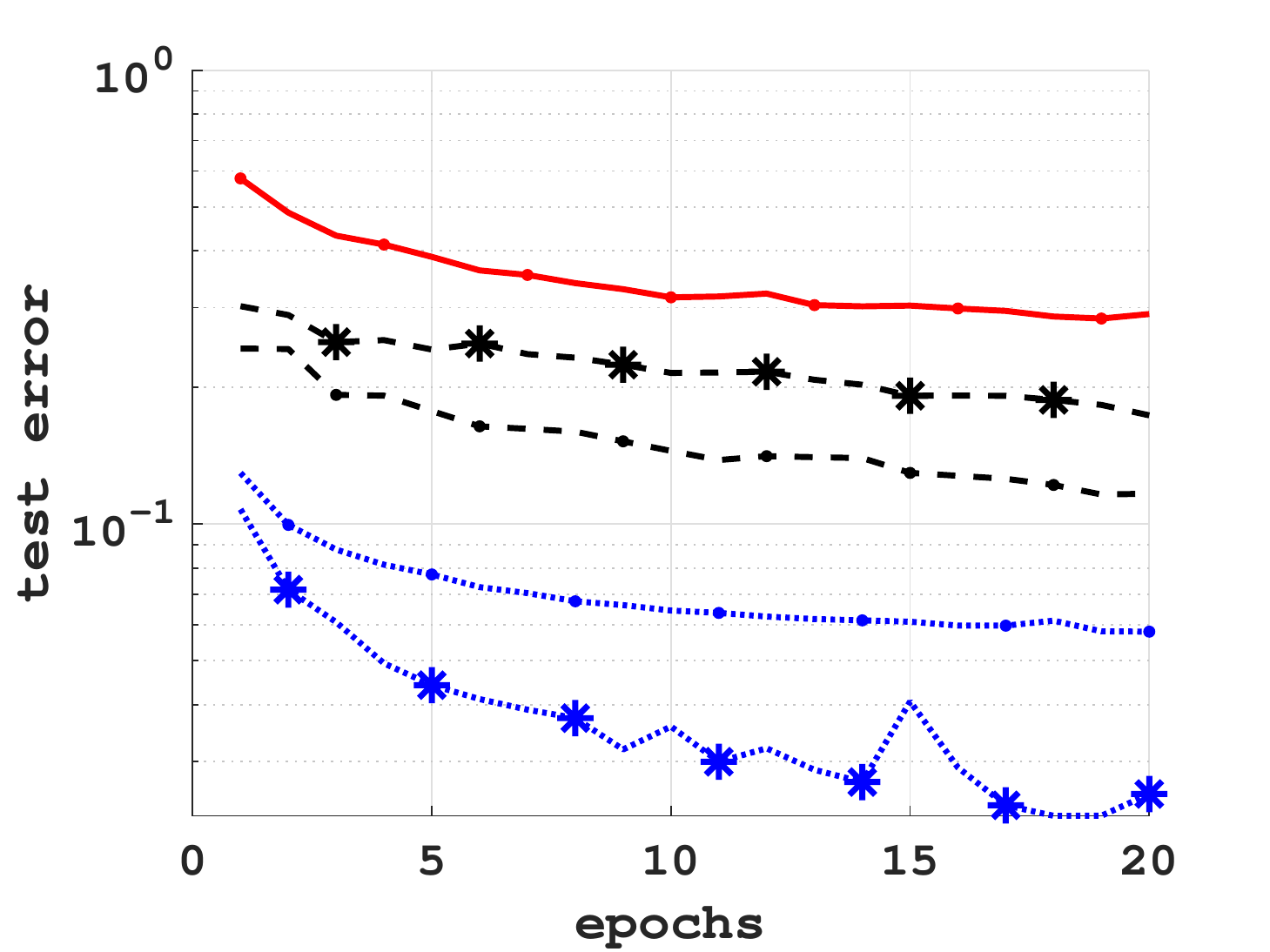}
\includegraphics[width=0.24\textwidth]{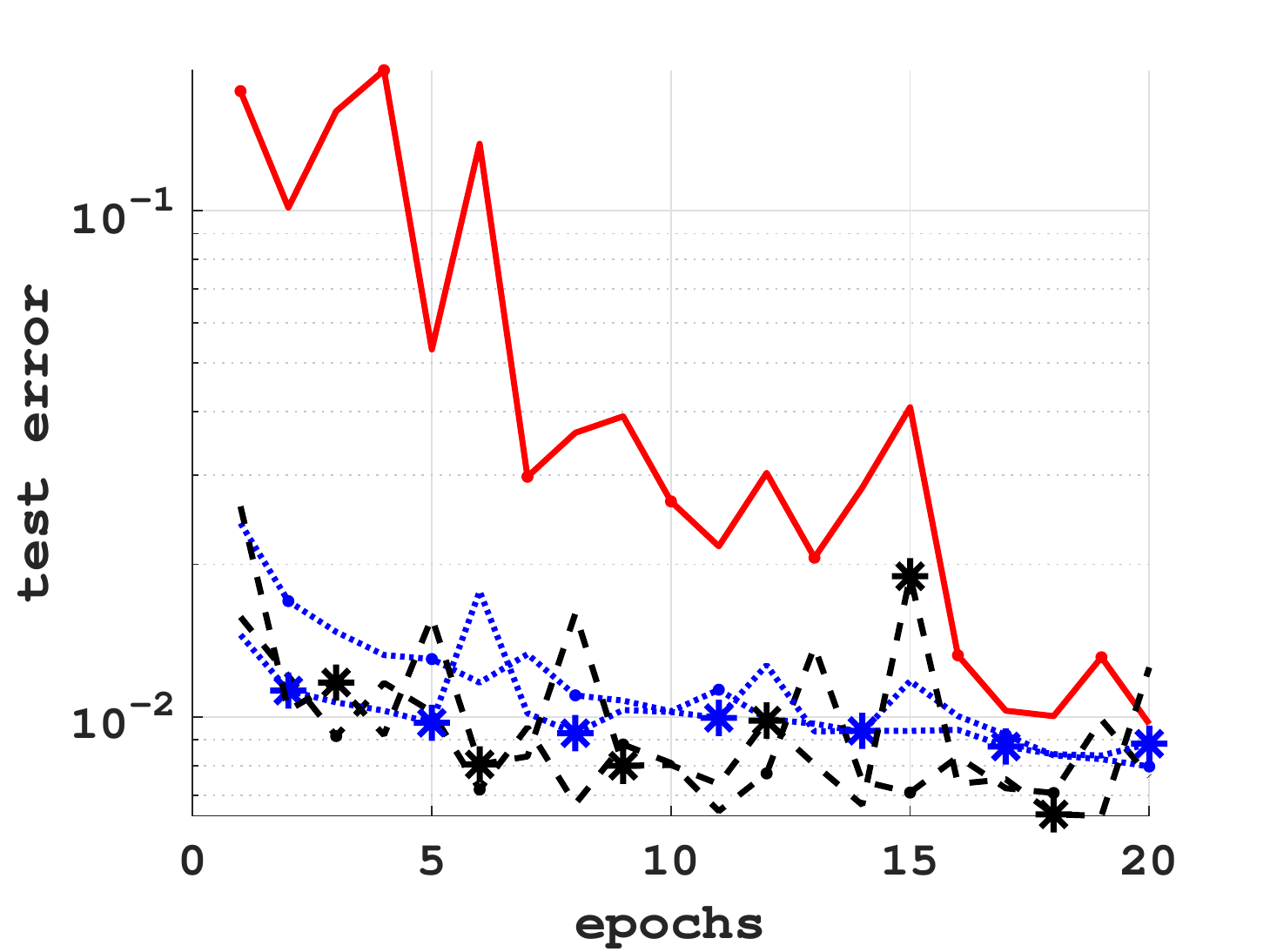}
\includegraphics[width=0.24\textwidth]{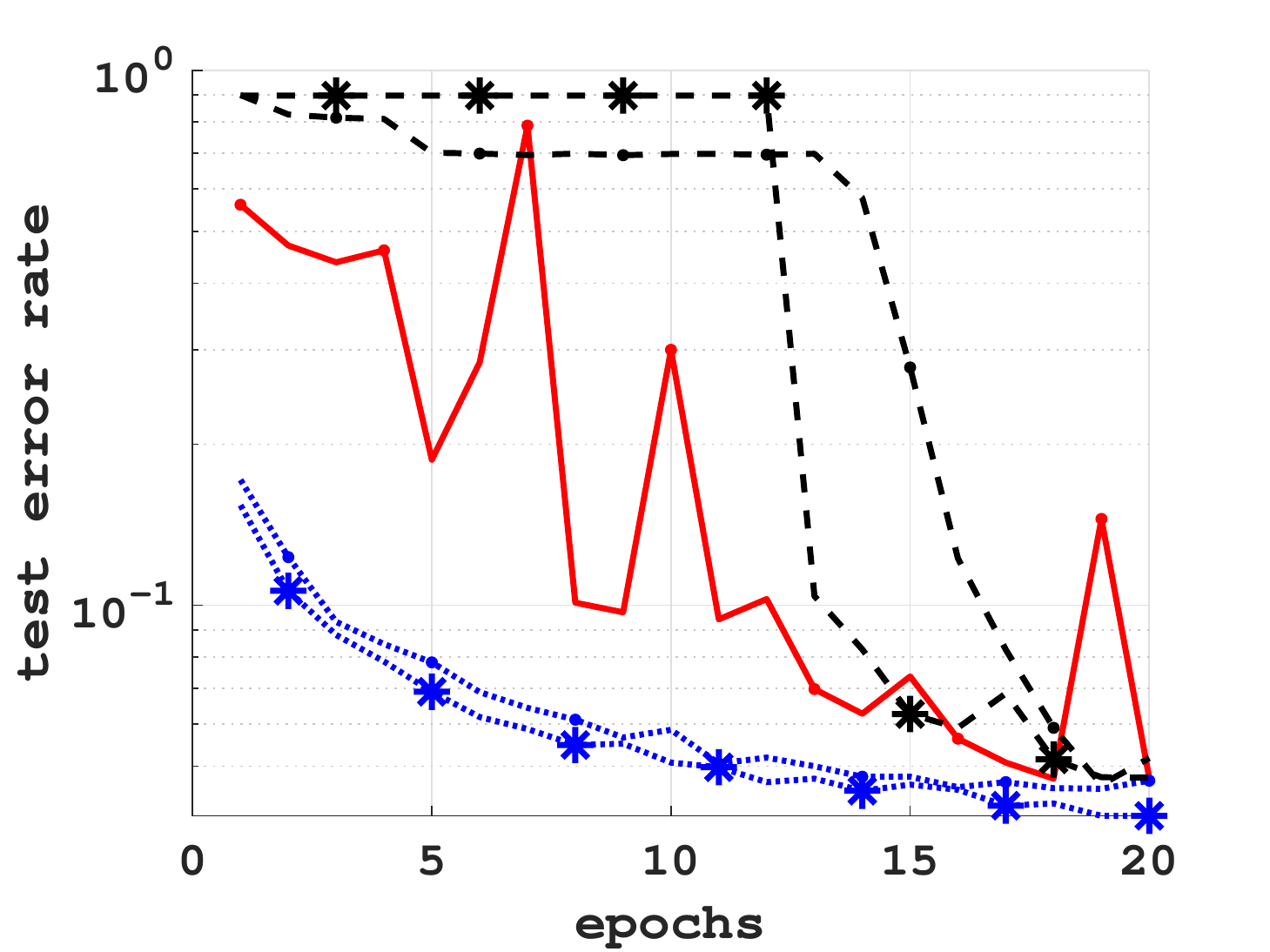}
\\
\subfloat[BSD (a) 
\label{bikes8b}]{\includegraphics[width=0.25\textwidth]{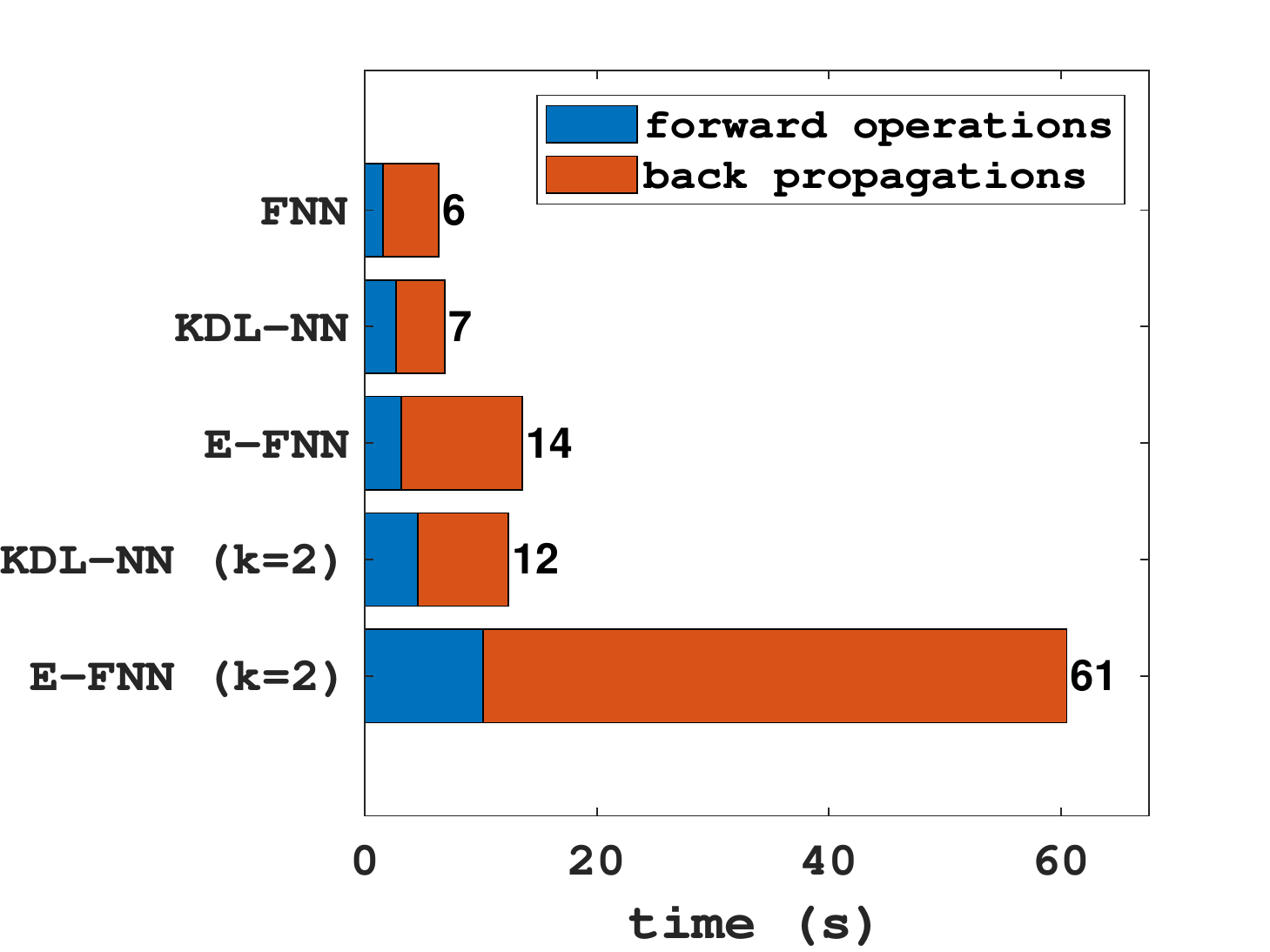}}
\subfloat[BSD (b)  
\label{bikes}]{\includegraphics[width=0.25\textwidth]{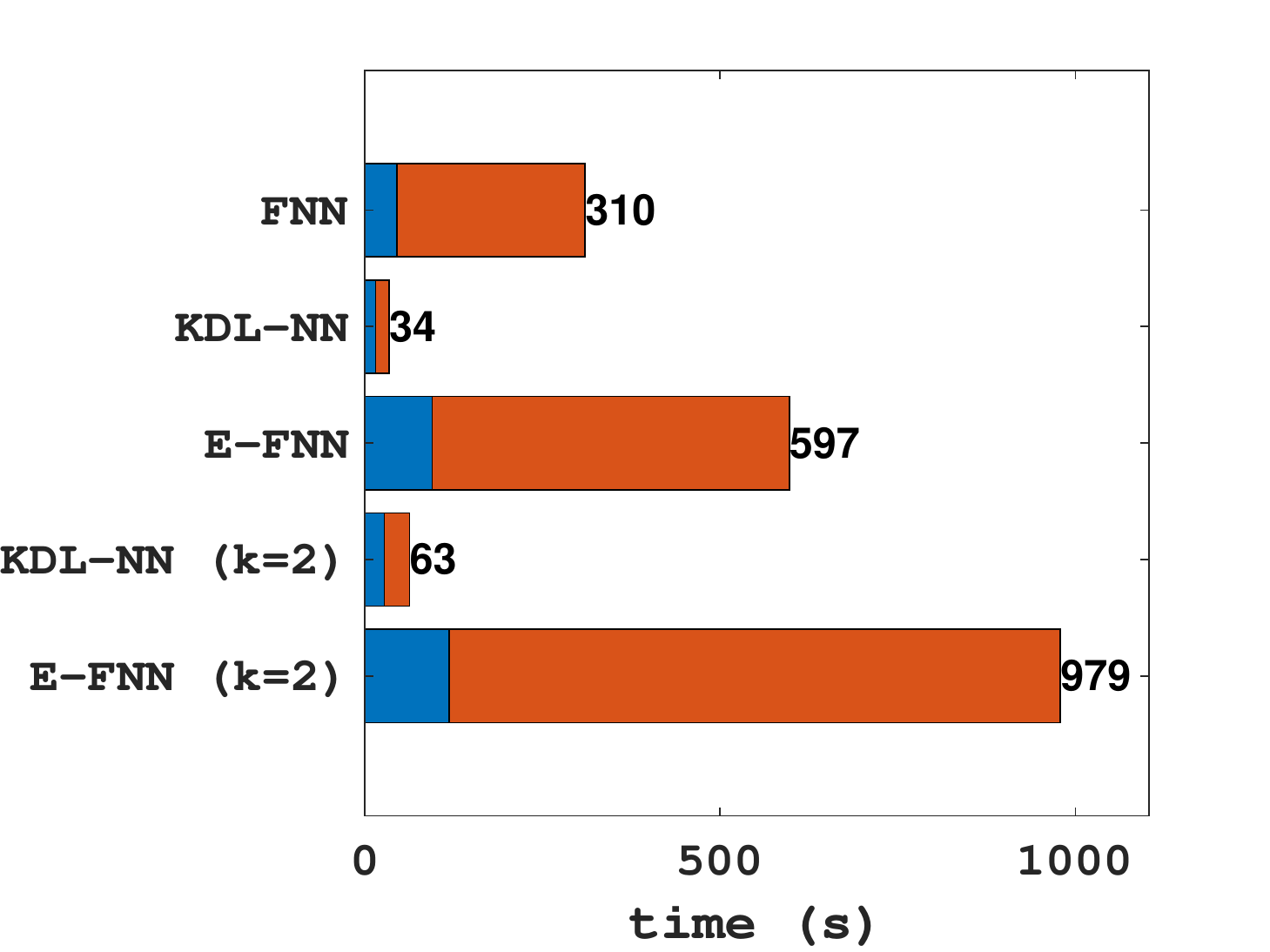}}
\subfloat[BF  
\label{blog}]{\includegraphics[width=0.25\textwidth]{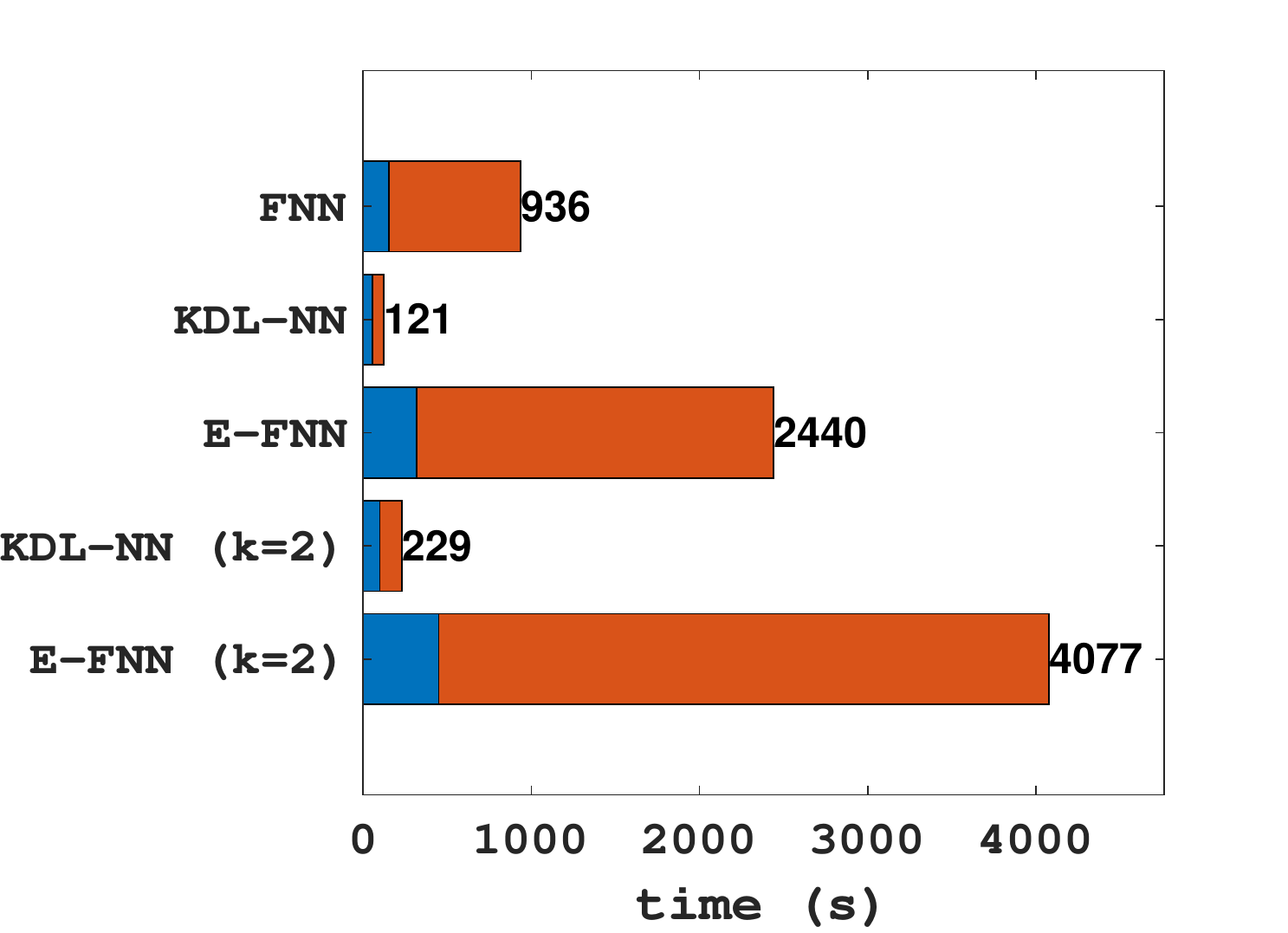}}
\subfloat[MNIST  
\label{MNIST}]{\includegraphics[width=0.25\textwidth]{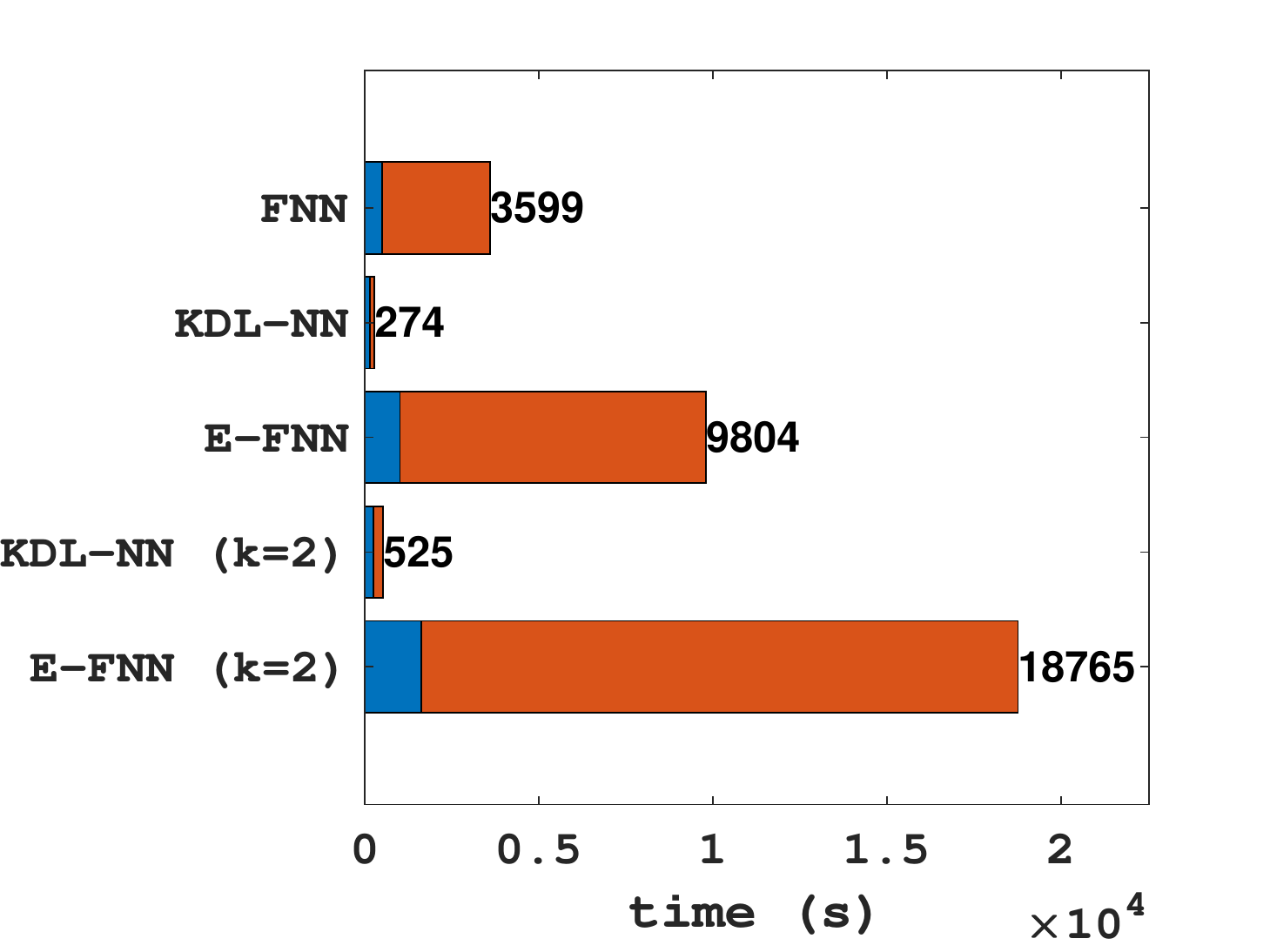}}
\end{center}
\caption{Figures~\ref{bikes8b} to \ref{MNIST} show the test errors and timing, broken down by forward operations and back-propagations, for FNN, KDL-NN, and E-FNN for BSD (a and b), BF, and MNIST respectively as defined in Table~\ref{Tab:times} with Kronecker ranks~$k=1$ and $k=2$. \label{fig:val}}
\end{figure}

\begin{table}[!htbp]
  \begin{center}
  \resizebox{1.0\textwidth}{!}{
    \renewcommand{\tabcolsep}{0.2cm}
    \renewcommand{\arraystretch}{1.3}
    {\scriptsize
    \begin{tabular}{lrcrrrrrr}
      \toprule
      Data set & Network & Architecture & \# Parameters & $\rightarrow$ & $\leftarrow$ & Total time & Test error (\%) \\\midrule \rowcolor{Gray}
               & FNN     & $8|64|64|1$ & 4,801 & \textbf{1.2} & \textbf{3.7} & \textbf{4.8} & 6.53\\\rowcolor{Gray}
    $f(\bm{x}$)& KDL-NN  & $(2,4)|(8,8)|(8,8)|(1,1)$ & \textbf{409} & 2.0 & 3.9 & 5.9 & 6.47 \\\rowcolor{Gray}
               & E-FNN   & $8|16|64|64|64|8|1$ & 10,081 & 2.2 & 8.3 & 10.5 & \textbf{6.47} \\
               & FNN     & $14|64|64|1$ & 5,185 & \textbf{1.6} & 4.8 & \textbf{6.4} & 12.45 \\
               & KDL-NN  & $(2,7)|(8,8)|(8,8)|(1,1)$ & \textbf{433} & 2.7 & \textbf{4.2} & 6.9 & 7.68 \\
      BSD (a)  & E-FNN   & $14|16|64|64|64|8|1$ & 10,177 & 3.1 & 10.4 & 13.6 & 6.87 \\
               & KDL-NN  & $(2,7)|^2(8,8)|^2(8,8)|^2(1,1)$ & 866 & 4.6 & 7.8 & 12.4 & \textbf{2.91} \\
               & E-FNN   & $14|32|64|128|64|16|1$ & 20,225 & 10.2 & 50.3 & 60.5 & 9.96 \\\rowcolor{Gray}
               & FNN     & $14|400|400|1$ & 166,801 & 45.1 & 264.7 & 309.8 & 29.03\\\rowcolor{Gray}
               & KDL-NN  & $(2,7)|(20,20)|(20,20)|(1,1)$ & \textbf{2,281} & \textbf{14.9} & \textbf{19.0} & \textbf{33.9} & 5.79 \\\rowcolor{Gray}
      BSD (b)  & E-FNN   & $14|40|400|400|400|20|1$ & 345,841 & 94.6 & 502.6 & 597.2 & 11.67 \\\rowcolor{Gray}
               & KDL-NN  & $(2,7)|^2(20,20)|^2(20,20)|^2(1,1)$ & 4,562 & 27.1 & 35.6 & 62.7 & \textbf{2.54} \\\rowcolor{Gray}
               & E-FNN   & $14|80|400|800|400|40|1$ & 690,881 & 118.1 & 860.7 & 978.8 & 17.36 \\
               & FNN     & $280|400|400|1$ & 273,201 & 153.7 & 782.3 & 935.9 & 0.97 \\
               & KDL-NN  & $(20,14)|(20,20)|(20,20)|(1,1)$ & \textbf{3,141} & \textbf{54.1} & \textbf{66.9} & \textbf{121.0} & 0.80 \\
      BF       & E-FNN   & $280|400|400|400|400|20|1$ & 601,641 & 318.6 & 2,121.1 & 2,439.7 & \textbf{0.76} \\
               & KDL-NN  & $(20,14)|^2(20,20)|^2(20,20)|^2(1,1)$ & 6,282 & 99.7 & 129.8 & 229.4 & 0.88 \\
               & E-FNN   & $280|800|400|800|400|40|1$ & 1,202,481 & 449.8 & 3,627.6 & 4,077.4 & 1.25 \\\rowcolor{Gray}
               & FNN     & $784|784|784|10$ & 1,238,730 & 486.5 & 3,112.8 & 3,599.4 & 4.75 \\\rowcolor{Gray}
               & KDL-NN  & $(28,28)|(28,28)|(28,28)|(5,2)$ & \textbf{6,534} & \textbf{140.1} & \textbf{134.3} & \textbf{274.3} & 4.70  \\\rowcolor{Gray}
     MNIST     & E-FNN   & $784|784|784|784|784|56|10$ & 2,506,290 & 1,004.9 & 8,799.6 & 9,804.5 & 5.18 \\\rowcolor{Gray}
               & KDL-NN  & $(28,28)|^2(28,28)|^2(28,28)|^2(5,2)$ & 13,068 & 236.6 & 288.0 & 524.5 & \textbf{4.04}  \\\rowcolor{Gray}
               & E-FNN   & $784|1568|784|1568|784|112|10$ & 5,011,002 & 1,627.1 & 17,137.8 & 18,764.8 & 4.76 \\
    \bottomrule
    \end{tabular}
  }
    \renewcommand{\arraystretch}{1}
  }
  \end{center}
  \caption{Network architectures for each example, including total number of trainable parameters, 20-Epoch training time (s) divided into forward propagation ($\rightarrow$), back-propagation ($\leftarrow$), and total training time, and test data loss. The optimal result for each category is boldfaced without regard to rounding.}\label{Tab:times}
\end{table}

\subsection{Adaptive Choice of Rank}
In general, the rank needed to optimize the KDL-NN is unknown.  Since the KDL summands use separate weight matrices, it is straightforward to add new pairs of weight matrices during training to increase the rank.  

To check the decay of the errors, a validation set is pulled from the training set based on 10\% of the total set size.  New pairs of weight matrices are initialized to normal random matrices, scaled to machine epsilon, and added when the decay of a range in the validation error levels off.  In order to achieve a reduction in error when adding matrices, the learning rate may need to be adjusted.  Factors are chosen in a range from $\frac{1}{n}$ to $2$.  A learning rate based on each factor is then used for a set number of epochs, and the learning rate that produces the smallest error is selected moving forward.  Figure~\ref{fig:adapt} shows training errors and timing using this procedure with a range of 4 learning rates running for 10 epochs when increasing rank, in comparison to preset ranks ranging form 1 to 3.  Results are shown for an average over 10 runs. While the adaptive method is capable of improving the model capacity to be in line with a larger rank solution, the timing is similar to choosing a larger rank to begin with.  In addition, choosing a rank larger than necessary does not seem to have any detrimental effects on the accuracy of the network.

\begin{figure}[!htbp]\begin{center}
\includegraphics[width=0.24\textwidth]{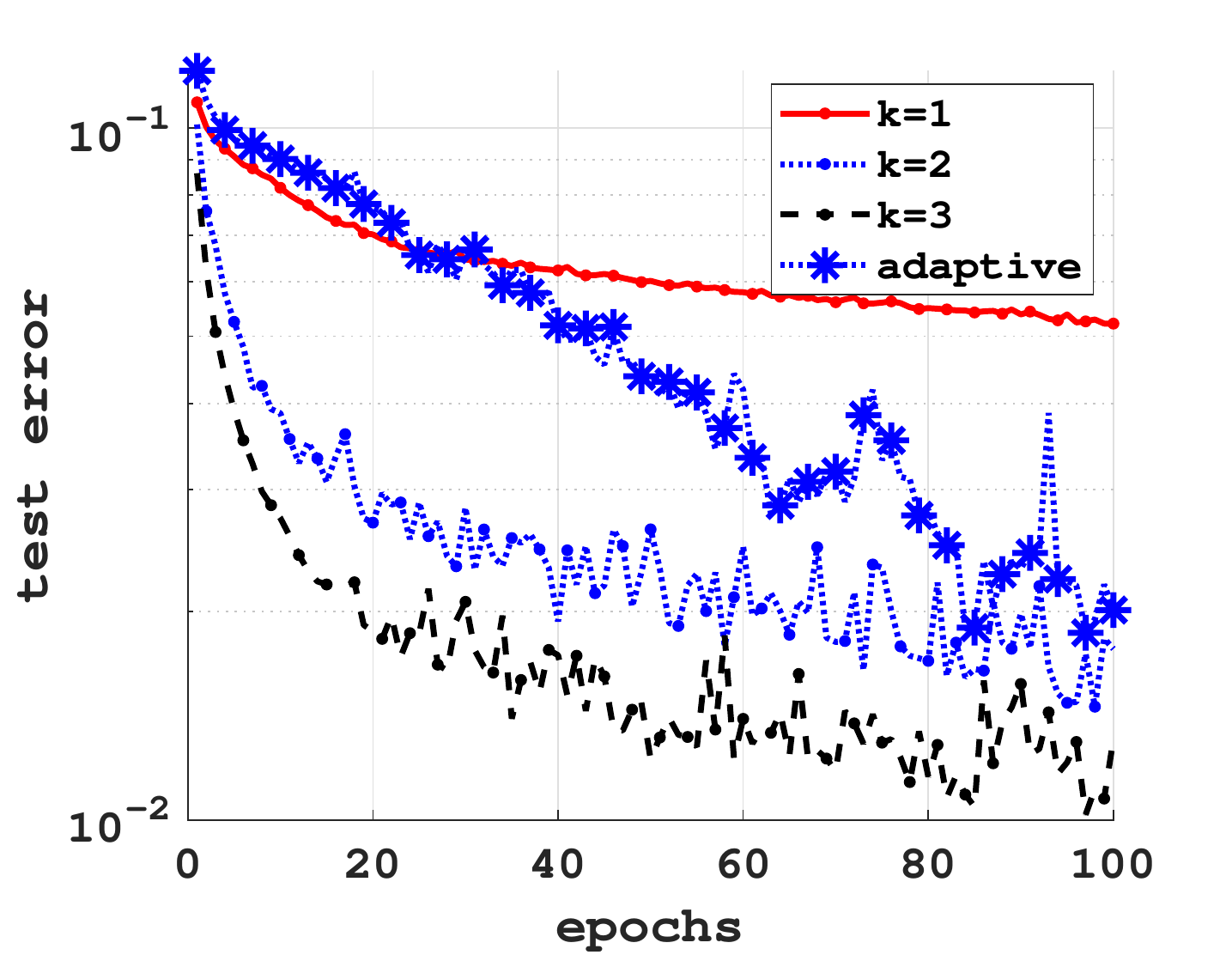}
\includegraphics[width=0.24\textwidth]{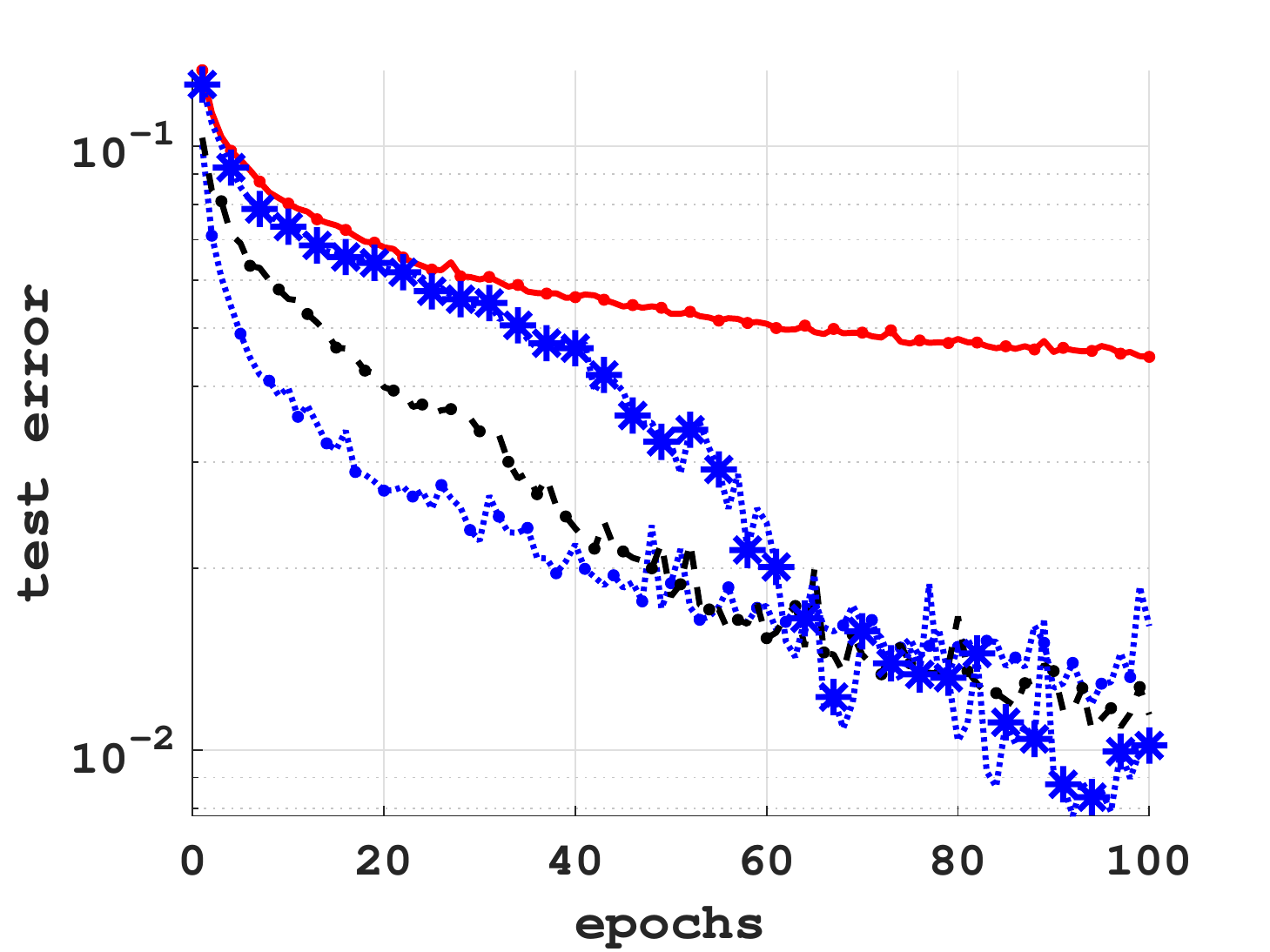}
\includegraphics[width=0.24\textwidth]{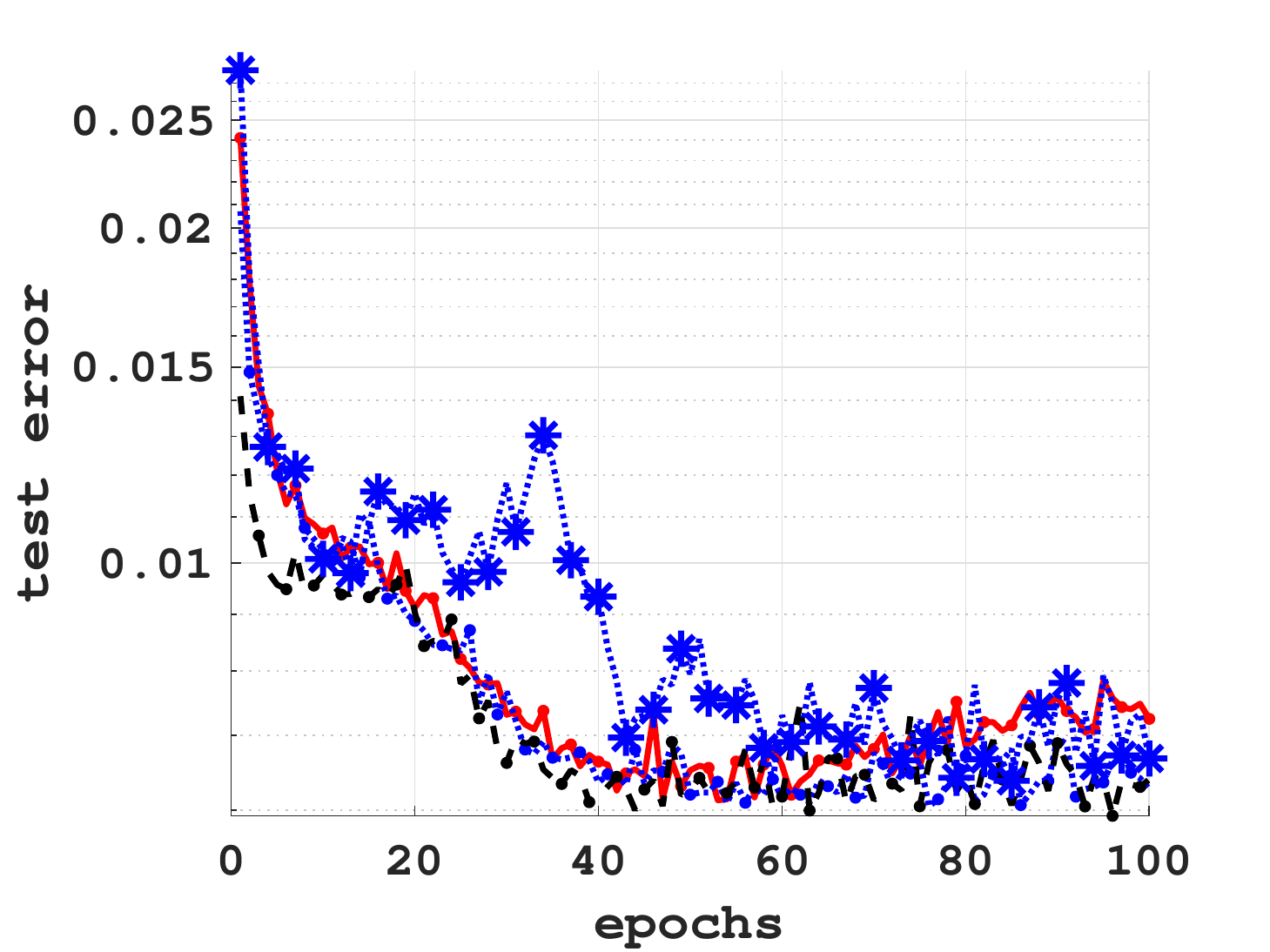}
\includegraphics[width=0.24\textwidth]{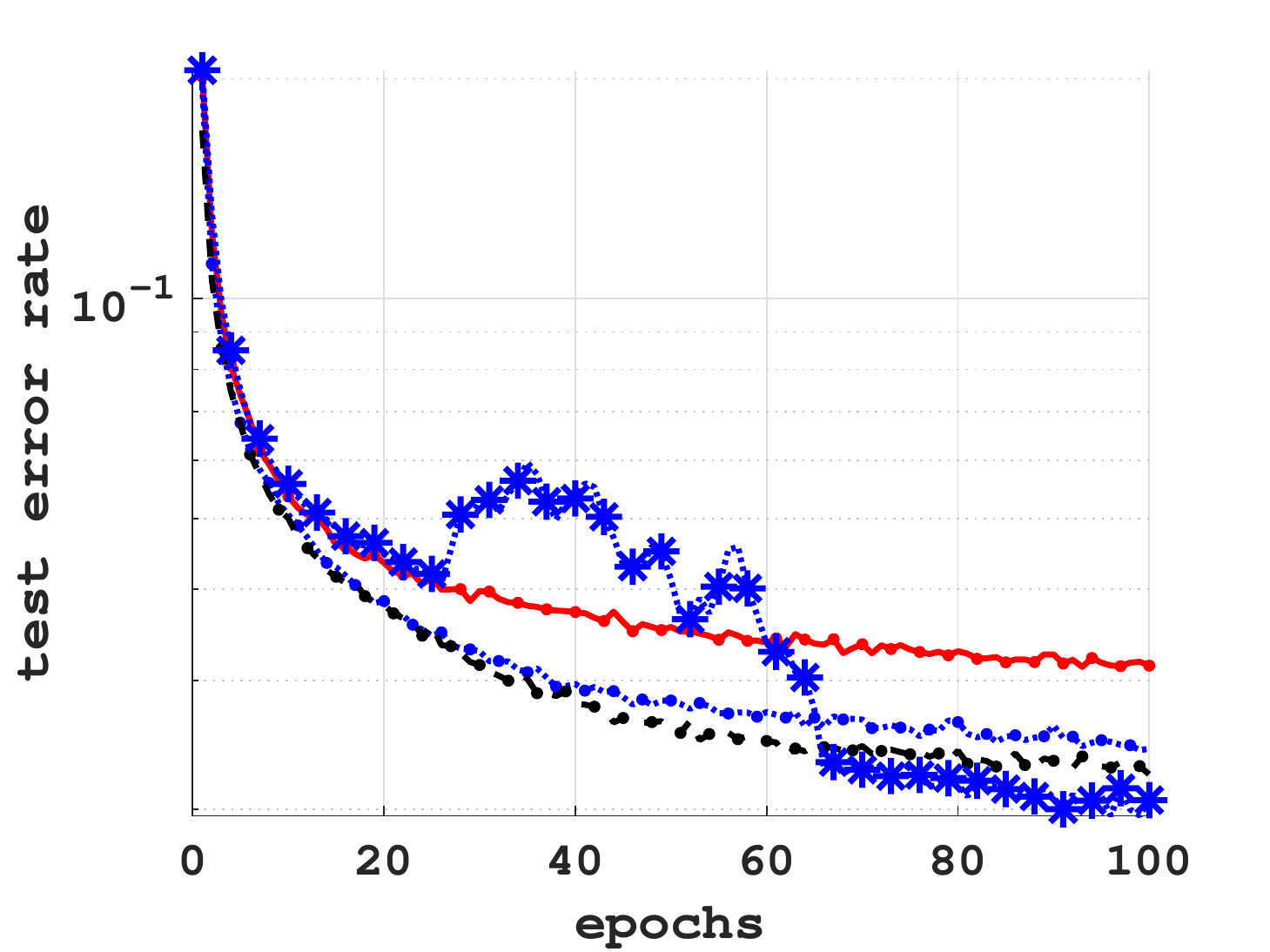}\\
\subfloat[BSD (a) \label{bikes8_adapt}]{\includegraphics[width=0.25\textwidth]{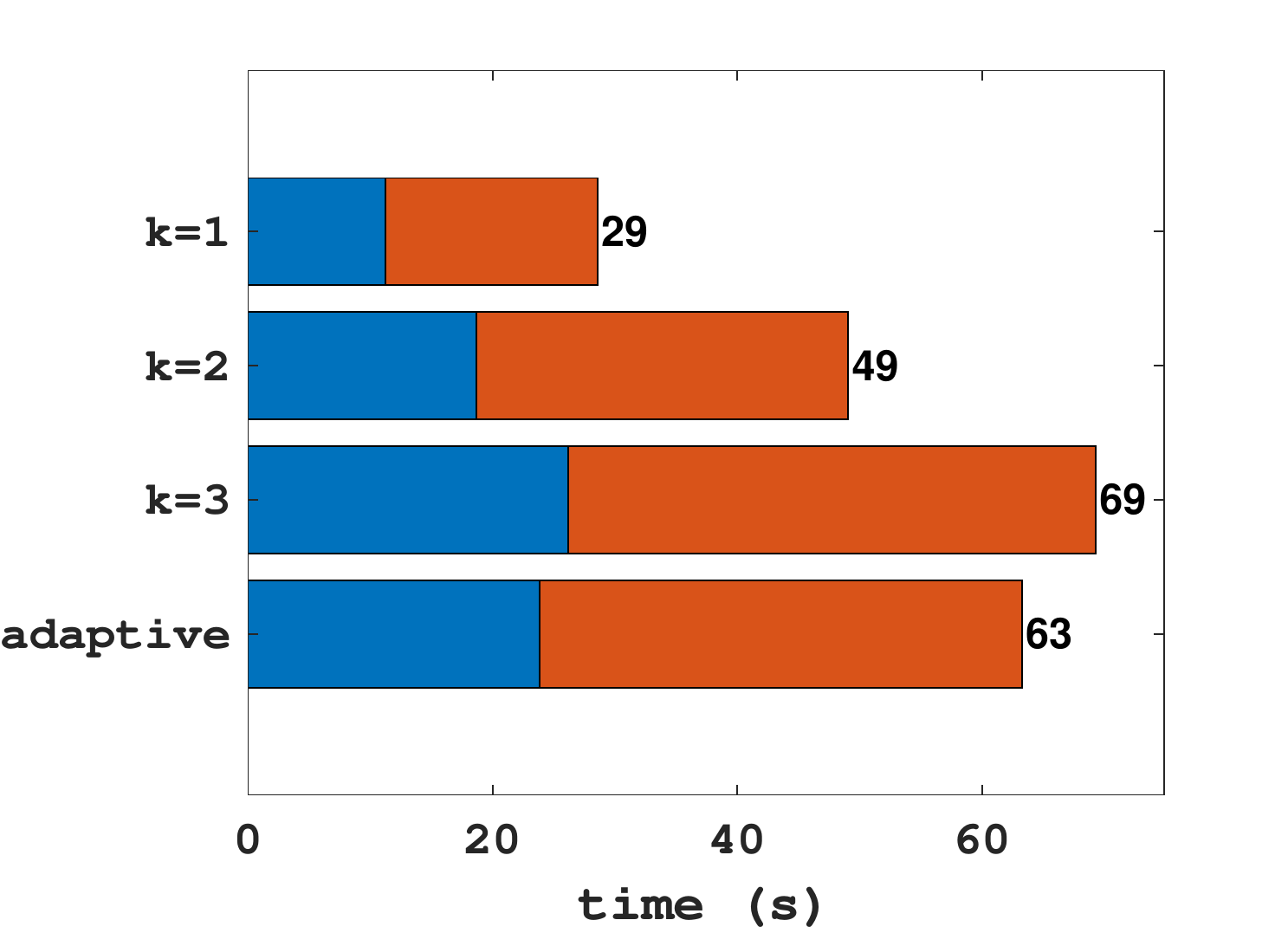}}
\subfloat[BSD (b) \label{bikes_adapt}]{\includegraphics[width=0.25\textwidth]{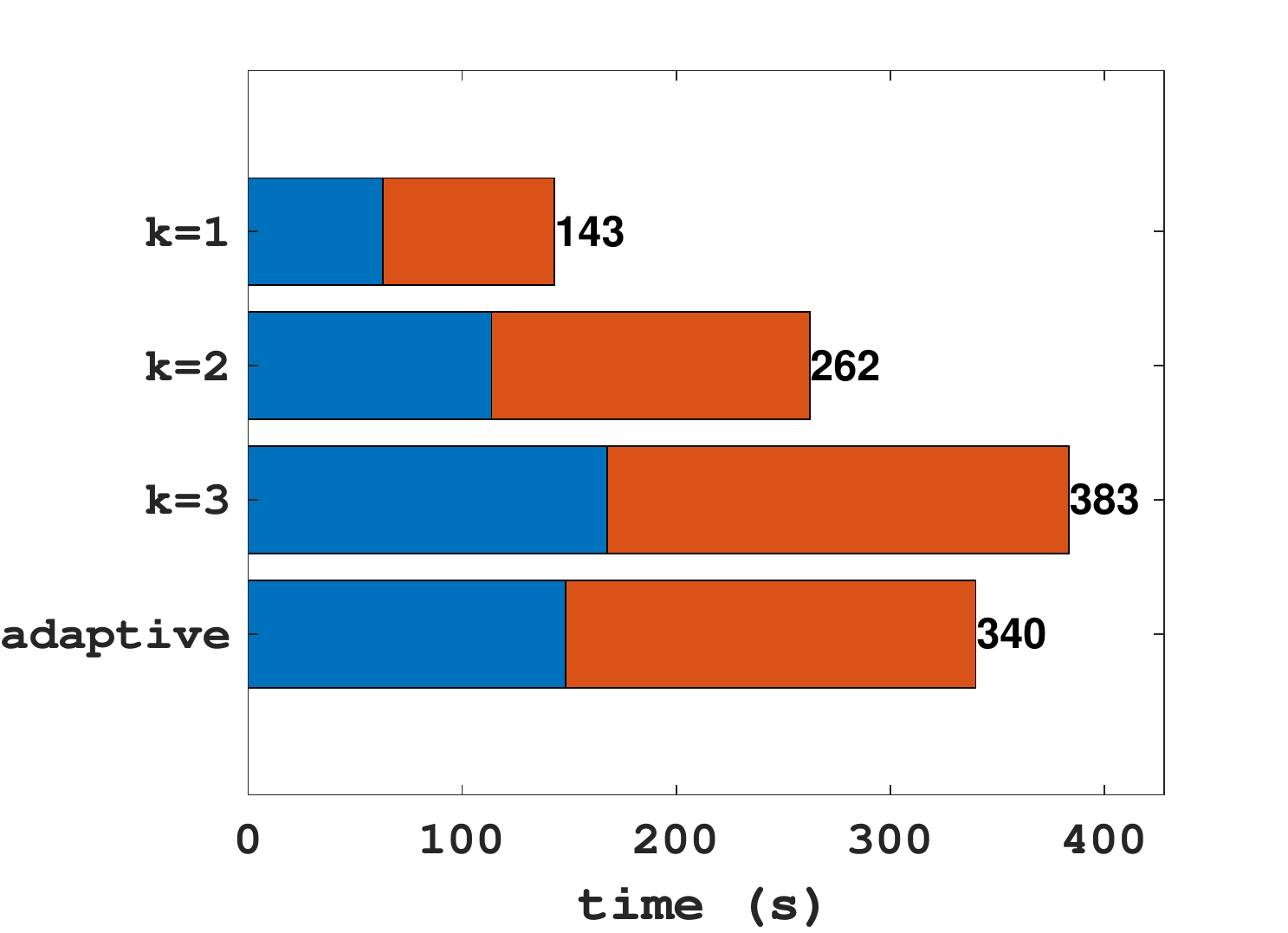}}
\subfloat[BF \label{blog_adapt}]{\includegraphics[width=0.25\textwidth]{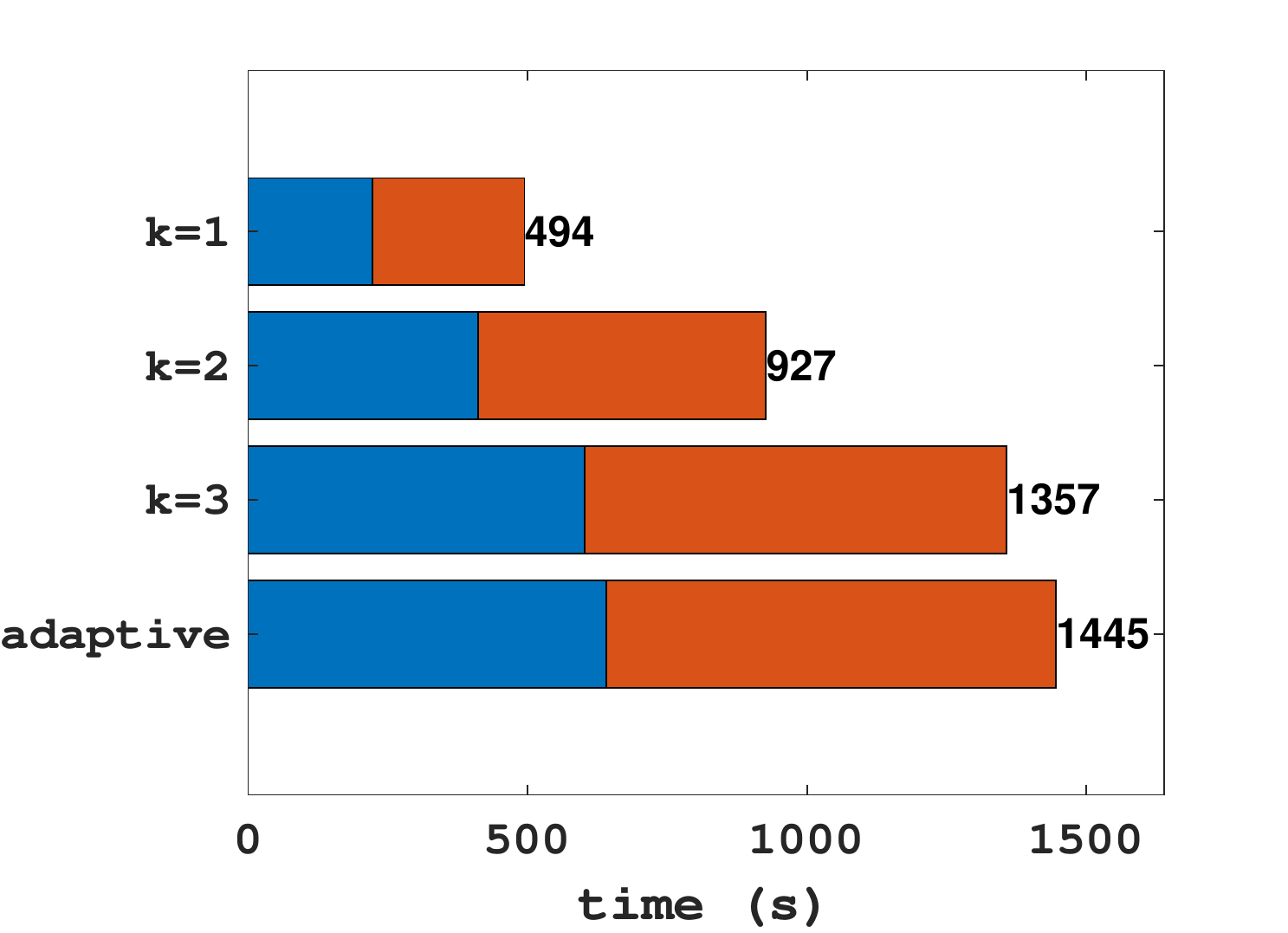}}
\subfloat[MNIST  \label{MNIST_adapt}]{\includegraphics[width=0.25\textwidth]{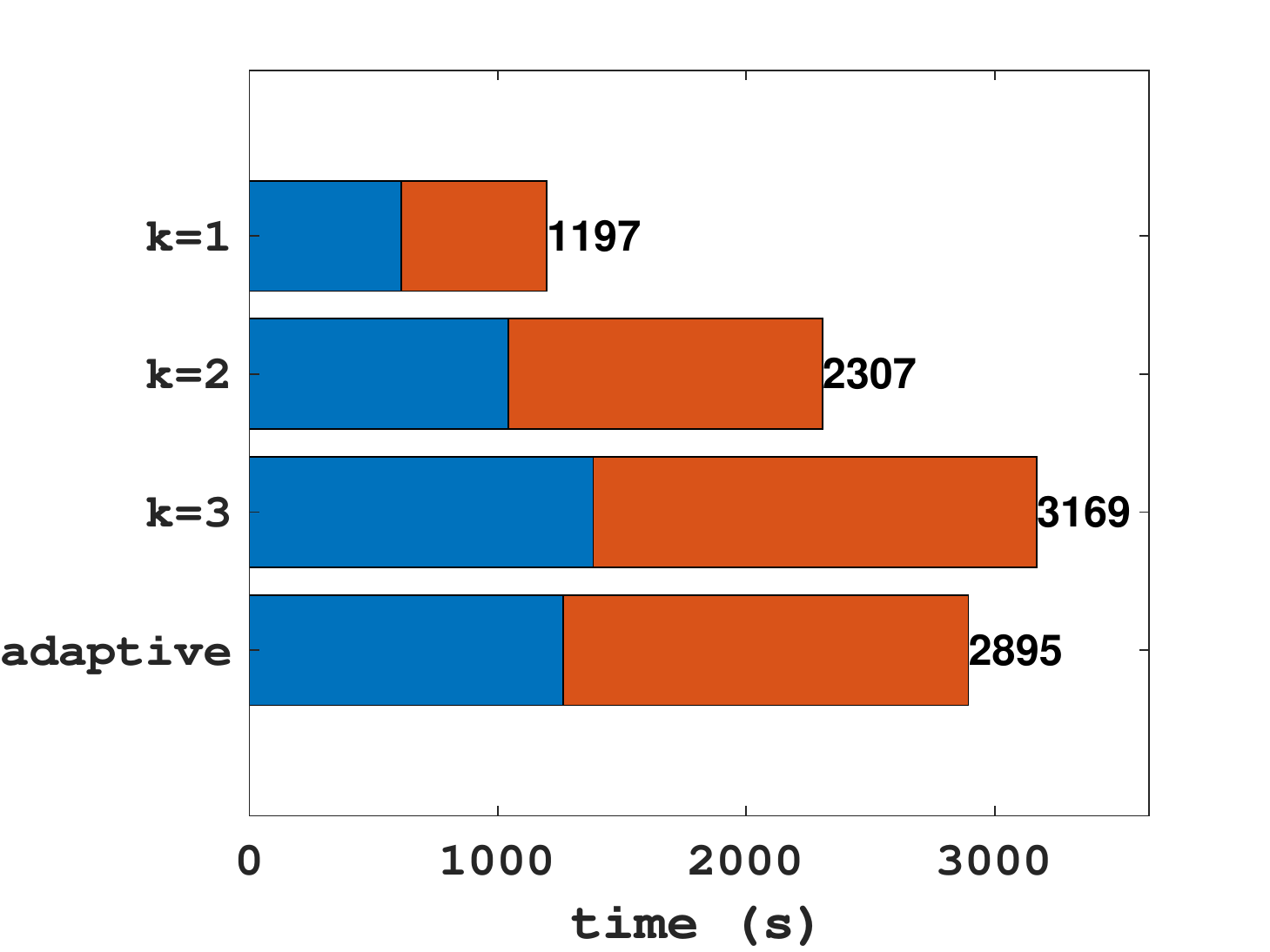}}
\end{center}
\caption{Figures~\ref{bikes8_adapt} to \ref{MNIST_adapt} show the average test errors and timing breakdown for KDL-NNs with ranks~$k=1$, $k=2$, $k=3$, and adaptive rank for BSD (a and b), BF, and MNIST respectively over 10 trials.\label{fig:adapt}}
\end{figure}

\subsection{Tensorflow Implementation}\label{sub:Tensorflow}
Here, Tensorflow is used to provide generalized implementations of KDL-NNs, and while more versatile, contains more overhead costs than the ``bare-bones'' \textsc{Matlab} implementation.  In addition, KDLs are incorporated into convolutional neural networks (CNNs) \cite{Nebauer}, which are commonly used for image classification.  Standard CNNs are comprised of three types of layers: convolutional layers, pooling layers, and fully connected layers.  As such, KDLs being used in place of the fully connected layers only comprise a portion of the CNNs.  Results are generated for the CIFAR-10 data set \cite{Krizhevsky}, comprised of $32\times 32\times 3$ color images.  After several convolutional and max pooling layers, 2 hidden layers and an output layer using either fully connected layers or KDLs are implemented.  ReLu activation functions are used for intermediate layers, and the softmax activation function is used on the output layer.  The networks and resulting dimensions at each layer are summarized in Table~\ref{Tab:CNN}.

\begin{table}[!htbp]
\begin{center}
\resizebox{1.0\textwidth}{!}{
\renewcommand{\tabcolsep}{0.2cm}
\renewcommand{\arraystretch}{1.3}
{\scriptsize
\begin{tabular}{lcccrrrr}
\toprule
Input & Conv. + Max Pool & Conv. + Max Pool & Conv. + Max Pool & Reshape & Hidden & Hidden & Output \\ 
\midrule 
\multirow{2}{*}{$32\times 32\times 3$} & \multirow{2}{*}{$16\times 16 \times 16$} & \multirow{2}{*}{$8\times 8\times 64$} & \multirow{2}{*}{$4\times 4\times 256$} & (FNN) 4096 & 4096 & 4096 & 10 \\ 
& & & & \cellcolor{Gray} (KDL) $64\times 64$ & \cellcolor{Gray} $64\times 64$ & \cellcolor{Gray} $64\times 64$ & \cellcolor{Gray} $5\times 2$\\ 
\bottomrule
\end{tabular}
}
}
    \renewcommand{\arraystretch}{1}
  \end{center}
  \caption{CNN Network architectures for CIFAR-10 with corresponding dimensions.}\label{Tab:CNN}
\end{table}

Test accuracy is shown for Tensorflow implementations of MNIST and CIFAR-10 in Figure~\ref{fig:tf} using a log-log scale.  Open source code is available at \href{https://github.com/JaromHogue/KDLayers}{github.com/JaromHogue/KDLayers}.  
While the Tensorflow implementation does not provide as much reduction in computational time using KDL-NNs compared to an FNN for MNIST, likely due to overhead costs, the KDLs still scale better as seen by the reduction in computational time for the CNNs used for CIFAR-10.  Thus, KDLs are still shown to scale better as the size of the problem increases, while maintaining a similar level of accuracy.  In addition, this shows that incorporating KDLs into frameworks that typically use fully connected layers is straightforward and simple.

\begin{figure}[!htbp]\begin{center}
\subfloat[MNIST \label{mnist_tf}]{\includegraphics[width=0.4\textwidth]{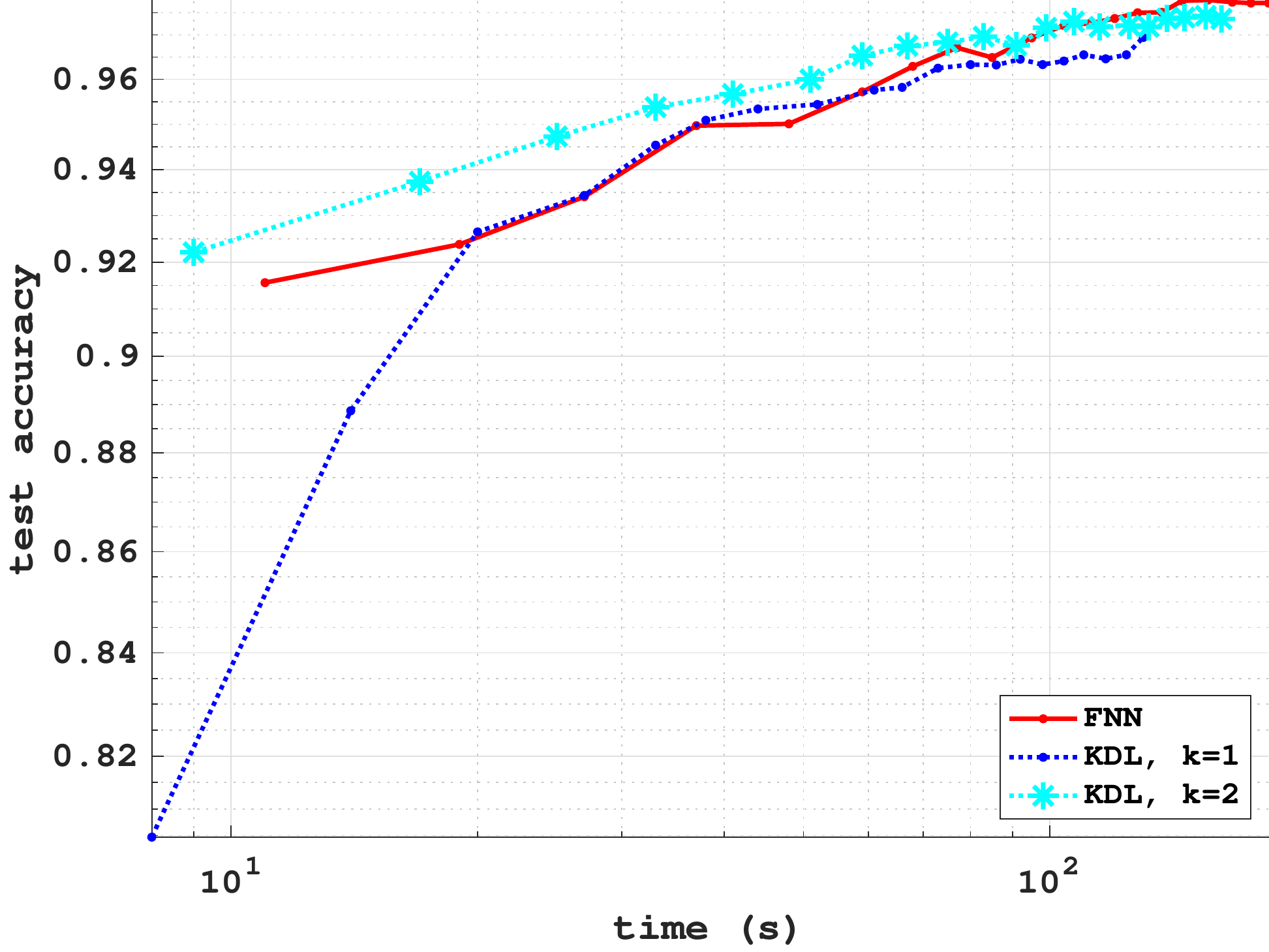}}
\subfloat[Cifar-10 \label{cifar-10}]{\includegraphics[width=0.4\textwidth]{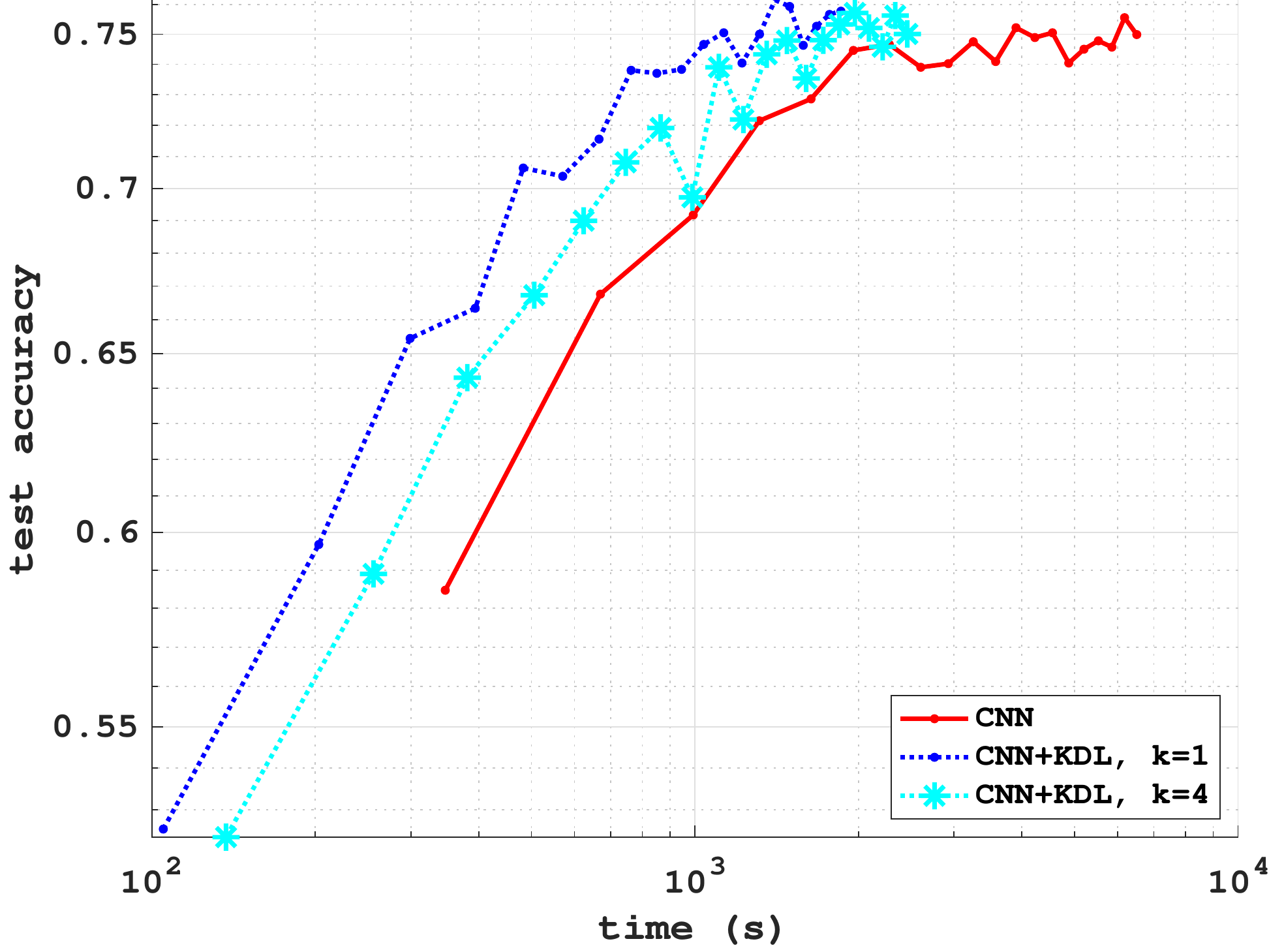}}
\end{center}
\caption{Test accuracy versus time on a log-log scale is shown in Figure~\ref{mnist_tf} for MNIST using an FNN and KDL-NNs with ranks~$k=1$ and $k=2$, and in Figure~\ref{cifar-10} for CIFAR-10 using the architectures in Table~\ref{Tab:CNN} where KDLs are implemented with ranks~$k=1$ and $k=4$. \label{fig:tf}}
\end{figure}

\section{Conclusions}
We have introduced a new neural network architecture, the KDL-NN, for use in deep learning. The architecture has been developed to exploit computational acceleration afforded by a Kronecker product representation of matrix multiplication when multiplying by large weight matrices during training. For an $m \times n$ matrix that is represent as the Kronecker product of $m_1 \times n_1$ and $m_2 \times n_2$ matrices ($m_1 m_2 = m$, $n_1 n_2 = n$), then analysis and practical evaluations have shown that when moderately large factors $m_1 \approx m_2 \approx n_1 \approx n_2$ are available for numbers of nodes, training a KDL-NN requires significantly less time compared to an FNN.  In addition, we have shown on several examples that the resulting accuracy of using a KDL-NN is generally improved compared to a FNN, and seems to be comparable to essentially doubling the number of layers.  However, further analysis is required to determine the extent to which this holds. In particular, our analysis does not reveal precisely what properties of the data suggest that a KDL-NN approach is effective.  While the reduction in computational time on a robust implementation such as Tensorflow is less dramatic, KDL's have been shown to scale at a slower rate compared to fully connected layers as the size of the problem increases, and are simple and straightforward to incorporate in place of fully connected layers.

Further, adding weight matrices to a KDL-NN is straightforward, but practical examples have shown that altering the learning rate when increasing the representative rank may be necessary. Since KDL-NNs provide a new framework for deep learning, there are many avenues of research that are yet to be pursued.  However, this work has shown the potential benefits of adopting KDLs and provided impetus to further establish the extent to which they may prove relevant. In the supplementary documentation of this paper, we show that a higher order Kronecker Multi-Layer NN (KML-NN) is feasible, but examples we have investigated suggest that such a generalization may be less effective than the simpler KDL-NN approach.

\appendix
\section{KDL Back-Propagation Derivation} \label{app:back}
Given a KDL-NN with $L-1$ KDL pairs and $A_R^{(1)} = X$, define
\begin{align*}
Z_L^{(\ell,i)} &= A_R^{(\ell-1)}W_L^{(\ell,i)} + B_L^{(\ell,i)}, &A_L^{(\ell,i)} &= \phi_2(Z_L^{(\ell,i)}),\\ Z_R^{(\ell,i)} &=  W_R^{(\ell,i)} A_L^{(\ell,i)} + B_R^{(\ell,i)}, & A_R^{(\ell)} &= \sum_i \phi_1(Z_R^{(\ell,i)}).\end{align*}

Differentiation from layer $L$ to 2, splitting the loss function $\mathcal{L}=\mathcal{L}_1+\mathcal{L}_2$ with $\mathcal{L}_1 = \frac{1}{2}\left\|Y-A_R^{(L)}\right\|_F^2$, and using $\circ$ to represent element-wise multiplication 
produces,
\begin{align*}
\frac{\partial \mathcal{L}}{\partial W_R^{(\ell,i)}} &= 
\frac{\partial \mathcal{L}_1}{\partial Z_R^{(\ell,i)}} \frac{\partial Z_R^{(\ell,i)}}{\partial W_R^{(\ell,i)}} + \frac{\partial \mathcal{L}_2}{\partial W_R^{(\ell,i)}} = \Delta_1^{(\ell,i)}A_L^{(\ell,i)T} + \lambda W_R^{(\ell,i)}\\
\frac{\partial \mathcal{L}}{\partial B_R^{(\ell,i)}} &= 
\frac{\partial \mathcal{L}_1}{\partial Z_R^{(\ell,i)}} \frac{\partial Z_R^{(\ell,i)}}{\partial B_R^{(\ell,i)}} + \frac{\partial \mathcal{L}_2}{\partial B_R^{(\ell,i)}} = \Delta_1^{(\ell,i)} + \lambda B_R^{(\ell,i)}\\
\frac{\partial \mathcal{L}}{\partial W_L^{(\ell,i)}} &=
\frac{\partial \mathcal{L}_1}{\partial Z_L^{(\ell,i)}} \frac{\partial Z_L^{(\ell,i)}}{\partial W_L^{(\ell,i)}} = A_R^{(\ell-1)T}\Delta_2^{(\ell,i)} + \frac{\partial \mathcal{L}_2}{\partial W_L^{(\ell,i)}} + \lambda W_L^{(\ell,i)}\\
\frac{\partial \mathcal{L}}{\partial B_L^{(\ell,i)}} &=
\frac{\partial \mathcal{L}_1}{\partial Z_L^{(\ell,i)}} \frac{\partial Z_L^{(\ell,i)}}{\partial B_L^{(\ell,i)}} + \frac{\partial \mathcal{L}_2}{\partial B_L^{(\ell,i)}} = \Delta_2^{(\ell,i)} + \lambda B_L^{(\ell,i)},
\end{align*}
where we have introduced the following notation for $\ell = L-1, \ldots, 2$:
\begin{align*}
  \Gamma^{(L+1)} &\coloneqq (Y-A_R^{(L)}), &
  \Gamma^{(\ell+1)} &\coloneqq \frac{\partial \mathcal{L}_1}{\partial A_R^{(\ell)}} = \sum_i \Delta_2^{(\ell+1,i)}W_L^{(\ell+1,i)T}, \\
  \Delta_1^{(\ell,i)} &\coloneqq \frac{\partial \mathcal{L}_1}{\partial A_R^{(\ell)}}, & 
  \Delta_2^{(\ell,i)} &\coloneqq \frac{\partial \mathcal{L}_1}{\partial Z_R^{(\ell,i)}}\frac{\partial Z_R^{(\ell,i)}}{\partial Z_L^{(\ell,i)}} = ((W_R^{(\ell,i)})^T\Delta_1^{(\ell,i)})\circ \phi_2'(Z_L^{(\ell,i)}),\\
  \frac{\partial A_R^{(\ell)}}{\partial Z_R^{(\ell,i)}} &= \Gamma^{(\ell+1)}\circ \phi_1'(Z_R^{(\ell,i)})\\
\end{align*}

\section{Proof of Theorem~\ref{thm:main-result}}\label{app:proof}
The proof of Theorem~\ref{thm:main-result} relies on some lemmas. Lemma~\ref{Lem:Lip} below computes Lipschitz constants for the individual functions $f_\ell$ defined in \eqref{eq:f-def}. Lemmas~\ref{Lem:2} and \ref{Lem:3} compute error estimates associated with KPD truncations of weight matrices, and hence leverage the foundational Kronecker product rearrangement result, Lemma~\ref{lemma:kron-approx}. The final intermediate result, Lemma~\ref{Lem:4}, computes an error estimate for a single layer of the KPD-NN versus a corresponding FNN. Following this, the proof of Theorem~\ref{thm:main-result} is furnished.

\begin{lemma}\label{Lem:Lip}
  Given KDL forward operation $f_\ell$ from \eqref{eq:f-def} with $c_1$- and $c_2$-Lipschitz activation functions $\phi_1$ and $\phi_2$ respectively, then $f_\ell$ is $C^{(\ell)}$-Lipschitz, where
\begin{align*}
C^{(\ell)} = c_1c_2 \left| \theta^{(\ell)} \right|_k^2.   
\end{align*}
\end{lemma}
\begin{proof}
Since $\phi_1$ and $\phi_2$ are Lipschitz, given inputs $X_1$ and $X_2$,
\begin{align*}
\left\|f_\ell \left(X_1\right) - f_\ell \left(X_2\right) \right\|_F \hspace{-60pt} &\\
&\leq c_1\left\|\sum_{i=1}^k W_R^{(\ell,i)}\phi_2\left(X_1 W_L^{(\ell,i)} + B_L^{(\ell,i)}\right) - \sum_{i=1}^k W_R^{(\ell,i)}\phi_2\left(X_2 W_L^{(\ell,i)} + B_L^{(\ell,i)}\right) \right\|_F \\
&\leq c_1\sum_{i=1}^k\left\| W_R^{(\ell,i)}\phi_2\left(X_1 W_L^{(\ell,i)} + B_L^{(\ell,i)}\right) - W_R^{(\ell,i)}\phi_2\left(X_2 W_L^{(\ell,i)} + B_L^{(\ell,i)}\right) \right\|_F \\
&\leq c_1\sum_{i=1}^k\left\| W_R^{(\ell,i)}\right\|_F \left\| \phi_2\left(X_1 W_L^{(\ell,i)} + B_L^{(\ell,i)}\right) - \phi_2\left(X_2 W_L^{(\ell,i)} + B_L^{(\ell,i)}\right) \right\|_F \\
&\leq c_1c_2\sum_{i=1}^k\left\| W_R^{(\ell,i)}\right\|_F \left\| X_1 W_L^{(\ell,i)} + B_L^{(\ell,i)} - X_2 W_L^{(\ell,i)} - B_L^{(\ell,i)} \right\|_F \\
&\leq c_1c_2\sum_{i=1}^k\left\| W_R^{(\ell,i)}\right\|_F \left\| X_1 - X_2\right\|_F \left\| W_L^{(\ell,i)} \right\|_F \\
&= \left(c_1c_2 \left| \theta^{(\ell)} \right|_k^2 \right) \left\| X_1 - X_2\right\|_F \\
\end{align*}
\end{proof}

\begin{lemma}\label{Lem:2}
Suppose fully connected FNN output $\bm{y}$ with matrix reshaping $Y$ has $L$ layers and output at layer $\ell\in (2,L)$ with parameters $\tilde{\theta}$, activation function $\phi = \phi_1$, layer input $\bm{a}^{(n-1)}$ with $\bm{a}^{(1)} = \bm{x}_m$ for training pair $(\bm{x}_m,\bm{y}_m)$ with matrix reshapings $X_m$ and $Y_m$, then there exists $\theta$, and activation function $\phi_2$ such that for full-rank KDL-NN output $Y_{r}^{\kappa}$ with $L$ layer pairs and layer pair $\ell\in (2,L)$ given by \eqref{Eqn:KPlayer1,Eqn:KPlayer2},
\begin{equation}
\argmin_{\theta,\phi_2} \left\| Y_m - Y_{r}^{\kappa}(X_m) \right\|_F^2 \leq \left\|Y_m - Y(X_m)\right\|_F^2
\end{equation}
\end{lemma}
\begin{proof}
First note that for KPD $W^{(\ell)} = \sum_{i=1}^r L^{(\ell,i)T} \otimes R^{(\ell,i)}$, setting $W_L^{(\ell,i)} = L^{(\ell,i)}$, $W_R^{(\ell,i)} = R^{(\ell,i)}$, $B_L^{(\ell,i)} = 0$, $\text{vec}\left(\sum_{i=1}^r B_R^{(\ell,i)}\right) = \bm{b}$,  and using $\phi_2$ as the linear activation function, then $A_R^{(\ell)}$ is a reshaping of $\bm{a}^{(\ell)}$ for $\ell\in(2,L)$, and $Y_{r}^{\kappa}(X_m) = Y(X_m)$.  Thus, the general result holds.
\end{proof}

\begin{lemma}\label{Lem:3}
Under assumptions of Lemma~\ref{Lem:2}, setting $k<r$, and for KPD at layer $\ell$ $\sum_{i=1}^r L^{(\ell,i)T} \otimes R^{(\ell,i)} = W^{(\ell)}$, then
$$ \left\| \left(\sum_{i=1}^k L^{(\ell,i)T} \otimes R^{(\ell,i)}\right)\bm{a}^{(\ell-1)} - W^{(\ell)}\bm{a}^{(\ell-1)} \right\|_2 \leq \epsilon^{(\ell,k)} \left\| \bm{a}^{(\ell-1)} \right\|_2, $$
where $\epsilon^{(\ell,k)} = \left(\sum_{i=k+1}^r \sigma_i^{(\ell)2}\right)^{\frac{1}{2}}$ for $\sigma_i^{(\ell)}$ as the $i^{\text{th}}$ singular value of $\mathcal{R}(W^{(\ell)})$.
\end{lemma}
\begin{proof}
\begin{align*}
\left\| \left(\sum_{i=1}^k L^{(\ell,i)T} \otimes R^{(\ell,i)}\right)\bm{a}^{(\ell-1)} - W^{(\ell)}\bm{a}^{(\ell-1)} \right\|_2^2\hspace{-100pt}&\\ &= \left\| \left(\sum_{i=1}^k L^{(\ell,i)T} \otimes R^{(\ell,i)} - W^{(\ell)}\right)\bm{a}^{(\ell-1)}\right\|_2^2\\
&\leq \left\| \left(\sum_{i=1}^k L^{(\ell,i)T} \otimes R^{(\ell,i)} - W^{(\ell)}\right)\right\|_2^2 \left\|\bm{a}^{(\ell-1)}\right\|_2^2\\
&= \sum_{i=k+1}^r \sigma_i^{(\ell)2} \left\| \bm{a}^{(\ell-1)} \right\|_2^2,
\end{align*} 
where the final equality holds by Lemma~\ref{lemma:kron-approx}.
\end{proof}

\begin{lemma}\label{Lem:4}
Under assumptions of Lemma~\ref{Lem:3}, and for $c_1$-Lipschitz activation functions $\phi$ and $\phi_1$, and layer operator $f_\ell = \left( \phi_1 \circ h_{\theta_R^{(\ell)}} \circ \phi_2 \circ h_{\theta_L^{(\ell)}} \right)$, then there exists $\theta$, and activation function $\phi_2$ such that 
$$
\argmin_{\theta,\phi_2} \left\| f_\ell \left(A^{(\ell-1)}\right) - A^{(\ell)} \right\|_F\\ \leq c_1 \epsilon^{(\ell,k)} \left\| A^{(\ell-1)} \right\|_F, $$
where $\epsilon^{(\ell,k)} = \left(\sum_{i=k+1}^r \sigma_i^{(\ell)2}\right)^{\frac{1}{2}}$ for $\sigma_i^{(\ell)}$ as the $i^{\text{th}}$ singular value of $\mathcal{R}(W^{(\ell)})$.
\end{lemma}
\begin{proof}
Setting $ B_L^{(\ell,i)} = 0 $, $ \sum_{i=1}^k \text{vec}(B_R^{(\ell,i)}) = \bm{b}^{(\ell)} $, $ W_L^{(\ell,i)} = L^{(\ell,i)} $, $ W_R^{(\ell,i)} = R^{(\ell,i)} $, and choosing $\phi_2$ as the linear activation function yields 
\begin{multline*}
\text{vec}\left(\sum_{i=1}^k W_R^{(\ell,i)}\left( A^{(\ell-1)} W_L^{(\ell,i)} + B_L^{(\ell,i)}\right) + B_R^{(\ell,i)}\right) = \left(\sum_{i=1}^k L^{(\ell,i)T} \otimes R^{(\ell,i)}\right)\bm{a}^{(\ell-1)} + \bm{b}^{(\ell)},
\end{multline*} 
and
\begin{multline*}
\left\| \phi\left(\left(\sum_{i=1}^k L^{(\ell,i)T} \otimes R^{(\ell,i)}\right)\bm{a}^{(\ell-1)} + \bm{b}^{(\ell)}\right) - \phi\left(W^{(\ell)}\bm{a}^{(\ell-1)} + \bm{b}^{(\ell)}\right) \right\|_2\\ 
\leq c_1 \left\| \left(\sum_{i=1}^k L^{(\ell,i)T} \otimes R^{(\ell,i)}\right)\bm{a}^{(\ell-1)} - W^{(\ell)}\bm{a}^{(\ell-1)} \right\|_2,
\end{multline*} 
since $\phi$ is $c_1$-Lipschitz.  Now applying Lemma~\ref{Lem:3} and reshaping into matrix format, the general result holds.
\end{proof}

\begin{proof}[Proof of Theorem~\ref{thm:main-result}]
Define the error at layer $\ell$ by $E^{(\ell)} = \left\| A_R^{(\ell)} - A^{(\ell)} \right\|_F$, and error from applying KDL forward operator $f_\ell$ by $e^{(\ell)} = \left\| f_\ell \left(A^{(\ell-1)}\right) - A^{(\ell)} \right\|_F$.  Then by Lemma~\ref{Lem:4} and by definition of $f_1(0)$, 
$$
E^{(2)} = e^{(2)} \leq c_1 \epsilon^{(2,k)} \left\| X \right\|_F = c_1 \epsilon^{(2,k)} \left\| f_1(0) \right\|_F, $$
since $A^{(1)} = A_R^{(1)} = X$, and 
\begin{align*}
e^{(\ell)} &\leq c_1 \epsilon^{(\ell,k)} \left\| A^{(\ell-1)} \right\|_F,\\
&= c_1 \epsilon^{(\ell,k)} \left\| f_{\ell-1}\left(A^{(\ell-2)}\right) - f_{\ell-1}(0) + f_{\ell-1}(0) \right\|_F\\ 
&\leq c_1^2 \epsilon^{(\ell,k)} \left\| A^{(\ell-2)} - 0 \right\|_F + c_1 \epsilon^{(\ell,k)} \left\| f_{\ell-1}(0) \right\|_F \\
&= c_1^2 \epsilon^{(\ell,k)} \left\| f_{\ell-2}\left(A^{(\ell-3)}\right) - f_{\ell-2}(0) + f_{\ell-2}(0) \right\|_F + c_1 \epsilon^{(\ell,k)} \left\| f_{\ell-1}(0) \right\|_F \\ 
&\vdots\\
&\leq \sum_{i=1}^{n-1} c_1^{n-i} \left\| f_i (0) \right\|_F.
\end{align*}
Further,
\begin{align*}
E^{(\ell)} &= \left\| A_R^{(\ell)} - f_{\ell} \left(A^{(\ell-1)}\right) + f_{\ell} \left(A^{(\ell-1)}\right) - A^{(\ell)} \right\|_F\\
&\leq \left\|A_R^{(\ell)} - f_{\ell} \left(A^{(\ell-1)}\right) \right\|_F + e^{(\ell)}.
\end{align*}
By Lemma~\ref{Lem:Lip},
\begin{align*}
\left\|f_{\ell} \left(A_R^{(\ell-1)}\right) - f_{\ell} \left(A^{(\ell-1)}\right) \right\|_F &\leq C^{(\ell)} \left\|A_R^{(\ell-1)} - A^{(\ell-1)} \right\|_F\\
&= C^{(\ell)} E^{(\ell-1)} \\
\end{align*}
Thus,
\begin{align*}
E^{(\ell)} &\leq C^{(\ell)} E^{(\ell-1)} + \epsilon^{(\ell,k)} \sum_{i=1}^{\ell-1} c_1^{\ell-i} \left\| f_i (0) \right\|_F,
\end{align*}
and
\begin{align*}
E^{(L)} &\leq C^{(L)}\left( C^{(L-1)} \left(\cdots \left( c_1 \epsilon^{(2,k)} C^{(3)} \left\| f_1(0) \right\|_F + \cdots \right)\cdots \right) + \cdots \right)\\ &\hspace{24pt} + \epsilon^{(L,k)} \sum_{i=1}^{L-1} c_1^{L-i} \left\| f_i (0) \right\|_F\\
&= \sum_{i=2}^L \epsilon^{(i,k)} \left( \prod_{j=i+1}^L C^{(j)} \right) \sum_{k=1}^{i-1} c_1^{i-k} \left\| f_k (0) \right\|_F\\
&= \sum_{i=2}^L \epsilon^{(i,k)} \left( \prod_{j=i+1}^L c_1c_2\left| \theta^{(j)} \right|_k^2 \right) \sum_{k=1}^{i-1} c_1^{i-k} \left\| f_k (0) \right\|_F\\
&= \sum_{i=2}^L \epsilon^{(i,k)} \left( \prod_{j=i+1}^L c_2\left| \theta^{(j)} \right|_k^2 \right) c_1^{(L-i)}\sum_{k=1}^{i-1} c_1^{i-k} \left\| f_k (0) \right\|_F\\
&= \sum_{i=2}^L \epsilon^{(i,k)} \left( \prod_{j=i+1}^L c_2\left| \theta^{(j)} \right|_k^2 \right) \sum_{k=1}^{i-1} c_1^{L-k} \left\| f_k (0) \right\|_F.
\end{align*}
Finally, given $\phi_2(X) = X$, then $c_2=1$.
\end{proof}

\bibliographystyle{siamplain}
\bibliography{references}
\clearpage

\thispagestyle{empty}
\vspace*{24pt}

\title{\LARGE{\textbf{Supplementary Materials: Weight Matrix Dimensionality Reduction in Deep Learning via Kronecker Multi-layer Architectures}}\footnote{This research was sponsored by ARL under Cooperative Agreement Number W911NF-12-2-0023. The views and conclusions contained in this document are those of the authors and should not be interpreted as representing the official policies, either expressed or implied, of ARL or the U.S. Government. The U.S. Government is authorized to reproduce and distribute reprints for Government purposes notwithstanding any copyright notation herein.  The first and third authors are partially supported by NSF DMS-1848508 and AFOSR FA9550-20-1-0338.}}
\author{Jarom D. Hogue\footnote{Scientific Computing and Imaging Institute, University of Utah, Salt Lake City, UT(\href{mailto:jdhogue@sci.utah.edu}{jdhogue@sci.utah.edu}).} \and Robert M. Kirby\footnote{Scientific Computing and Imaging Institute and School of Computing, University of Utah, Salt Lake City, UT(\href{mailto:kirby@cs.utah.edu}{kirby@cs.utah.edu}).} \and Akil  Narayan\footnote{Scientific Computing and Imaging Institute and Department of Mathematics, University of Utah, Salt Lake City, UT(\href{mailto:akil@sci.utah.edu}{akil@sci.utah.edu}).}}

\maketitle

\noindent\rule{\textwidth}{0.8pt}

\section{Choice of Optimizer and Minibatch Size}

Figure~\ref{fig:val_a100} utilizes the Adam optimizer with a minibatch size of 100 instead of stochastic gradient descent (SGD) used in Figure 5.2, and produces similar results.  While this is not an exhaustive test on optimizers, it does show that the efficacy of KDL-NNs is not reliant on SGD.

\begin{figure}[htbp]
  \begin{center}
\includegraphics[width=0.25\textwidth]{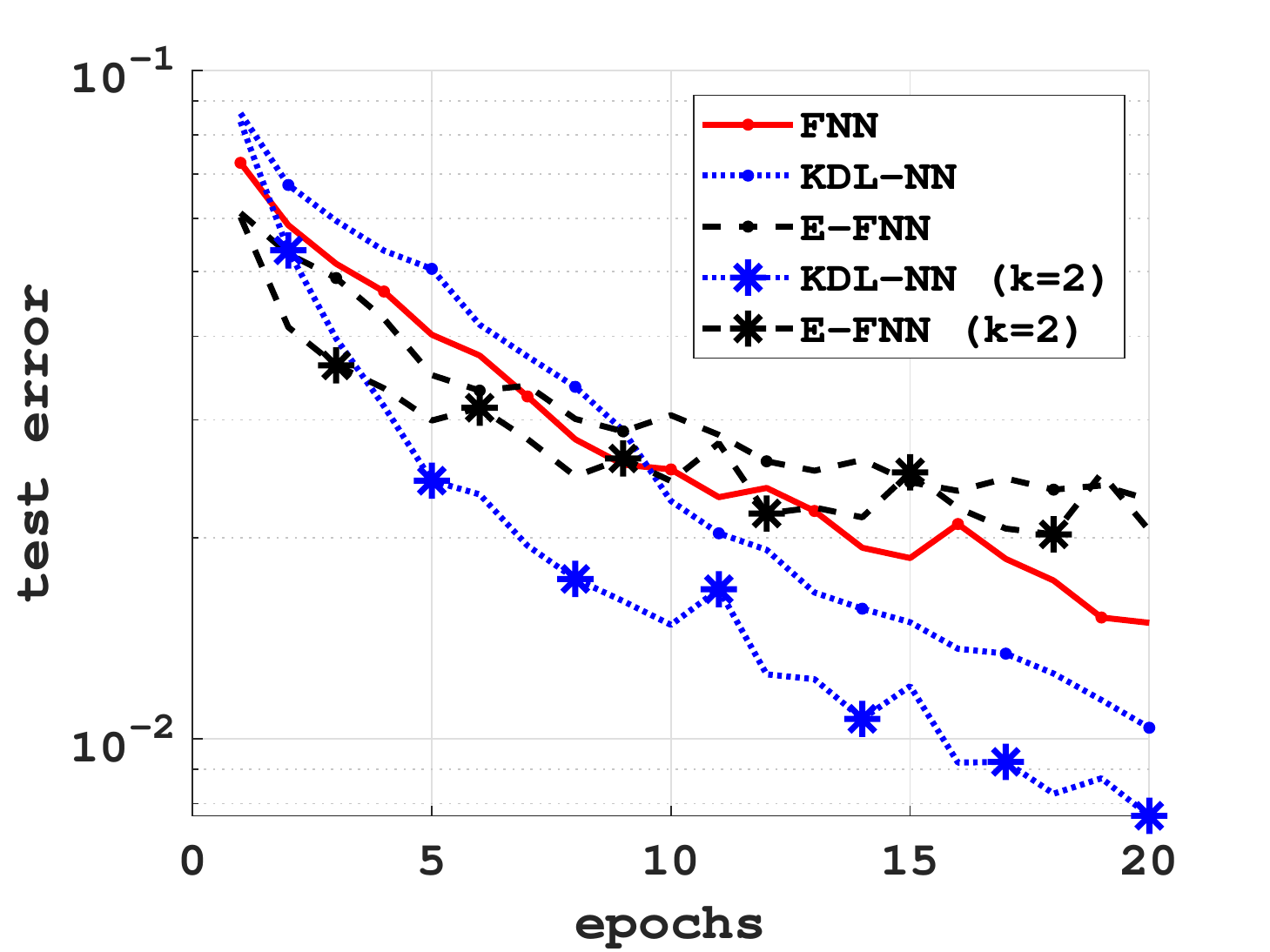}
\includegraphics[width=0.25\textwidth]{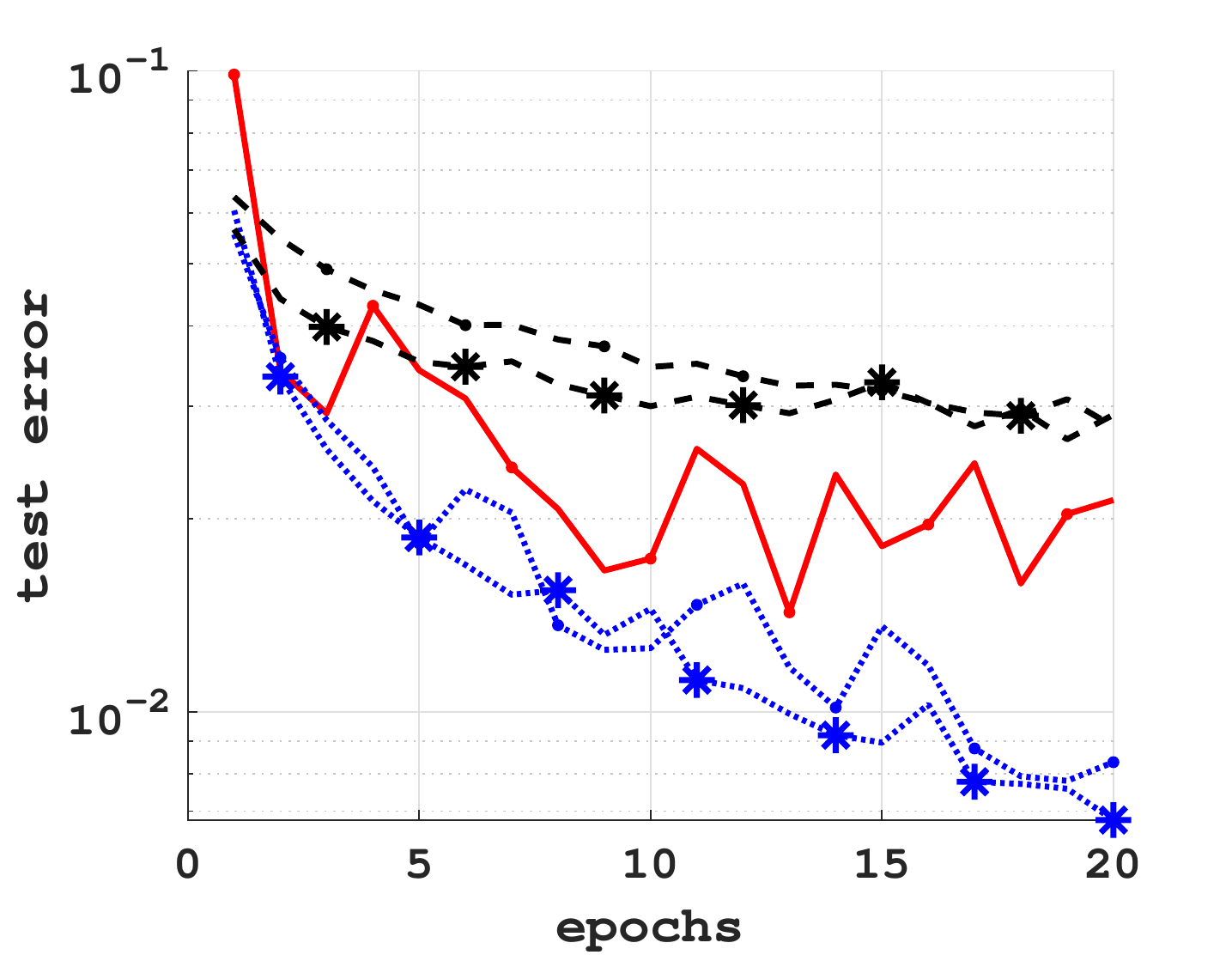}
\includegraphics[width=0.25\textwidth]{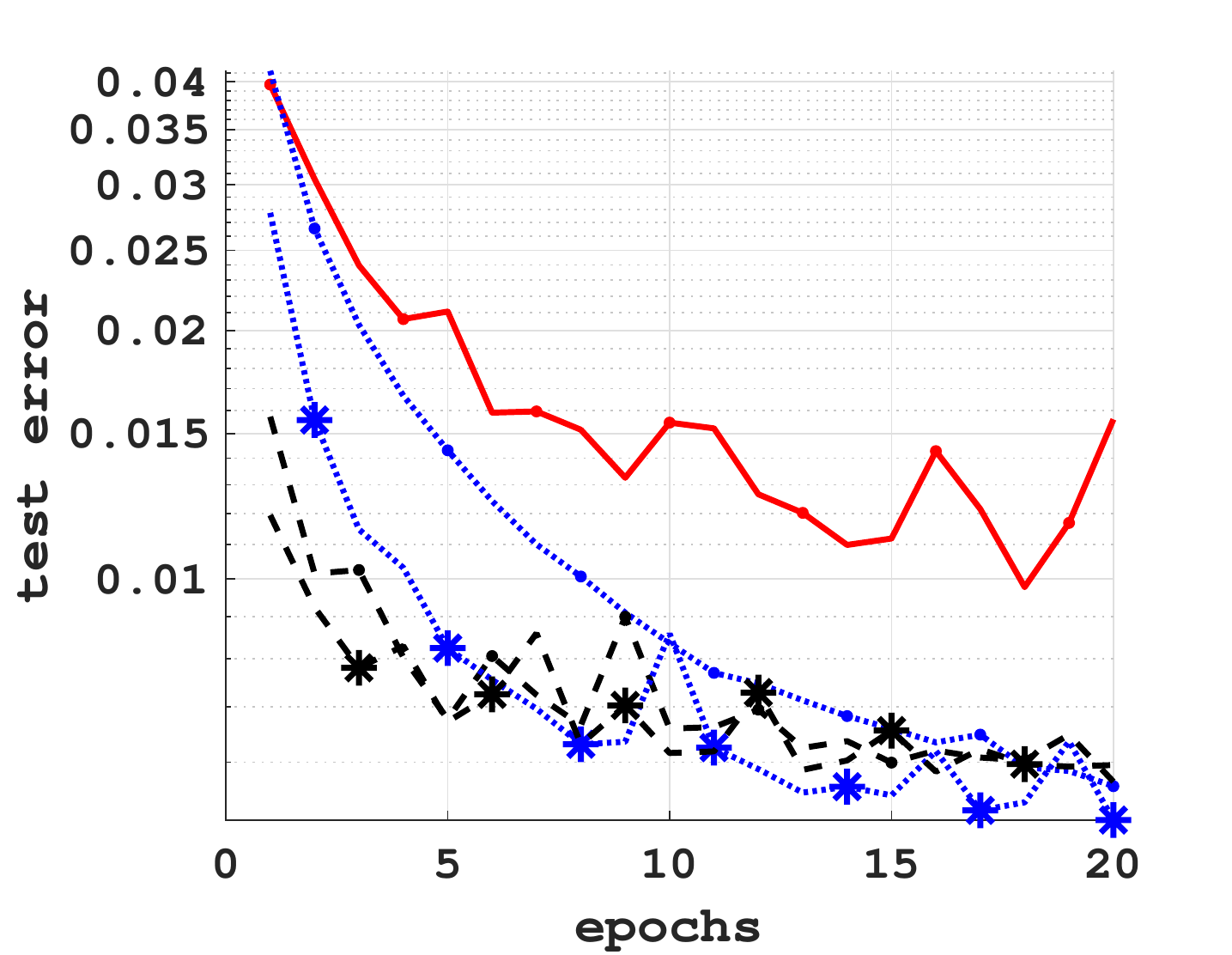}
\includegraphics[width=0.25\textwidth]{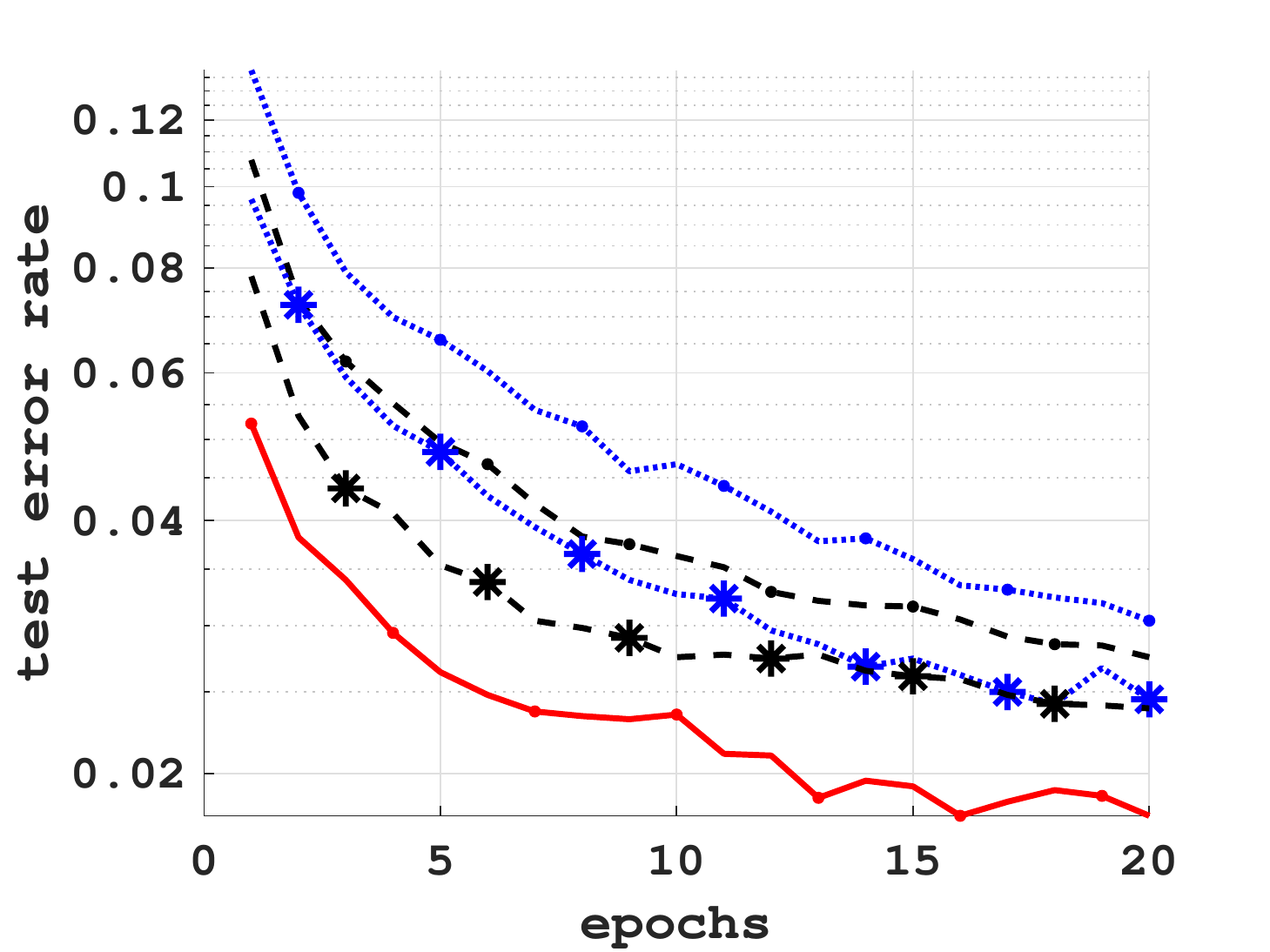}
\\
\subfloat[BSD (a) 
\label{bikes8c}]{\includegraphics[width=0.25\textwidth]{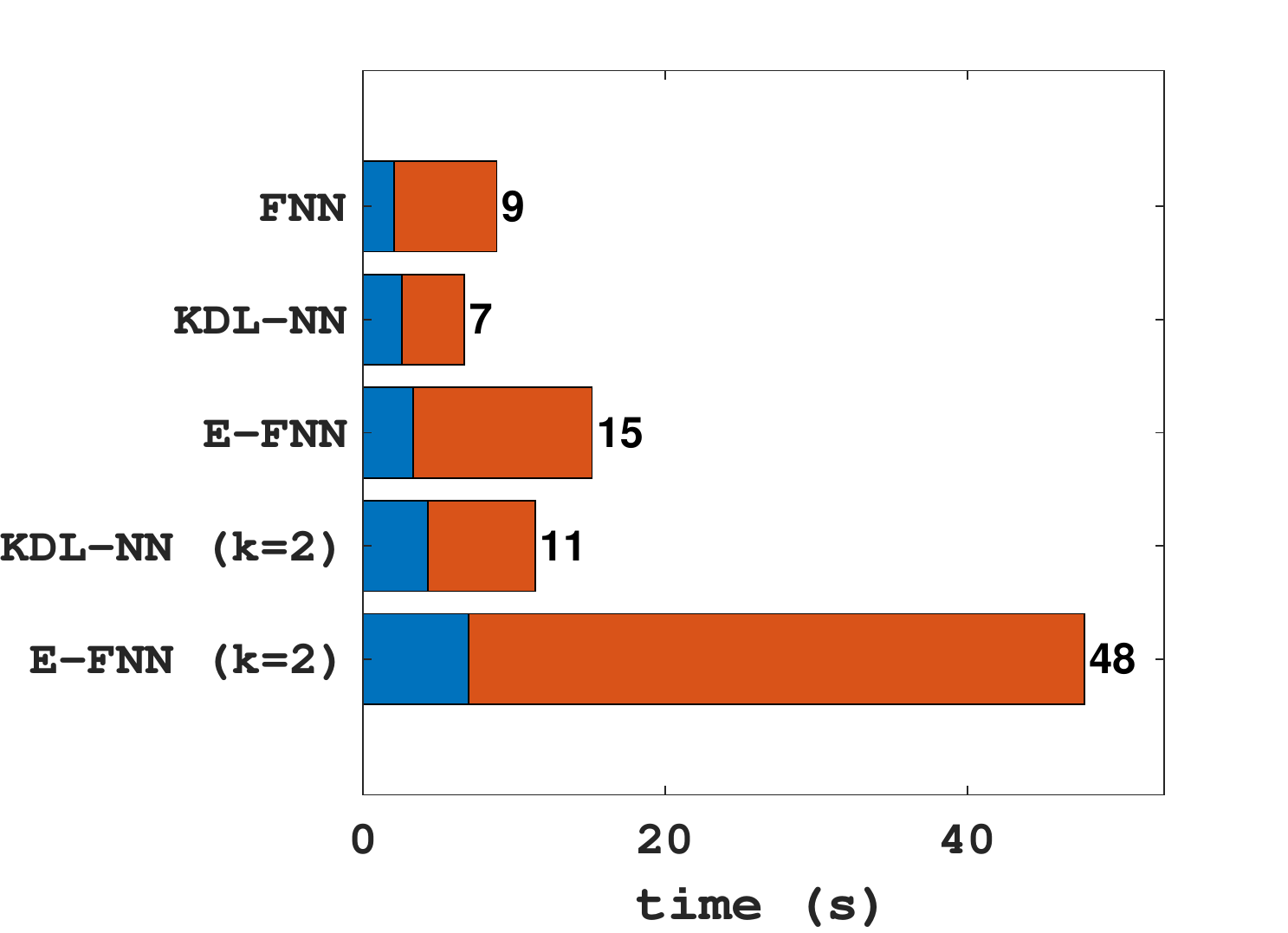}}
\subfloat[BSD (b) 
\label{bikesc}]{\includegraphics[width=0.25\textwidth]{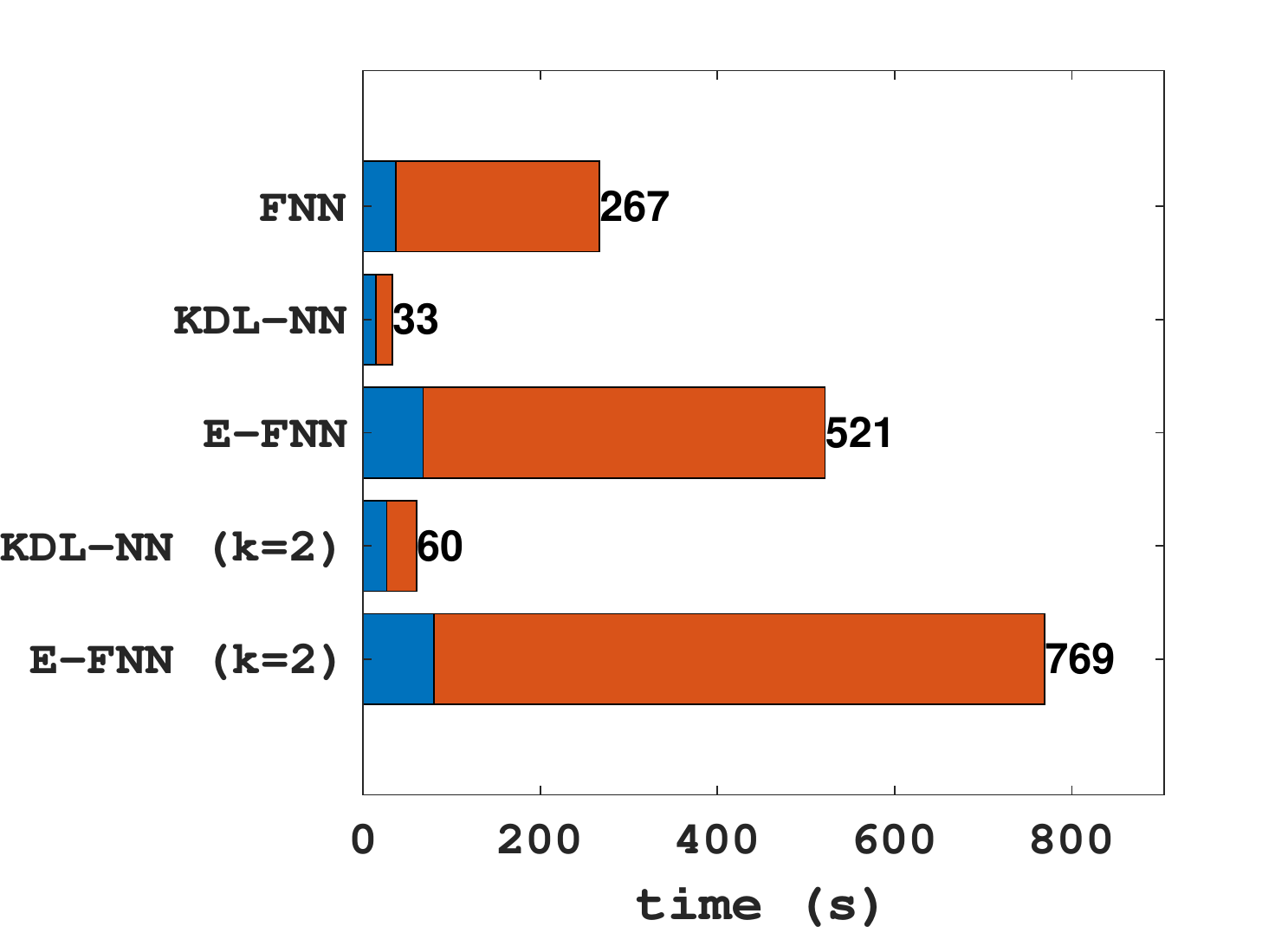}}
\subfloat[BF 
\label{blogc}]{\includegraphics[width=0.25\textwidth]{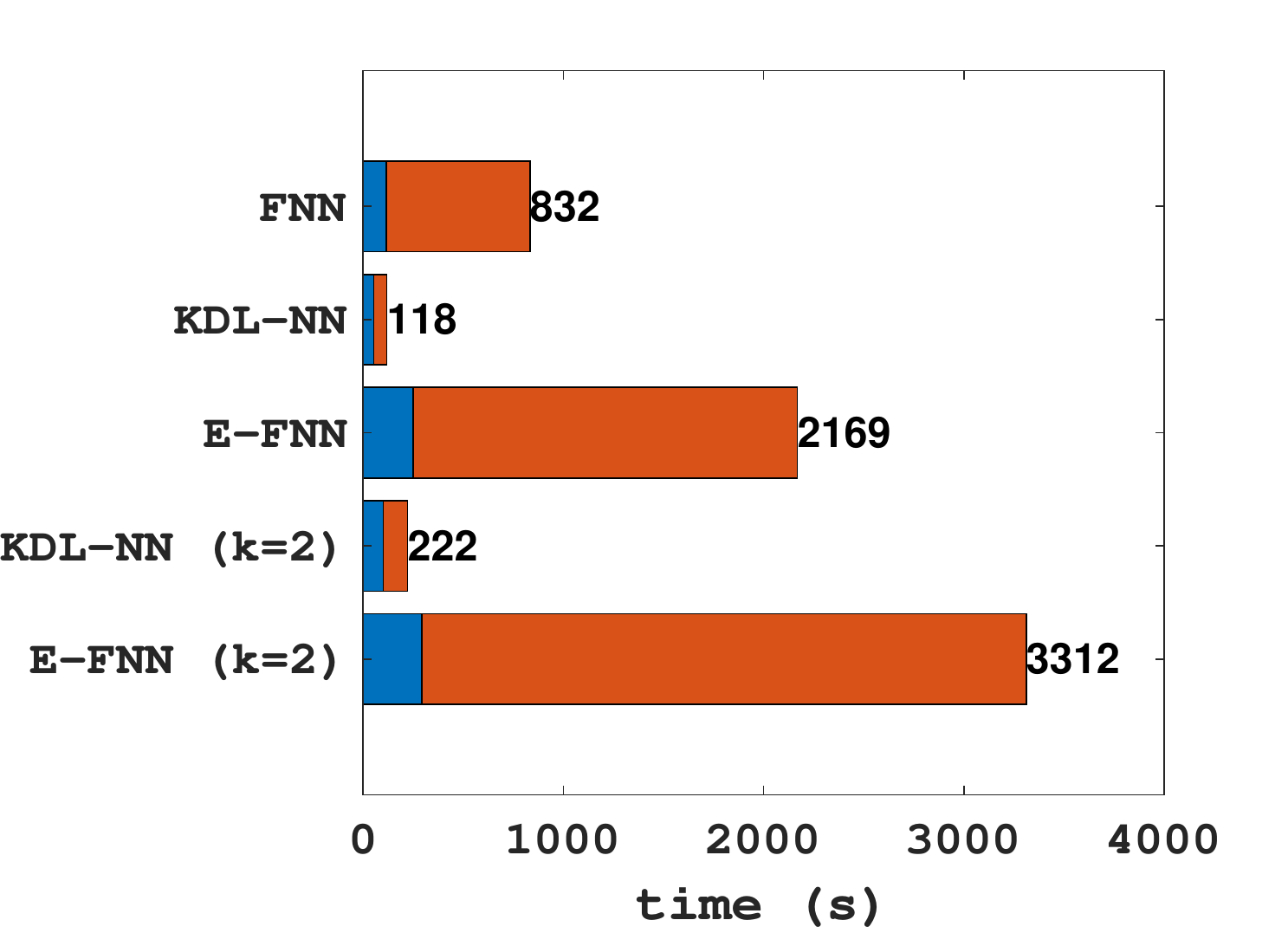}}
\subfloat[MNIST 
\label{MNISTc}]{\includegraphics[width=0.25\textwidth]{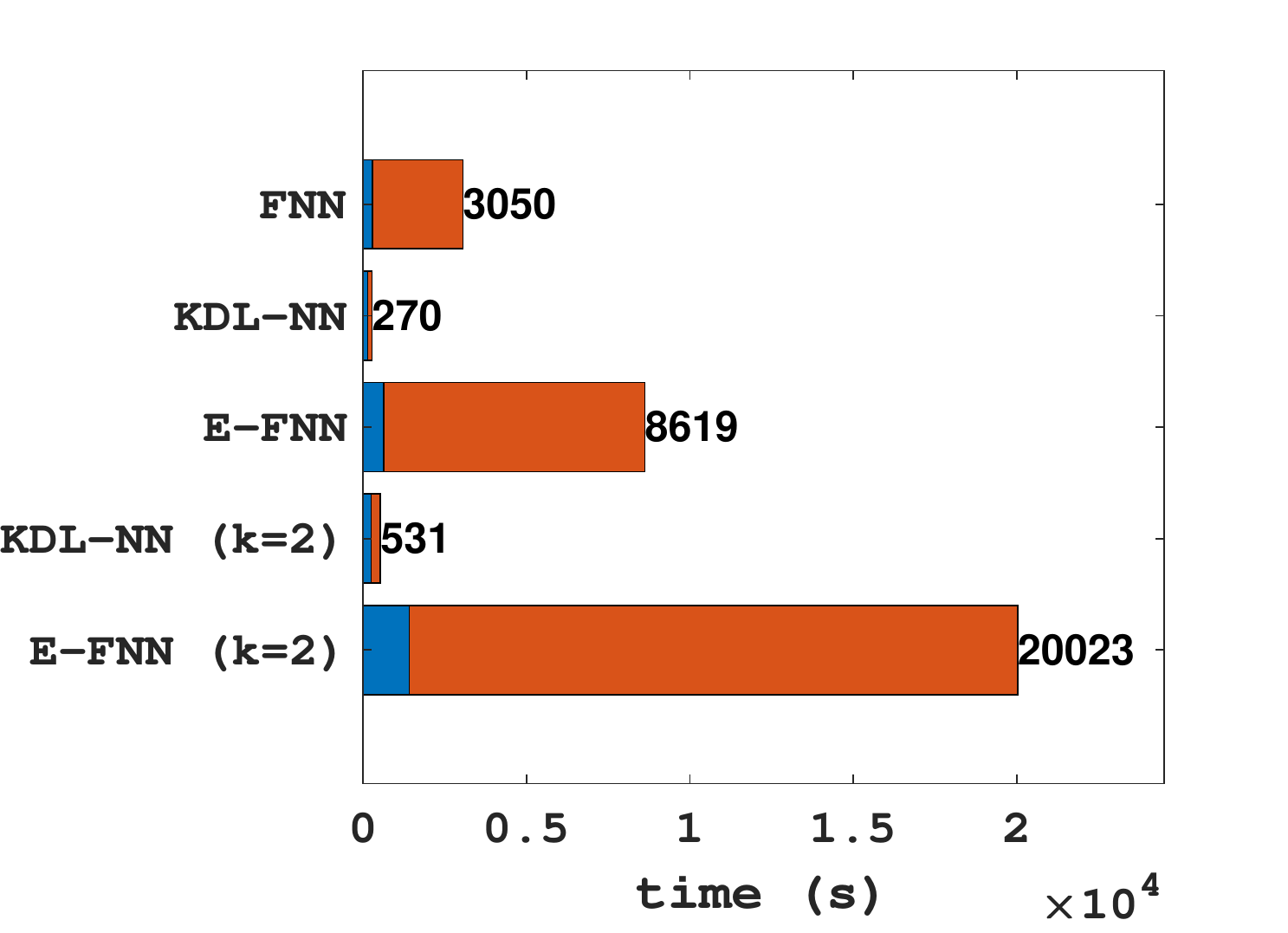}}
\end{center}
\caption{Figures~\ref{bikes8c} to \ref{MNISTc} show the test errors and timing, broken down by forward operations and back-propagations, for FNN, KDL-NN, and E-FNN for BSD (a and b), BF, and MNIST respectively as defined in Table 5.2 with Kronecker ranks~$k=1$ and $k=2$. Adam optimization is utilized with a mini-batch size of 100. 
\label{fig:val_a100}}
\end{figure}

\section{High Order Kronecker Multi-Layer}\label{app:high}

A Kronecker multi-layer (KML) approach is explained here. Given a KDL, consider the refinements $W_L^{(\ell,i)} = \sum_{j=1}^r W_2^{(\ell,j)T}\otimes W_1^{(\ell,j)}$ and $W_R^{(\ell,i)} = \sum_{j=1}^r W_3^{(\ell,j)}\otimes W_4^{(\ell,j)T}$. Then for column $\bm{a}_i$ of $A$, \linebreak $(A^T W_L^{(\ell,i)})^T = \sum_{j=1}^r \left[\begin{array}{ccc}(W_1^{(\ell,j)T}\otimes W_2^{(\ell,j)})\bm{a}_1 & \cdots & (W_1^{(\ell,j)T}\otimes W_2^{(\ell,j)})\bm{a}_p\end{array}\right]$, and with a slight abuse of notation, $W_R^{(\ell,i)} A = \sum_{j=1}^r \left[\begin{array}{ccc}(W_3^{(\ell,j)T}\otimes W_4^{(\ell,j)})\bm{a}_1 & \cdots & (W_3^{(\ell,j)T}\otimes W_4^{(\ell,j)})\bm{a}_p\end{array}\right]$. KP multiplication operations with $A_i = \text(mat)(\bm{a}_i$ are then implemented as $W_2^{(\ell,j)})A_i W_1^{(\ell,j)}$ and $W_2^{(\ell,j)})A_i W_1^{(\ell,j)}$.  This even split into 4 multi-layers will be referred to in the node configuration by a refinement using parenthesis. i.e. a KDL given by $(28,28)$ could be refined into a configuration given by $((7,4),(7,4))$.

Adding activation functions and bias terms, the multi-layers can be written as
\begin{align*}
Z_1^{(\ell,i,j)} &= A_{4(j,:)}^{(\ell-1)}W_1^{(\ell,i)} + B_1^{(\ell,i)}, &A_1^{(\ell,i,j)} &= \phi(Z_1^{(\ell,i,j)}),\\ 
Z_2^{(\ell,i,j)} &=  W_2^{(\ell,i)} A_1^{(\ell,i,j)} + B_2^{(\ell,i)}, & A_{2(j,:)}^{(\ell,i)} &= \phi(\text{vec}(Z_2^{(\ell,i,j)})^T),\\
Z_3^{(\ell,i,j)} &=  A_{2(:,j)}^{(\ell,i)}W_3^{(\ell,i)} + B_3^{(\ell,i)}, &A_3^{(\ell,i,j)} &= \phi(Z_3^{(\ell,i,j)}),\\ 
Z_4^{(\ell,i,j)} &=  W_4^{(\ell,i)} A_3^{(\ell,i,j)} + B_4^{(\ell,i)}, & A_{4(:,j)}^{(\ell)} &= \sum_i \phi(\text{vec}(Z_4^{(\ell,i,j)})),
\end{align*}
where \textsc{Matlab} style notation is used in subscripts to differentiate between rows or columns being reshaped into matrix form.

Back-propagation follows with
\begin{align*}
& \Gamma^{(L+1,j)} := (Y-A_4^{(L)}) \\
& \Gamma^{(\ell+1,j)} :=  
\frac{\partial \mathcal{L}_1}{\partial A_{4(:,j)}^{(\ell)}} = \sum_i \Delta_1^{(\ell+1,i,j)}W_1^{(\ell+1,i)T}, \qquad \ell = L-1,\cdots 2 \\
& \Delta_4^{(\ell,i,j)} := \frac{\partial \mathcal{L}_1}{\partial A_{4(:,j)}^{(\ell)}}
\frac{\partial A_{4(:,j)}^{(\ell)}}{\partial Z_4^{(\ell,i,j)}} = \sum_j \Gamma^{(\ell+1,j)} \circ \phi'(Z_4^{(\ell,i,j)})\\
\frac{\partial C}{\partial W_4^{(\ell,i)}} &= \sum_j\frac{\partial \mathcal{L}_1}{\partial Z_4^{(\ell,i,j)}}\frac{\partial Z_4^{(\ell,i,j)}}{\partial W_4^{(\ell,i)}} + \frac{\partial \mathcal{L}_2}{\partial W_4^{(\ell,i)}} = \sum_j\Delta_4^{(\ell,i,j)} A_3^{(\ell,i,j)T} + \lambda W_4^{(\ell,i)}\\
\frac{\partial C}{\partial B_4^{(\ell,i)}} &= \sum_j\frac{\partial \mathcal{L}_1}{\partial Z_4^{(\ell,i,j)}}\frac{\partial Z_4^{(\ell,i,j)}}{\partial B_4^{(\ell,i)}} + \frac{\partial \mathcal{L}_2}{\partial B_4^{(\ell,i)}} = \sum_j\Delta_4^{(\ell,i,j)} + \lambda B_4^{(\ell,i)}\\
& \Delta_3^{(\ell,i,j)} := \frac{\partial \mathcal{L}_1}{\partial Z_4^{(\ell,i,j)}}\frac{\partial Z_4^{(\ell,i,j)}}{\partial Z_3^{(\ell,i,j)}} = ((W_4^{(\ell,i)})^T\Delta_4^{(\ell,i,j)})\circ \phi'(Z_3^{(\ell,i,j)})\\
\frac{\partial C}{\partial W_3^{(\ell,i)}} &= \sum_j\frac{\partial \mathcal{L}_1}{\partial Z_3^{(\ell,i,j)}} \frac{\partial Z_3^{(\ell,i,j)}}{\partial W_3^{(\ell,i)}} + \frac{\partial \mathcal{L}_2}{\partial W_3^{(\ell,i)}} = \sum_j A_2^{(\ell,i,j)T}\Delta_3^{(\ell,i,j)} + \lambda W_3^{(\ell,i)}\\
\frac{\partial C}{\partial B_3^{(\ell,i)}} &= \sum_j\frac{\partial \mathcal{L}_1}{\partial Z_3^{(\ell,i,j)}} \frac{\partial Z_3^{(\ell,i,j)}}{\partial B_3^{(\ell,i)}} + \frac{\partial \mathcal{L}_2}{\partial B_3^{(\ell,i)}} = \sum_j\Delta_3^{(\ell,i,j)} + \lambda B_3^{(\ell,i)}\\
& \Delta_{2(p,q)}^{(\ell,i,j)}:= \frac{\partial \mathcal{L}_1}{\partial Z_{3(j,q)}^{(\ell,i,p)}}\frac{\partial Z_{3(j,q)}^{(\ell,i,p)}}{\partial Z_{2(p,q)}^{(\ell,i,j)}} = (\Delta_{3(j,q)}^{(\ell,i,p)} W_{3(p,q)}^{(\ell,i)T}) \circ \phi'(Z_{2(p,q)}^{(\ell,i,j)})\\
\frac{\partial C}{\partial W_2^{(\ell,i)}} &= \sum_j \frac{\partial \mathcal{L}_1}{\partial Z_2^{(\ell,i,j)}}\frac{\partial Z_2^{(\ell,i,j)}}{\partial W_2^{(\ell,i)}} + \frac{\partial \mathcal{L}_2}{\partial W_2^{(\ell,i)}} = \sum_j \Delta_2^{(\ell,i,j)} A_1^{(\ell,i,j)T} + \lambda W_2^{(\ell,i)}\\
\frac{\partial C}{\partial B_2^{(\ell,i)}} &= \sum_j \frac{\partial \mathcal{L}_1}{\partial Z_2^{(\ell,i,j)}}\frac{\partial Z_2^{(\ell,i,j)}}{\partial B_2^{(\ell,i)}} + \frac{\partial \mathcal{L}_2}{\partial B_4^{(\ell,i)}} = \sum_j \Delta_2^{(\ell,i,j)} + \lambda B_2^{(\ell,i)}\\
& \Delta_1^{(\ell,i,j)} := \frac{\partial \mathcal{L}_1}{\partial Z_2^{(\ell,i,j)}}\frac{\partial Z_2^{(\ell,i,j)}}{\partial Z_1^{(\ell,i,j)}} = \sum_j ((W_2^{(\ell,i)})^T\Delta_2^{(\ell,i,j)})\circ \phi'(Z_1^{(\ell,i,j)})\\
\frac{\partial C}{\partial W_1^{(\ell,i)}} &= \sum_j \frac{\partial \mathcal{L}_1}{\partial Z_1^{(\ell,i,j)}} \frac{\partial Z_1^{(\ell,i,j)}}{\partial W_1^{(\ell,i)}} + \frac{\partial \mathcal{L}_2}{\partial W_1^{(\ell,i)}} = \sum_j A_2^{(\ell,i,j)T}\Delta_1^{(\ell,i,j)} + \lambda W_1^{(\ell,i)}\\
\frac{\partial C}{\partial B_1^{(\ell,i)}} &= \sum_j \frac{\partial \mathcal{L}_1}{\partial Z_1^{(\ell,i,j)}} \frac{\partial Z_1^{(\ell,i,j)}}{\partial B_1^{(\ell,i)}} + \frac{\partial \mathcal{L}_2}{\partial B_1^{(\ell,i)}} = \sum_j \Delta_1^{(\ell,i,j)} + \lambda B_1^{(\ell,i)}
\end{align*}

In practice, viable training has required separate weights and biases for each layer of the split, i.e.
\begin{align*}
Z_1^{(\ell,i,j)} &= A_{4(j,:)}^{(\ell-1)}W_1^{(\ell,i,j)} + B_1^{(\ell,i,j)}, &A_1^{(\ell,i,j)} &= \phi(Z_1^{(\ell,i,j)}),\\ 
Z_2^{(\ell,i,j)} &=  W_2^{(\ell,i,j)} A_1^{(\ell,i,j)} + B_2^{(\ell,i,j)}, & A_{2(j,:)}^{(\ell,i)} &= \phi(\text{vec}(Z_2^{(\ell,i,j)})^T),\\
Z_3^{(\ell,i,j)} &=  A_{2(:,j)}^{(\ell,i)}W_3^{(\ell,i,j)} + B_3^{(\ell,i,j)}, &A_3^{(\ell,i,j)} &= \phi(Z_3^{(\ell,i,j)}),\\ 
Z_4^{(\ell,i,j)} &=  W_4^{(\ell,i,j)} A_3^{(\ell,i,j)} + B_4^{(\ell,i,j)}, & A_{4(:,j)}^{(\ell)} &= \sum_i \phi(\text{vec}(Z_4^{(\ell,i,j)})).
\end{align*}

Results are shown in Figure~\ref{fig:K4} using an even 4-split KML-NN on MNIST with node configuration $N = \{ ((7,4),(7,4)),((7,4),(7,4)),((7,4),(7,4)),((5,1),(2,1)) \}$.  Similar to using KDL-NN on BSD (a), using this KML-NN on MNIST with these small values results in an overall increase in time, and further analysis on larger sets is still required to determine the benefits of adopting higher order KMLs.

\begin{figure}[htbp]\begin{center} 
\subfloat[MNIST \label{MNIST_K4}]{\includegraphics[width=0.35\textwidth]{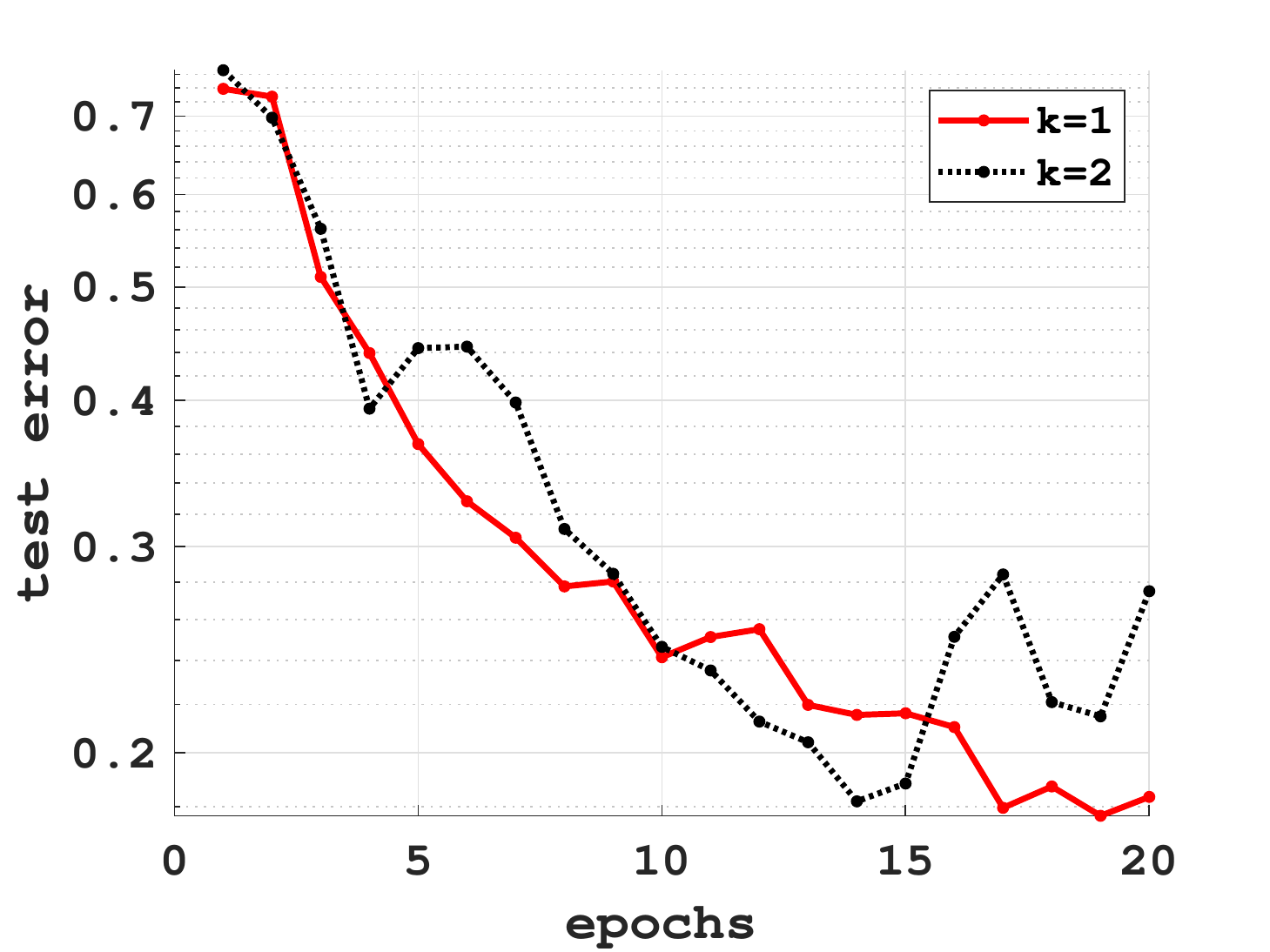}}
\subfloat[MNIST \label{MNIST_K4t}]{\includegraphics[width=0.35\textwidth]{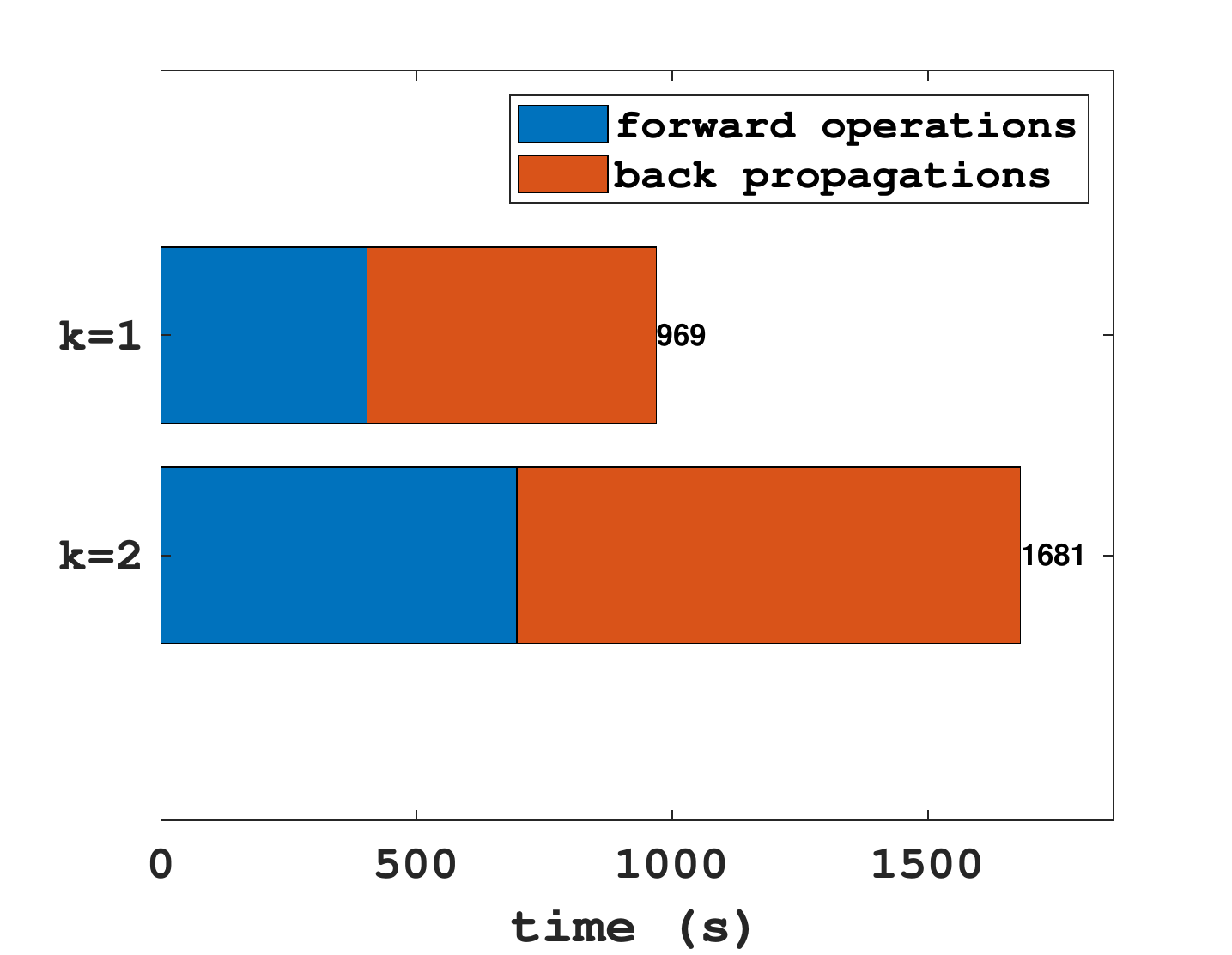}}
\end{center}
\caption{Figures~\ref{MNIST_K4} and \ref{MNIST_K4t} show the test errors and timing breakdowns for a KML-NN with 4 multi-layers with Kronecker ranks 1 and 2 for MNIST.\label{fig:K4}}
\end{figure}

\end{document}